\documentclass[letterpaper, 10 pt, conference]{ieeeconf}  
\IEEEoverridecommandlockouts                              
\usepackage{graphicx}
\usepackage{epsfig} 
\usepackage{mathptmx}
\usepackage{amsmath} 
\usepackage{bm}

\usepackage{amssymb}  
\usepackage{amsthm}
\usepackage{siunitx}
\usepackage{wasysym}
\usepackage[mathcal]{euscript}
\usepackage{bbm}
\usepackage{color}
\usepackage{lipsum}
\usepackage[ruled,vlined,linesnumbered]{algorithm2e}\usepackage[colorlinks=true,linkcolor=black,anchorcolor=black,citecolor=black,filecolor=black,menucolor=black,runcolor=black,urlcolor=black]{hyperref}
\usepackage{etoolbox}
\usepackage[table,xcdraw,dvipsnames]{xcolor}
\usepackage{multirow}
\usepackage{bbm}

\usepackage{outlines}
\let\labelindent\relax
\usepackage{enumitem}
\setenumerate[2]{label=\alph*.}
\setenumerate[3]{label=\roman*.}

\makeatletter
\patchcmd{\@makecaption}
  {\scshape}
  {}
  {}
  {}
\makeatother

\usepackage[
    style=ieee,
    doi=false,
    isbn=false,
    url=false,
    eprint=false,
    backend=biber,
    natbib=true
    ]{biblatex}

\usepackage{caption}
\usepackage{subcaption}
\usepackage{marginnote}
\bibliography{references}
\newcounter{change}
\newtheorem{defn}{Definition}
\newtheorem{rem}[defn]{Remark}
\newtheorem{lem}[defn]{Lemma}
\newtheorem{prop}[defn]{Proposition}
\newtheorem{assum}[defn]{Assumption}

\newtheorem{thm}[defn]{Theorem}
\newtheorem{cor}[defn]{Corollary}
\newtheorem{problem}[defn]{Problem}

\providecommand{\R}{\ensuremath \mathbb{R}}
\providecommand{\N}{\ensuremath \mathbb{N}}
\newcommand{\Xtrack}{Z}
\newcommand{\Xplan}{P}

\newcommand{\Workspace}{W}
\newcommand{\K}{K}
\newcommand{\Xgoal}{\Xplan\lbl{G}}

\newcommand{\U}{U}
\newcommand{\Eset}{E}

\newcommand{\Rbrs}{\Omega}

\newcommand{\Obs}{\mathcal{O}}
\newcommand{\Sap}{S}
\newcommand{\Poly}{P}
\newcommand{\inset}{X}

\newcommand{\Xpoly}{X}
\newcommand{\Voro}{V}
\newcommand{\Voror}{\bar{\Voro}}
\newcommand{\Ravd}{\Lambda}

\newcommand{\regtext}[1]{\mathrm{\textnormal{#1}}}
\newcommand{\ol}[1]{\overline{#1}}

\newcommand{\tss}[1]{\textsuperscript{#1}}
\newcommand{\vc}[1]{\mathbf{#1}}
\newcommand{\lbl}[1]{_{\regtext{#1}}}
\newcommand{\itv}[1]{\tilde{#1}}

\newcommand{\zeros}{\mathbf{0}}

\newcommand{\eye}{\mathbf{I}}

\newcommand{\proj}[2][]{\regtext{proj}_{#1}\!\left(#2\right)}
\newcommand{\norm}[1]{\left\Vert#1\right\Vert}
\newcommand{\abs}[1]{\left\vert#1\right\vert}

\newcommand{\diag}[1]{\regtext{diag}\!\left(#1\right)}
\newcommand{\union}{\bigcup}

\newcommand{\trans}{^\top}

\newcommand{\tp}{^\intercal} 

\newcommand{\dyn}{\vc{f}}
\newcommand{\dyng}{\vc{g}}

\newcommand{\final}{\lbl{f}}
\newcommand{\other}{\lbl{other}}
\newcommand{\noneti}{\lbl{nonETI}}

\newcommand{\track}{\lbl{track}}
\newcommand{\plan}{\lbl{plan}}
\newcommand{\eti}{\lbl{ETI}}
\newcommand{\plant}{_{\regtext{plan}}}
\newcommand{\feedback}{\lbl{fb}}
\newcommand{\hpoly}{\mathcal{P}}
\newcommand{\conv}[1]{\regtext{conv}\!\left(#1\right)}
\newcommand{\slice}[2][]{\regtext{slice}_{#1}\!\left(#2\right)}
\newcommand{\slicedInt}[1]{\regtext{slicedTraj}\!\left(#1\right)}
\newcommand{\brsfunction}{\mathcal{B}}
\newcommand{\brs}[1]{\brsfunction\!\left(#1\right)}

\newcommand{\hatmap}{\regtext{hat}}

\newcommand{\interior}[1]{\regtext{int}\!\left(#1\right)}

\newcommand{\ts}{t}

\newcommand{\is}{i}
\newcommand{\js}{j}
\newcommand{\ks}{k}

\newcommand{\gams}{\gamma}
\newcommand{\sap}{s}
\newcommand{\tf}{\ts\final}
\newcommand{\es}{e}

\newcommand{\xv}{\vc{x}}
\newcommand{\yv}{\vc{y}}
\newcommand{\wv}{\vc{w}}

\newcommand{\trackv}{\vc{z}}
\newcommand{\planv}{\vc{p}}
\newcommand{\paramv}{\vc{k}}

\newcommand{\Acon}{\vc{A}}
\newcommand{\bcon}{\vc{b}}
\newcommand{\Ccon}{\vc{C}}
\newcommand{\dcon}{\vc{d}}
\newcommand{\uv}{\vc{u}}

\newcommand{\px}{p_x}
\newcommand{\py}{p_y}
\newcommand{\pz}{p_z}
\newcommand{\pth}{\theta}
\newcommand{\pthx}{\theta_x}
\newcommand{\pthy}{\theta_y}

\newcommand{\des}{_\regtext{des}}

\newcommand{\vx}{v_x}
\newcommand{\vy}{v_y}
\newcommand{\vz}{v_z}
\newcommand{\vth}{\omega}
\newcommand{\vthx}{\omega_x}
\newcommand{\vthy}{\omega_y}

\newcommand{\bound}{_\regtext{bound}}

\newcommand{\dl}{\delta}
\newcommand{\fxf}{{f_{x,\regtext{f}}}}
\newcommand{\fxr}{{f_{x,\regtext{r}}}}
\newcommand{\fyf}{{f_{y,\regtext{f}}}}
\newcommand{\fyr}{{f_{y,\regtext{r}}}}
\newcommand{\La}{L_a}
\newcommand{\Lb}{L_b}
\newcommand{\Iz}{I_z}
\newcommand{\mc}{m}
\newcommand{\Bc}{B}
\newcommand{\vel}{v}
\newcommand{\sideslip}{\beta}
\newcommand{\cparam}{c}
\newcommand{\dparam}{d}

\newcommand{\mq}{m}
\newcommand{\radq}{r}
\newcommand{\Iq}{I}
\newcommand{\flf}{{F\lbl{l}}}
\newcommand{\frt}{{F\lbl{r}}}

\newcommand{\nhkt}{k_T} 
\newcommand{\nhn}{n_0}
\newcommand{\nhd}{d_0}
\newcommand{\nhdd}{d_1}
\newcommand{\nhax}{\alpha_x}
\newcommand{\nhay}{\alpha_y}
\newcommand{\nhaz}{\alpha_z}

\newcommand{\cf}{c}
\newcommand{\pk}{\regtext{pk}}
\newcommand{\ka}{k_a}
\newcommand{\kv}{k_v}
\newcommand{\kpk}{k_{pk}}
\newcommand{\kax}{k_{ax}}
\newcommand{\kvx}{k_{vx}}
\newcommand{\kpkx}{k_{\pk x}}
\newcommand{\kay}{k_{ay}}
\newcommand{\kvy}{k_{vy}}
\newcommand{\kpky}{k_{\pk y}}
\newcommand{\kaz}{k_{az}}
\newcommand{\kvz}{k_{vz}}
\newcommand{\kpkz}{k_{\pk z}}
\newcommand{\tpk}{\ts_{\pk}}

\newcommand{\vv}{\vc{v}}
\newcommand{\Rm}{\vc{R}}
\newcommand{\eunit}{\vc{e}_3}
\newcommand{\mass}{m}
\newcommand{\grav}{g}
\newcommand{\thrust}{\tau}
\newcommand{\moi}{\vc{J}}
\newcommand{\angv}{\boldsymbol{\omega}}

\newcommand{\ndim}{{n}}

\newcommand{\pdim}{{p}}
\newcommand{\qdim}{{q}}

\newcommand{\nstate}{\ndim_z}
\newcommand{\nparam}{\ndim_k}

\newcommand{\nlow}{\ndim_x}
\newcommand{\nwork}{\ndim_w}
\newcommand{\nplan}{\ndim_p}

\newcommand{\npwat}{\ndim_{\regtext{PWA}, \ts}}

\newcommand{\nhp}{\ndim\lbl{h}}
\newcommand{\nobs}{\ndim_{\Obs}}

\newcommand{\goal}{\regtext{goal}}

\newcommand{\sample}{\regtext{sample}}
\newcommand{\data}{\regtext{data}}
\newcommand{\mini}{\regtext{min}}
\newcommand{\maxi}{\regtext{max}}

\newcommand{\lin}{^{*}}

\title{\LARGE \bf
Goal-Reaching Trajectory Design Near Danger with \\ Piecewise Affine Reach-avoid Computation
}

\author{Long Kiu Chung*, Wonsuhk Jung*, Chuizheng Kong, and Shreyas Kousik
\thanks{
\textbf{*} indicates the equal contribution.
All authors are with the School of Mechanical Engineering at the Georgia Institute of Technology, Atlanta, GA.
(e-mail: \texttt{\{lchung33, wonsuhk.jung, ckong35\}@gatech.edu; shreyas.kousik@me.gatech.edu}).
\textbf{Website:} \href{https://saferoboticslab.me.gatech.edu/research/parc/}{https://saferoboticslab.me.gatech.edu/research/parc/},    $\quad\quad\quad\quad$ $\quad \quad$    \textbf{GitHub:} \quad \href{  https://github.com/safe-robotics-lab-gt/PARC}{https://github.com/safe-robotics-lab-gt/PARC}
}}

\begin{document}

\maketitle

\thispagestyle{plain}
\pagestyle{plain}

\begin{abstract}
Autonomous mobile robots must maintain safety, but should not sacrifice performance, leading to the classical \textit{reach-avoid} problem: find a trajectory that is guaranteed to \textit{reach} a goal and \textit{avoid} obstacles.
This paper addresses the \textit{near danger} case, also known as a \textit{narrow gap}, where the agent starts near the goal, but must navigate through tight obstacles that block its path.
The proposed method builds off the common approach of using a simplified planning model to generate plans, which are then tracked using a high-fidelity tracking model and controller.
Existing approaches use reachability analysis to overapproximate the error between these models and ensure safety, but doing so introduces numerical approximation error conservativeness that prevents goal-reaching.
The present work instead proposes a Piecewise Affine Reach-avoid Computation (PARC) method to tightly approximate the reachable set of the planning model.
PARC significantly reduces conservativeness through a careful choice of the planning model and set representation, along with an effective approach to handling time-varying tracking errors.
The utility of this method is demonstrated through extensive numerical experiments in which PARC outperforms state-of-the-art reach avoid methods in near-danger goal reaching.
Furthermore, in a simulated demonstration, PARC enables the generation of provably-safe extreme vehicle dynamics drift parking maneuvers.
A preliminary hardware demo on a TurtleBot3 also validates the method.
\end{abstract}

\section{Introduction}

The increasing deployment of autonomous robots in daily tasks necessitates addressing complex operational challenges.
Many of these tasks can be conceptualized as \textit{reach-avoid} problems, wherein a robot must steer its states into a set of goal states without violating safety-critical constraints \cite{fisac2015reach}. 
Achieving both \textit{liveness}---the assurance of goal fulfillment \cite{alpern1985defining}---and adherence to safety constraints is pivotal in these scenarios.
In this paper, we are particularly interested in collision-free goal reaching \textit{near danger}, meaning that the agent starts near the goal, but must navigate past tightly-spaced obstacles, typically less than the width of a robot's body from the goal.
We assume an alternative approach (e.g., \cite{kousik2019safe,kousik2020bridging,chen2021fastrack}) has safely brought the robot near the goal but cannot complete the final goal-reaching maneuver.
As we will show, this problem setup is very difficult for existing robust planning and control approaches to provide reach-avoid guarantees on general dynamical systems due to over-conservativeness.
We note that the specific ``near danger'' condition where these planners and controllers fail is task-dependent.
Some examples include flying a quadrotor drone through a narrow gap, or parallel parking an autonomous car.

\subsection{Reach-Avoid Approaches and Challenges}
The reach-avoid problem is typically solved with one of two approaches, both based on reachability analysis or computing reachable sets of a system's trajectories.
Note that an extended literature review is presented in Appendix \ref{app:related_work}.
\begin{figure}[t]
\centering
\begin{subfigure}[t]{0.99\columnwidth}
    \centering
    \includegraphics[width=0.99\columnwidth]{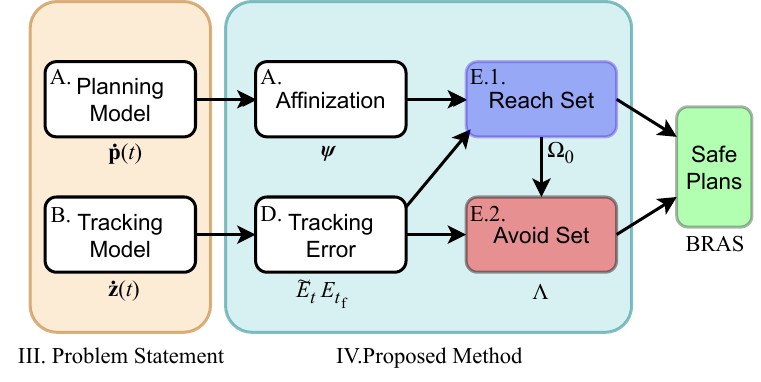}
    \caption{}
    \label{fig:flow_chat}
\end{subfigure}
\begin{subfigure}[t]{0.99\columnwidth}
    \includegraphics[width=0.99\columnwidth]{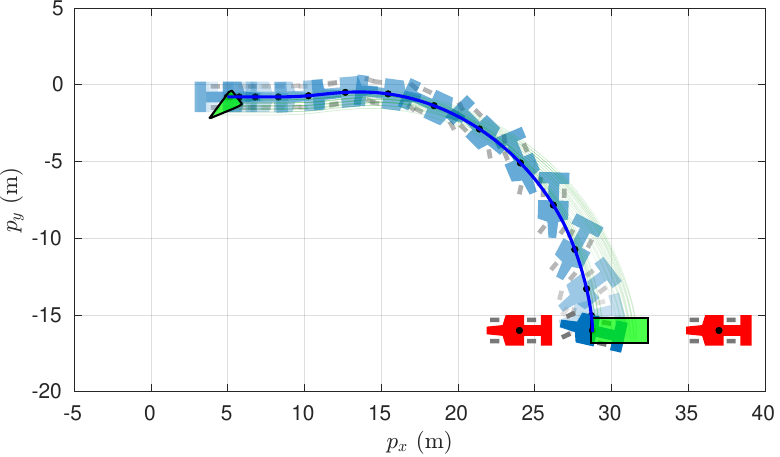}
    \caption{}
    \label{fig:drift_BRAS}
\end{subfigure}
\caption{
Our overall goal is to find safe goal-reaching motion plans near danger.
In (a), we show the components of our proposed Piecewise Affine Reach-avoid Computation (PARC) method, with the paper section number on the top left, and the corresponding symbol on the bottom.
This method enables a new state-of-the-art for safe reach-avoid problems, as evidenced by (b), which shows examples of extreme vehicle maneuvers.
These planned drift parking trajectories are sampled from the Backward Reach-Avoid Set (BRAS, green triangle on the top left) using PARC.
Our vehicle is shown with a blue timelapse.
The two red cars are obstacles. 
The goal set for the center of mass is the green rectangle.
The small size of the BRAS indicates the challenge of finding safe trajectories.
}
\label{fig:front_figure}
\end{figure}
The first approach is \textit{control intervention}, wherein a motion planner proposes actions, while a safety filter \cite{wabersich2023data,hsu2023safety} prevents the robot from executing actions for which its reachable set would intersect obstacles.
This approach often requires compensating for the worst-case behavior of the motion planner \cite{chen2021fastrack,garg2021multi}, which we find leads to conservativeness (i.e., sacrificing liveness) in our experiments (Section \ref{sec:experiments}).
The second approach is a \textit{planner-tracker} framework, wherein a motion planner proposes plans that are safe by leveraging the knowledge of a low-level controller's tracking ability.
This controller can be synthesized \cite{majumdar2017funnel,tedrake2010lqr} or hand-engineered \cite{kousik2020bridging,kousik2019safe}; the key difference is that the planner need not be treated as adversarial, which reduces conservativeness.
However, this approach typically limits the space of possible plans due to the need to compute the robot's tracking ability.
Furthermore, one must compensate for \textit{tracking error} between the simplified \textit{planning model} and the high-fidelity \textit{tracking model} (i.e., the robot cannot perfectly track plans).

Both control intervention and planner-tracker approaches suffer from challenges in numerical representation due to the need to \textit{underapproximate} the set of states that reach the target set and \textit{overapproximate} those that avoids obstacles when guaranteeing safety and liveness in implementation.
Implementations such as sums-of-squares polynomials \cite{majumdar2017funnel,kousik2020bridging,landry2018reach}, grid-based partial differential equation (PDE) solutions \cite{chen2021fastrack}, or set-based methods \cite{kousik2019safe,althoff2021set} all introduce additional conservativeness purely due to compensating for numerical representation error.
This challenge is exacerbated for systems with high-dimensional, nonlinear dynamics or observations; machine learning can address some of these concerns \cite{bansal2021deepreach,michaux2023reachability,dawson2022safe,xiao2023barriernet,selim2022safe}, at the expense of losing guarantees.

\subsection{Proposed Method Summary and Context}
In this work, we leverage two key insights to perform goal-reaching near danger.
First, we follow the literature \cite{kousik2020bridging,kousik2019safe} in using a planner-tracker framework to reduce conservativeness by treating the planner and tracking controller as cooperative, as opposed to adversarial.
Second, by representing our motion planner as a \textit{piecewise affine} (PWA) system \cite{christophersen2007optimal}, we \textit{tightly approximate} its reachable set, with minimal conservativeness from numerical approximation error.
We demonstrate how this unlocks an efficient and less conservative approach to a variety of reach-avoid problems from the robotics literature, and provides the solution to a reachable set computation for drift parking vehicle dynamics.
An overview and example of our approach is shown in Fig.~\ref{fig:front_figure}.

Our method fits in the context of both control and planning.
We build off of a long history of discrete-time PWA system analysis in optimal control  (e.g., \cite{thomas2006robust,kerrigan2002optimal,rakovic2005robust}), in which we extend to the continuous-time reach-avoid problem for robot motion planning.
Furthermore, our approach is complementary to sampling-based motion planners such as rapidly exploring random tree (RRT) \cite{lavalle1998rapidly}; our approach could be applied in the future to compute a safety funnel around a sampled motion plan, or we could leverage sampling to seed our reachability analysis.

\subsection{Contributions}
\begin{enumerate}
    \item We propose an efficient, parallelizable method for Piecewise Affine Reach-avoid Computation (PARC) to underapproximate the Backwards Reach-Avoid Set (BRAS) of a planning model represented as a low-dimensional discrete-time, time-variant PWA system.
    Our method extends to discrete-time, time-variant affine systems as a particular case.
    
    \item We propose a scheme to incorporate tracking error into PARC to compute a subset of the \textit{reach set} and overapproximate the \textit{avoid set} for a trajectory planning model.
    This allows us to represent the BRAS that compactly represents all trajectory plans that reach the goal and avoid obstacles when tracked by the robot, represented by a high-fidelity dynamical model.
    
    \item We stress-test PARC on examples drawn from the safe control and planning literature: a 4-D TurtleBot, a 6-D planar quadrotor, a 10-D near-hover quadrotor, and a 13-D general quadrotor.
    We show that PARC outperforms state-of-the-art methods such as Fast and Safe Tracking (FaSTrack) \cite{chen2021fastrack}, Neural Control Lyapunov Barrier Function (Neural CLBF) \cite{dawson2022safe}, Reachability-based Trajectory Design (RTD) \cite{kousik2020bridging,kousik2019safe}, and KinoDynamic motion planning via Funnel control (KDF) \cite{verginis2022kdf} in \textit{near danger} reach-avoid tasks due to reduced conservativeness from the combination of planning model and set representation choice and better implementation of tracking error.

    \item We demonstrate PARC on a simulated 6-D vehicle model to plan safe, extreme drift parallel parking maneuvers, and on a 4-D physical TurtleBot as a preliminary hardware experiment.
\end{enumerate}

\textit{Paper Organization:}
Next, we introduce mathematical objects used throughout our formulation (Section \ref{sec:preliminaries}).
We then present our problem formulation (Section \ref{sec:problem_formulation}) and proposed approach (Section \ref{sec:method}).
We assess PARC with numerical experiments (Section \ref{sec:experiments}), then demonstrate a drift parking extreme vehicle dynamics example in simulation, extending beyond the state of the art in reachability-based safe motion-planning and control methods (Section \ref{sec:drift_demo}).
We also showcase PARC's ability to extend to real robots with a preliminary hardware experiment (Section \ref{sec:hardware_demo}).
Finally, we discuss takeaways and limitations in Section \ref{sec:conclusion}.
For presentation's sake, we move most proofs and experiment details to appendices, and call out \textbf{key results} in the text.
\section{Preliminaries}\label{sec:preliminaries}
We now introduce our notation conventions and define H-polytopes and PWA system, which we use to represent sets and dynamics of planning model throughout the paper.
We then present a one-step Backward Reachable Set (BRS) computation that is the core of our proposed PARC method.

\subsection{Notation}
The real numbers are $\R$ and the natural numbers are $\N$.
Scalars are lowercase italic and sets are in uppercase italic (e.g., $a \in A \subset \R^n$).
Vectors are lowercase bold, and matrices are uppercase bold (e.g., $\vc{y} = \vc{A}\vc{x}$).
An $n\times m$ matrix of zeros is $\zeros_{n\times m}$, and an $n\times n$ identity matrix is $\eye_{n}$.
The $p^{\regtext{th}}$ to $q^{\regtext{th}}$ dimensions of vector $\xv$ and the submatrix of $\vc{A}$ defined by row dimensions $p_1$:$q_1$ and column dimensions $p_2$:$q_2$ are denoted as $(\vc{x})_{p:q}$ and $(\vc{A})_{p_1:q_1, p_2:q_2}$, respectively.

\subsection{H-Polytopes}

In this work, we represent all sets as H-polytopes or unions of H-polytopes, which have extensive algorithm and toolbox support \cite{herceg2013multi,althoff2015introduction}.
An $\ndim$-dimensional H-polytope $\hpoly(\Acon, \bcon) \subset \R^{\ndim}$ is a convex set parameterized by $\nhp$ linear constraints $\Acon \in \R^{\nhp \times \ndim}$ and $\bcon \in \R^{\nhp}$:
\begin{align}
    \hpoly(\Acon, \bcon) &= \left\{\xv\ |\ \Acon\xv\leq\bcon\right\}.
\end{align}
Unless otherwise stated, we assume all H-polytopes are compact.

Consider a pair of H-polytopes $\Poly_1 = \hpoly(\Acon_1, \bcon_1)$, $\Poly_2 = \hpoly(\Acon_2, \bcon_2)$.
We make use of intersections $\Poly_1 \cap \Poly_2$, Minkowski sums ($\Poly_1 \oplus \Poly_2 = \left\{\xv + \yv\ |\ \xv \in \Poly_1, \yv \in \Poly_2\right\}$), Pontryagin differences ($\Poly_1 \ominus \Poly_2 = \left\{\xv \in \Poly_1\ |\ \xv + \yv \in \Poly_1\ \forall\ \yv \in \Poly_2\right\}$), set differences $\Poly_1 \setminus \Poly_2$, Cartesian products $\Poly_1 \times \Poly_2$, and convex hulls $\conv{\Poly_1,\Poly_2}$.
We also project an $n$-dimensional polytope $\Poly$ into dimensions $p:q$ as $\proj[\pdim:\qdim]{\Poly}$, and slice it in the dimensions $p:q$ with respect to some constant $\xv_0\in\R^{q-p+1}$ as $\slice[\pdim:\qdim]{\Poly, \xv_0}$.
These operations are detailed in Appendix \ref{app:h_poly_ops}.

\subsection{PWA Systems}
In this section, we formally define the discrete-time, time-variant PWA system, which becomes the main objective of our reachability analysis.

The state of the system, $\xv(\ts) \in \inset \subset \R^{\ndim}$, evolves at discrete times $\ts = 0, \Delta\ts, \cdots, \tf - \Delta\ts$ according to the difference equation:
\begin{align}\label{eq:pwa_dyn}
\xv(\ts + \Delta\ts) = \Ccon_{\ell,\ts} \xv(\ts) + \dcon_{\ell,\ts},
\end{align}
where $\Ccon_{\ell,\ts} \in \R^{\ndim \times \ndim}$ and $\dcon_{\ell,\ts} \in \R^{\ndim}$ define the dynamics of the $\ell^{\regtext{th}}$ PWA region at time $\ts$.
The index $\ell$ is given by
\begin{align} \label{eq:pwa_idx}
    \ell &= \min \{\is \ |\ \xv(\ts) \in \hpoly(\Acon_{\is,\ts}, \bcon_{\is,\ts}), \is=1,\dots,\npwat\},
\end{align}
where $\hpoly(\Acon_{\is,\ts}, \bcon_{\is,\ts})$ is the $\is^{\regtext{th}}$ of the $\npwat$ PWA regions at time $\ts$.
We ensure that the PWA regions form a \textit{tessellation} of the domain $\inset$ at each time step $\ts$, meaning their union covers the entire domain and their interiors do not intersect:
\begin{align}\label{eq:bound_con}
    \inset = \union_{\is = 1}^{\npwat} &\hpoly(\Acon_{\is,\ts}, \bcon_{\is,\ts}) \\
    \left\{\xv\ |\ \Acon_{\is, \ts} \xv < \bcon_{\is, \ts}\right\} \cap &\hpoly(\Acon_{\js, \ts}, \bcon_{\js, \ts}) = \emptyset,
\end{align}
for all $\ts = 0, \Delta\ts, \cdots, \tf - \Delta\ts$ and $\is, \js \in \{1,\dots,\npwat\}$, $\is \neq \js$.
With this condition and the decision rule \eqref{eq:pwa_idx}, the PWA-system has well-defined dynamics on the domain $\inset$.
Hence, for a given initial state $\xv_0$, the PWA system has a unique state at each time $\ts$, which we denote as $\xv(\ts; \xv_0)$.
We drop $\xv_0$ and simplify it to $\xv(\ts)$ when it is contextually evident.

We say that the state $\xv(\ts)$ has the \textit{mode} $\sap_\ts = \ell$ if $\xv(\ts) \in \hpoly(\Acon_{\ell, \ts}, \bcon_{\ell, \ts})$.
Further, if the mode of $\xv(0), \xv(\Delta\ts), \cdots, \xv(\tf - \Delta\ts)$ is $\sap_0, \sap_{\Delta\ts}, \cdots, \sap_{\tf - \Delta\ts}$ for some input $\xv(0) \in \inset$, we denote the \textit{mode sequence} for $\xv(0)$ as:
\begin{align}\label{eq:mode_sequence}
    \Sap = (\sap_0, \sap_{\Delta\ts}, \cdots, \sap_{\tf - \Delta\ts}).
\end{align}
For all $\xv(0)$ that share the same mode sequence,
the dynamics reduce from time-variant PWA to time-variant affine up to time $\tf$.
Mathematically, the discrete-time, time-variant affine dynamics with respect to the mode sequence \eqref{eq:mode_sequence} is:
\begin{align}\begin{split}\label{eq:time_variant_affine}
    &\xv(\ts + \Delta\ts) = \Ccon_{\sap_{\ts},\ts} \xv(\ts) + \dcon_{\sap_{\ts},\ts}\\
    &\forall\ \xv(\ts) \in \hpoly(\Acon_{\sap_{\ts},\ts}, \bcon_{\sap_{\ts},\ts}),
\end{split}\end{align}
where $\ts = 0, \Delta\ts, \cdots, \tf - \Delta\ts$.

The reduction of a PWA system to specific affine dynamics is the core idea that enables the efficient underapproximating BRS computation of PARC.
Note, taking advantage of fixed mode sequences is well-established in robotics and controls (e.g., \cite{thomas2006robust,desimini2020robust,kerrigan2002optimal,wensing2023optimization}), but we believe we are the first to show its utility in the safe reach-avoid computation of robot planning models.

Finally, to ensure safety for realistic systems, we must extend our system definition from discrete to continuous time.
We do so with linear interpolation:
at each $\ts = 0, \Delta\ts, \cdots, \tf - \Delta\ts$, $\xv(\ts')$ for $\ts < \ts' < \ts + \Delta\ts$ is defined by
\begin{align}\label{eq:cont_time_approx}
    \xv(\ts') = \xv(\ts) + (\xv(\ts + \Delta\ts) - \xv(\ts))(\tfrac{\ts' - \ts}{\Delta\ts}).
\end{align}
As will be shown in Section \ref{subsec:parc_avoid}, the use of linear interpolation to bridge between discrete and continuous time is necessary to enable avoid set computation with H-polytopes.
Similar reachability techniques to ours have been shown to work with Bernstein polynomials \cite{michaux2023can}, but only for forward reachability, to the best of our knowledge.
Nonetheless, it may still be possible for PARC to handle alternative interpolation methods such as cubic splines or B\'ezier curves by using other set representations such as polynomial zonotopes \cite{kochdumper2020sparse}.
In this paper, we opted for the most straightforward approach of linear interpolation and leave investigation of other interpolation methods as future work.

\subsection{One Step BRS via Inverse Affine Map}\label{subsec:one_step_brs}

Finally, we take advantage of the H-polytope analytic solution for inverse affine maps to enable BRS computation using the following well-known lemma (c.f., \cite{thomas2006robust,vincent2021reachable,herceg2013multi}). 
\begin{lem}\label{lem:one_step_BRS}
Consider an affine system $\xv(\ts+\Delta\ts) = \Ccon\xv(\ts)+\dcon$.
Define the one-step BRS as the map
$\brs{\hpoly(\Acon, \bcon), \Ccon, \dcon} = \left\{\xv |\ \Acon\xv_{\regtext{next}}\leq\bcon,\ \xv_{\regtext{next}} = \Ccon\xv+\dcon\right\}$.
This one-step BRS is in fact an H-polytope:
\begin{align}\label{eq:1stepbrs}
\begin{split}
    \brs{\hpoly(\Acon, \bcon), \Ccon, \dcon} = \hpoly(\Acon\Ccon, \bcon-\Acon\dcon).
\end{split}
\end{align}
\end{lem}
\begin{proof}
\begin{align}
    \brs{\hpoly(\Acon, \bcon), \Ccon, \dcon}
    &= \left\{\xv |\ \Acon(\Ccon\xv+\dcon)\leq\bcon\right\},\\
    &= \hpoly(\Acon\Ccon, \bcon-\Acon\dcon).
\end{align}
\end{proof}
\noindent We take advantage of the fact that repeated applications of Lemma \ref{lem:one_step_BRS} to an H-polytope still produce an H-polytope.
Notice that \eqref{eq:1stepbrs} does not require invertibility of $\Ccon$, conferring an advantage over other set polytope representations such as constrained zonotopes \cite{scott2016constrained}.
This makes H-polytopes an excellent choice for representing the BRS.
\section{Problem Statement}
\label{sec:problem_formulation}

Our overall goal is to design a reach-avoid trajectory using a simple \textit{planning model}, which enables computational efficiency while also accounting for the tracking error when the plan is tracked by a robot, represented by a \textit{tracking model}.
We now describe each of these models and the problem we seek to solve.

\subsection{Planning Model}
A planning model provides a low-fidelity but easy-to-compute model of the robot's dynamics.
In this work, we use parameterized reference trajectories, which can allow quick computation \textit{and} strong safety guarantees \cite{kousik2020bridging,kousik2019safe,hess2014formal}. 

Before defining the model, we define a workspace state $\wv \in \Workspace \subset \R^{\ndim_w}$ (i.e., $\R^2$ or $\R^3$) that represents the position of the robot. The planning state which we denote as $\planv \in \Xplan \subset \R^{\nplan}$ is chosen to include the workspace state, such that $\planv = [\wv\tp, \planv\other\tp]\tp$ where $\planv\other \in \Xplan\other \subset \R^{{\nplan}\other}$ represents states such as heading, velocity, etc.

We then define the planning model as
\begin{equation}
    \label{eq:planning-model}
    \dot\planv(\ts) = \dyn\plan(\ts, \planv(\ts), \paramv), \quad \planv(0) = \planv_0
\end{equation}
where $\planv_0$ $\in$ $\Xplan \subset \R^{\nplan}$ defines the \textit{initial planning state}, $\paramv \in \K \subset \R^{\nparam}$ represents a \textit{trajectory parameter}, and $\ts \in [0,\tf]$ is the time horizon. 
A solution $\planv: [0, \tf] \rightarrow \Xplan$ is referred to as a \textit{plan}, or equivalently, \textit{reference trajectory}.
For a given initial planning state and trajectory parameter, the plan at time $\ts$ is $\planv(\ts; \planv_0, \paramv)$.
We restrict the class of planning models to enable computing reachable sets for parameterized planning models:
\begin{assum}[Extended Translation Invariance (ETI) in Workspace]
\label{ass:eti}
Changing the workspace states either does not affect the planning dynamics $\dyn\plan$ in the workspace or else the dynamics are constant:
\begin{align}
    \norm{\tfrac{\partial\dyn_{\wv}}{\partial\wv}\dyn_{\wv}}_2 = 
    \norm{\tfrac{\partial\dyn_{\wv}}{\partial\wv}\tfrac{\partial\dyn_{\wv}}{\partial\planv}}\lbl{F} = 
    \norm{\tfrac{\partial\dyn_{\wv}}{\partial\wv}\tfrac{\partial\dyn_{\wv}}{\partial\paramv}}\lbl{F} = 0
\end{align}
where $\dyn_{\wv} = ({\dyn\plan})_{1:\ndim_w}$ and $\norm{\cdot}\lbl{F}$ indicates the Frobenius norm.
\end{assum}
\noindent This is not overly restrictive; a variety of workspace-ETI planning models are shown in Section \ref{sec:experiments}.

\subsection{Tracking Model}

The tracking model is a high-fidelity model of robot dynamics, denoted
\begin{equation}
    \label{eq:tracking-model}
    \dot{\trackv}(\ts) = \dyn\track({\trackv(\ts)},{\uv\feedback(\ts)}), ~~~ \trackv(0)=\trackv_0,
\end{equation}
where $\trackv_0 \in \Xtrack \subset \R^{\nstate}$ is the initial state.
The tracking state $\trackv$ includes the planning state $\planv$, i.e., $\trackv=[\planv\tp, \trackv\other\tp]\tp$.
Given a plan $\planv$ parameterized by $\paramv$, we use a feedback tracking controller $\uv\feedback(\ts, \trackv(t); \paramv) \in \U$ (e.g., proportional–integral–derivative (PID) or model predictive control (MPC)) to control the tracking model.
We assume $\Xtrack$ and $\U$ are compact and the dynamics are Lipschitz such that trajectories uniquely exist \cite{grant2014theory}.
For a particular $\trackv_0$, we denote the resulting \textit{realized  trajectory} (i.e., solution to \eqref{eq:tracking-model}) as $\trackv(\ts; \trackv_0, \paramv)$ at time $\ts$.

\begin{rem}
    Designing the feedback controller $\uv\feedback$ is not the focus of this work.
    Conveniently, specifying a parameterized (i.e., restricted) space of plans $\planv$ simplifies controller design, as we find experimentally in Section~\ref{exp:time-varying}.
    Our overall goal is to compensate for the fact that the controller is imperfect.
\end{rem}

\subsection{Reach-Avoid Problem}
We now define our reach-avoid problem: find all plans for which the tracking model safely reaches a goal set at time $\tf$.
We denote the goal $\ol{\Xgoal} \subset \Workspace$ and obstacles $\ol{\Obs} \subset \Workspace$; the overline indicates these are restricted to the workspace, which we modify later to enable the proposed method.

Note, we only consider static obstacles, because our goal is to plan a robot's motion near danger.
The danger resulting from moving or adversarial obstacles depends upon accurately predicting obstacle motion, which introduces an independent variable unrelated to our purpose.
We leave an extension to dynamic scenes for future work, but we anticipate that it can be accomplished by adapting reachability-based methods such as \cite{vaskov2019towards,vaskov2019not}.

\begin{problem}[Backward Reach-Avoid Set (BRAS)]
\label{prob:reach-avoid-set}
Given the the planning model \eqref{eq:planning-model}, tracking model \eqref{eq:tracking-model},  goal set $\ol{\Xgoal}$, obstacles $\ol{\Obs}$, and a fixed time $\tf$, find the set of planning models states and trajectory parameters for which the robot reaches $\ol{\Xgoal}$ exactly at time $\tf$ while avoiding collisions with obstacles $\ol{\Obs}$ :
\begin{equation}\label{eq:bras_formulation}
\begin{split}
\regtext{BRAS}(\tf, \ol{\Xgoal}, \ol{\Obs}) =
        \{&(\planv_0, \paramv) \in \Xplan \times \K\ \mid
        \forall~\trackv_0~\regtext{s.t.}\\
    &(\trackv_0)_{1:\nplan} = \planv_0, \\
    &(\trackv(\tf; \trackv_0, 
    \paramv))_{1:\ndim_w} \in \ol{\Xgoal},~\regtext{and} \\
    &(\trackv(\tau; \trackv_0, \paramv))_{1:\ndim_w}\notin \ol{\Obs} \quad \forall\tau \in [0, \tf]\}
\end{split}
\end{equation}
\end{problem}

Several remarks are in order to discuss our fixed time horizon, choice of goals and obstacles, and choice of tracking error model.

\begin{rem}[Fixed Final Time]
\label{rem:fixed final time}
    A fixed time horizon is reasonable because we seek to compute a trajectory near danger; the robot is already near its goal, but reaching the goal is difficult.
    The BRAS is ideally a large set of initial conditions far from obstacles that can be used as a target set for other methods to reach safely.
\end{rem}

\begin{rem}[Goal and Obstacles]
\label{rem:goal set expansion}
    We restrict the goal and obstacles to the workspace only to simplify exposition; we demonstrate more general goals and obstacles, such as in the robot's orientation, in Sections \ref{sec:drift_demo}.
\end{rem}

\begin{rem}[Disturbances to the System]
To simplify exposition, we treat tracking error as a lumped uncertainty representation between planning and control, encompassing multiple error sources---disturbances, state estimation, object detection, and trajectory prediction.
Separately analyzing these sources would introduce confounding variables, complicating the interpretation of our results.
Instead, we consolidate them into a single metric of tracking error.
For example, one can design a robust tracking controller that yields a tracking error bound, robust to external disturbances and model uncertainties \cite{chen2021fastrack, bechlioulis2014low}.
Our planning level only considers the result of those errors as a tracking error bound without separately accounting for each.
Alternatively, if the bounds of inertial uncertainties \cite{michaux2023can}, dynamic disturbance \cite{selim2022safe, alanwar2023data}, and perception uncertainty \cite{kousik2020bridging} are known (which are strong assumptions), then planner-tracker frameworks similar to ours have been shown to maintain reach-avoid guarantees in the past by ``buffering'' the bounds into the tracking error \cite{kousik2020bridging,vaskov2019towards,selim2022safe,chen2021fastrack}.
In Section \ref{sec:hardware_demo}, we show the preliminary results of a hardware experiment that successfully includes localization drift from the inertial measurement unit (IMU) and the encoders.\end{rem}

\subsection{Example System: TurtleBot}\label{subsec:turtlebot_example}

We use the TurtleBot unicycle model as an illustrative example from the safe planning and control literature \cite{chen2021fastrack,kousik2020bridging}.
Our planning model is a Dubins car:
\begin{align} \label{eq:dubinscar}
    \dot\planv = \frac{d}{dt}\begin{bmatrix}
        \px \\ \py \\ \theta
    \end{bmatrix}
    = \begin{bmatrix}
        v\des\cos\theta \\
        v\des\sin\theta \\
        \omega\des
    \end{bmatrix},
\end{align}
with workspace (position) states $\wv = [\px,\py]\tp$, other (heading) state $\planv\other = \theta$, and trajectory parameters $\paramv = [v\des,\omega\des]\tp$ representing desired linear velocity and angular velocity.
Note that a Dubins car model \eqref{eq:dubinscar} is ETI in the workspace (Assum. \ref{ass:eti}) since changing the workspace states does not affect the planning dynamics (i.e. $\norm{\tfrac{\partial\dyn_{\wv}}{\partial\wv}\dyn_{\wv}}_2 = 
    \norm{\tfrac{\partial\dyn_{\wv}}{\partial\wv}\tfrac{\partial\dyn_{\wv}}{\partial\planv}}\lbl{F} = 
    \norm{\tfrac{\partial\dyn_{\wv}}{\partial\wv}\tfrac{\partial\dyn_{\wv}}{\partial\paramv}}\lbl{F} = 0$.)

We use the following tracking model:
\begin{align}
\dot\trackv = \frac{d}{dt}\begin{bmatrix}
    p_x \\
    p_y \\
    \theta \\
    v
\end{bmatrix} = \begin{bmatrix}
    v\cos \theta \\
    v\sin \theta \\
    u_\omega \\
    u_v
\end{bmatrix},
\end{align}
where $\trackv\other = [\theta, v]\tp$ are heading and linear velocity, and $\uv\feedback = [u_\omega, u_v]\tp$ is implemented as an iterative linear quadratic regulator (iLQR) tracking controller \cite{kong2021ilqr} that commands angular velocity and linear acceleration.

We define a goal set as a unit box, $\ol{\Xgoal} = \{\wv \mid p_x, p_y \in [-1,1]\}$, and an obstacle near the goal, $\ol{\Obs} = \{\wv \mid p_x \in [-1.75, -1.25], p_y \in [-0.25, 0.25] \}$ (see Fig. \ref{fig:turtlebot-reachset} and \ref{fig:avoid_set_turtlebot}).
\section{Proposed Method}\label{sec:method}

In this section, we detail our approach to solving the BRAS problem.
First, we convert the planning model to a PWA system to enable our Piecewise Affine Reach-Avoid Computation (PARC).
We then adapt the method in \cite{thomas2006robust} to compute a subset of the reach set of the PWA planning model \textit{without} tracking error, to simplify exposition.
Next, we propose a novel approach to \textit{overapproximate} the \textit{continuous time} avoid set of the PWA system without tracking error (\textbf{Key result:} Theorem \ref{thm:avoid_set}).
To enable incorporating tracking error between the planned and realized trajectories, we adapt the methods in \cite{kousik2019safe,kousik2020bridging} to estimate time-varying tracking error.
Finally, we propose PARC to incorporate tracking error and compute the BRAS (\textbf{Key result:} Theorem \ref{thm:avoid_set_real}), enabling provably-safe goal-reaching.

\subsection{Reformulation to Enable PARC}\label{subsec:PWA_convert}

\subsubsection{Augmented State, Goal, and Obstacles}
Before proceeding, to enable PARC, we define an \textit{augmented planning state} $\xv(\ts) \in \Xpoly = \Workspace\times\K\times\Xplan\other \subset \R^{\nwork + \nparam + {\nplan}\other} = \R^{\nlow}$ at time $\ts$ as:
\begin{align}\label{eq:x_order}
    \xv(\ts) = \begin{bmatrix}
        \wv(\ts)\\
        \paramv\\
        \planv\other(\ts)
    \end{bmatrix}.
\end{align}
We perform this reordering and stacking because, as we will see later in Lem. \ref{lem:pwa_eti}, $\wv$ and $\paramv$ are ETI states, so indexing them as the first states of $\xv$ will simplify avoid set computation.
We similarly augment the goal set $\ol{\Xgoal}$ and the obstacles $\ol{\Obs}=\union_{i=1}^{\ndim_\Obs}\ol{\Obs}_\is$, which we assume are union of H-polytopes:
\begin{align}
\begin{split}
    \Xgoal &= \ol{\Xgoal}\times\K\times \Xplan\other
    \\
    \Obs_\is &= \ol{\Obs}_\is \times \K\times \Xplan\other,
\end{split}
\end{align}
where $\is = 1, \cdots, \ndim_\Obs$.

\subsubsection{PWA Planning Model}\label{subsec:pwa_planning_model_conversion}
We approximate the planning model $\dyn\plan$ as a discrete-time, time-variant PWA to enable computation.
First, we define linearization points $\xv\lin_{\ts, \is} \in \Xpoly$, $\is = 1, \cdots, \ndim\lin_\ts$ for each $\ts = 0, \Delta\ts, \cdots, \tf - \Delta\ts$, which can be generated by gridding or uniform sampling.
We define the PWA regions as the Voronoi partition of the linearization points \cite{casselman2009new} (see Appendix \ref{app:voronoi_affinization}).
Each $\is$\tss{th} region is an H-polytope $\hpoly(\Acon_{\is,\ts},\bcon_{\is,\ts})$.
Then, in each region, at each time $\ts = 0, \Delta\ts, \cdots, \tf - \Delta\ts$, we define the dynamics $\Ccon_{\is,\ts}$ and $\dcon_{\is,\ts}$ by Taylor expanding about the linearization points:

\begin{subequations}\label{eq:pwa_conversion}
\begin{align}
    \xv(\ts+\Delta\ts) &= \Ccon_{\is, \ts}\xv(\ts) + \dcon_{\is,\ts},\\
    \Ccon_{\is,\ts} &= \eye_{\nlow} + \Delta\ts\frac{\partial\dyng\plant(\ts, \xv)}{\partial\xv} \Bigr|_{{\xv=\xv\lin_{\is,\ts}}},\\
    \dcon_{\is,\ts} &= \Delta\ts(\dyng\plant(\ts, \xv\lin_{\is,\ts}) - \frac{\partial\dyng\plant(\ts, \xv)}{\partial\xv} \Bigr|_{{\xv=\xv\lin_{\is,\ts}}}\xv\lin_{\is, \ts}),\\
    \dyng\plant(\ts, \xv) &= \begin{bmatrix}
        \dyn_{\wv}\\
        \zeros_{\nparam} \\
        ({\dyn\plan})_{(\nwork+1):\nplan}\\
    \end{bmatrix} 
\end{align}
\end{subequations}
Depending on the planning model, there may be some PWA regions with identical dynamics.
If the union of the regions is convex, we combine them into a single PWA region (note, the linearization points are discarded).

From this point onwards, we will be using the PWA system defined by \eqref{eq:pwa_conversion} as the planning model instead of \eqref{eq:planning-model}.
That is, the tracking error, BRAS, and reach-avoid guarantees will be computed with respect to the tracking model \eqref{eq:tracking-model} and the planning model \eqref{eq:pwa_conversion}.
Thus, PARC is concerned with the tracking error between \eqref{eq:tracking-model} and \eqref{eq:pwa_conversion}, but not the error induced by affinizing \eqref{eq:planning-model} to \eqref{eq:pwa_conversion}.

\begin{rem}
Choosing a larger $\Delta\ts$ would induce higher mismatch between the tracking model and the planning model, thus increasing the conservativeness of the method.
In scenarios with highly dynamic systems or cluttered environments, selecting a smaller $\Delta \ts$ reduces conservativeness, as demonstrated in Fig.~\ref{fig:avoid_approx}.
As will be shown in the rest of this section, PARC's reach-avoid guarantees will still be maintained regardless of the choice of $\Delta \ts$, and its computation time is inversely proportional to $\Delta \ts$.
\end{rem}

\subsection{PARC: Reach Set without Tracking Error}\label{subsec:parc_reach}
Without accounting for tracking error, the reach set of a planning model is the set of initial conditions $\xv_0$ of trajectories that are \textit{guaranteed} to reach the goal region $\Xgoal$ at time $\tf$.
Instead of naively computing the \textit{full} reach set, which scales exponentially in memory, computation time, and number of set representations with the number of timesteps, PARC is \textit{planning aware}: given an \textit{expert plan}, we compute only a local, underapproximated BRS efficiently for the neighborhood that shares the same mode sequence as the \textit{expert plan}.
Fixing the mode sequence reduces the PWA system to a time-variant affine system, so the BRS is a \textit{single} H-polytope using the inverse affine map of the goal (see \eqref{eq:1stepbrs}), with computation time scaling linearly with the number of timesteps and quadratically with dimensions from matrix multiplication.

\begin{assum}[Expert Plan w/o Tracking Error]\label{as:expert_traj}
We have access to at least one goal-reaching initial condition $\xv_0$ for which the PWA system reaches the goal \emph{without necessarily avoiding obstacles}:
\begin{align}
    \xv(\tf; \xv_0) \in \Xgoal.
\end{align}
\end{assum}
\noindent Obtaining expert plans is system-dependent, but is not difficult for short-horizon and low-dimensional PWA models, because we do not require the expert plan to avoid obstacles.
Thus, one option is to solve a two-point boundary value problem using nonlinear MPC \cite{bonalli2019gusto} or mixed-integer linear program (MILP) \cite{lofberg2004yalmip}, both of which are \textit{complete}.
In practice, it is often sufficient to uniformly sample trajectory parameters $\paramv$ and roll out the PWA planning model, since the robot starts out near the goal.
This assumption is even easier to fulfill if the planning model is a time-variant affine (which we use for a general 3D quadrotor and drift vehicle planning in Section \ref{sec:experiments}, \ref{sec:drift_demo}) since plans have only one mode.

If Assumption \ref{as:expert_traj} is satisfied, we can obtain $\xv_0$'s mode sequence $\Sap = (\sap_0, \cdots, \sap_{\tf - \Delta\ts})$ using \eqref{eq:pwa_idx}.
That is,
\begin{align}\label{eq:ms_identify}
    \sap_\ts &= \min \{\is \ |\ \xv(\ts; \xv_0) \in \hpoly(\Acon_{\is,\ts}, \bcon_{\is,\ts}), \is=1,\dots,\npwat\},
\end{align}
for all $\ts = 0, \Delta\ts, \cdots, \tf - \Delta\ts$.
Then, we can compute the exact $\tf$-time BRS using the method in \cite{thomas2006robust}:
\begin{prop}[Reach Set w/o Tracking Error]
\label{prop:reach_set}
    Consider a mode sequence $\Sap = (\sap_0, \cdots, \sap_{\tf - \Delta\ts})$.
    Then the \textit{exact} $\tf$-time BRS $\ol{\Rbrs}_0$ of the goal for the affine dynamics associated with $\Sap$ is
    \begin{align}
    \begin{split}\label{eq:brs_no_error}
    \ol{\Rbrs}_{\tf} &= \Xgoal,\\
    \ol{\Rbrs}_{\ts} &= \brs{\ol{\Rbrs}_{\ts+\Delta\ts}, \Ccon_{\sap_{\ts},\ts}, \dcon_{\sap_{\ts},\ts}} \cap \hpoly(\Acon_{\sap_{\ts},\ts}, \bcon_{\sap_{\ts},\ts}),
    \end{split}
    \end{align}
    for all $\ts = 0,\Delta\ts,\cdots,\tf - \Delta\ts$.
\end{prop}
\begin{proof}
    This follows directly from \eqref{eq:time_variant_affine} and Lemma \ref{lem:one_step_BRS}.
\end{proof}
\noindent 
This means the reference trajectories generated from any augmented planning state within $\ol{\Rbrs}_0$ are guaranteed to reach $\Xgoal$ at time $\tf$ with mode sequence $\Sap$.
This process, along with an expert plan, is illustrated in Fig.~\ref{fig:turtlebot-reachset}.

Traditionally, methods such as the \texttt{reachableSet} function in MPT3 \cite{herceg2013multi} compute the \textit{full} BRS of the PWA system instead of that of only a particular mode sequence.
While computing the full BRS would return all possible reachable states and do not require knowledge of an expert plan, the computation time grows exponentially with the number of BRS steps and quickly becomes intractable due to the large number of polytopes, as shown in Table \ref{table:mpt3vsparc}.
Instead, Proposition \ref{prop:reach_set} enables a much more reasonable computation time by isolating the affine dynamics of a single mode sequence and keeping only one polytope regardless of the number of BRS steps.

\begin{figure}[!htb]
\centering
\includegraphics[width=1\columnwidth]{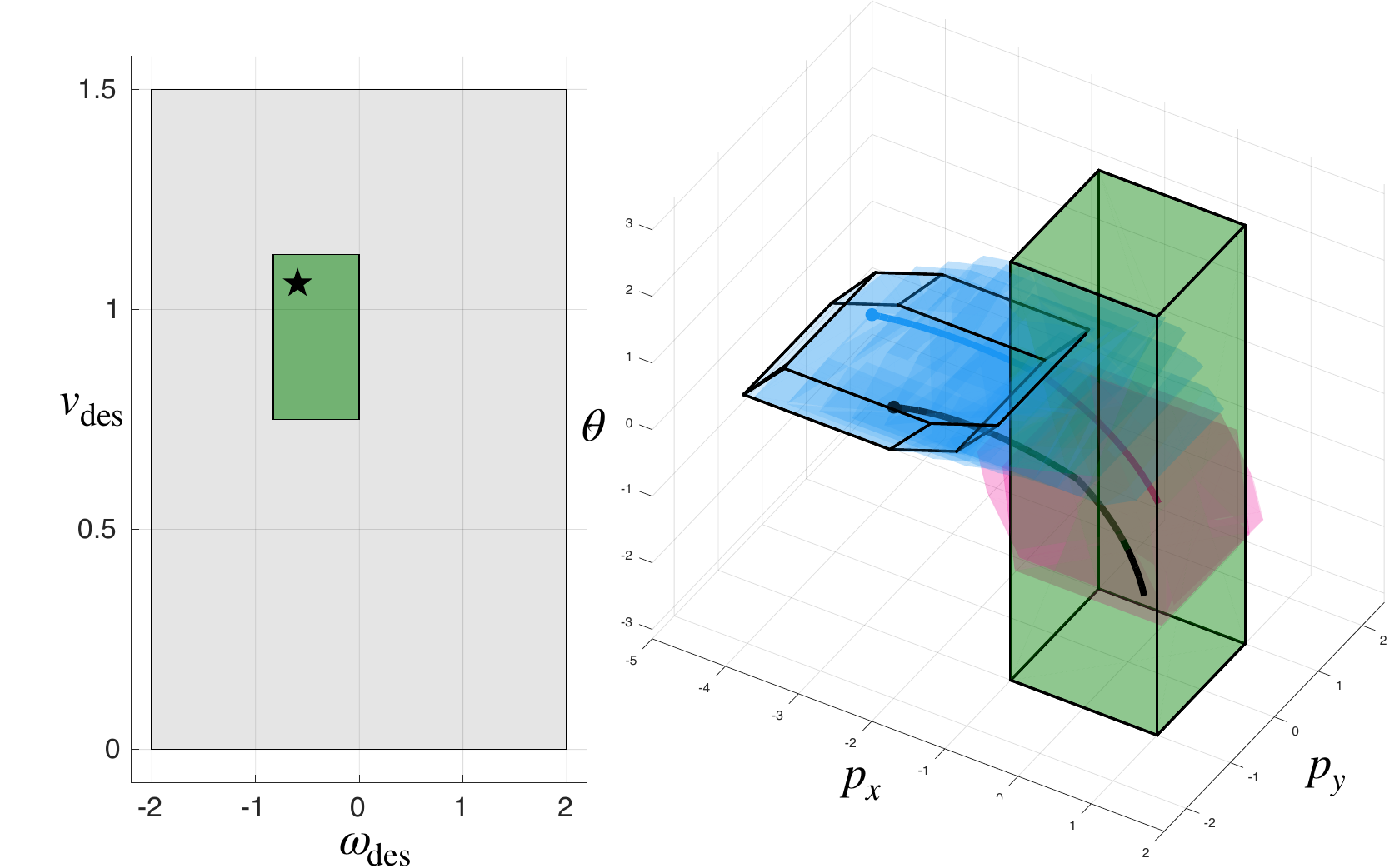}
\vspace{-1em}
\caption{(Right) 
Visualization of the reach set computation in Section \ref{subsec:parc_reach} for the TurtleBot example (Section \ref{subsec:turtlebot_example}), projected onto the planning state space $\Xplan$, when $\planv_0=[-4,0,\pi/5]\tp$, $\tf=4(\unit{s})$.
The goal set is the green polytope extended to all $\theta \in [-\pi,\pi]$.
An expert plan, computed by solving a two-point boundary value problem, passes through two different types of modes (blue and magenta line).
We compute the corresponding $\tf$-time BRS $\ol{\Rbrs}_{0}$ (blue polytope with solid outline).
Intermediate BRSs $\ol{\Rbrs}_{\ts}$ are the polytopes without outlines, colored by the type of mode.
(Left)
Illustration of the trajectory parameter space; the grey region represents $\K$, the green region represents the projection of $\ol{\Rbrs}_{0}$ onto $\K$.
After reach-set computation, the state $\xv\lbl{sample}=[\planv\lbl{sample}\tp, \paramv\lbl{sample}\tp]\tp$ is sampled from $\ol{\Rbrs}_{0}$: black dot (right) denotes the $\planv\lbl{sample}$ and black star (left) denotes the $\paramv\lbl{sample}$. The resulting trajectory is denoted as the black line which reaches the goal.}
\label{fig:turtlebot-reachset}
\vspace{-1em}
\end{figure}

\begin{table}[!ht]
\centering
\begin{tabular}{r | r r | r r }
     & \multicolumn{2}{|c|}{\textbf{Computation Time (s)}} & \multicolumn{2}{|c}{\textbf{Number of Polytopes}}\\
     \hline
     \textbf{BRS Steps} & \textbf{MPT3} & \textbf{Prop. \ref{prop:reach_set}} & \textbf{MPT3} & \textbf{Prop. \ref{prop:reach_set}}\\
     \hline
     1 & 0.03 & 0.0015 & 16 & 1\\
     2 & 0.27 & 0.0022 & 100 & 1\\
     3 & 1.45 & 0.0038 & 484 & 1\\
     4 & 6.63 & 0.0036 & 2116 & 1\\
     5 & 29.46 & 0.0027 & 8836 & 1\\
     26 & timeout & 0.0155 & --- & 1
\end{tabular}
\caption{Computation time and number of polytopes for 1, 2, 3, 4, 5, and 26-step BRS for the TurtleBot example (Section \ref{subsec:turtlebot_example}) using the \texttt{reachableSet} function in MPT3 \cite{herceg2013multi} and Proposition \ref{prop:reach_set}.
Since MPT3 computes the full BRS of the PWA system, the computation time quickly becomes intractable due to the large number of polytopes.
In contrast, Proposition \ref{prop:reach_set} maintains a reasonable computation time by keeping only one polytope at each step, characterized by the affine dynamics corresponding to the mode sequence of the expert plan.}
\label{table:mpt3vsparc}
\end{table}

\subsection{PARC: Avoid Set without Tracking Error}\label{subsec:parc_avoid}
Without accounting for tracking error, the avoid set is the set of pairs of initial planning model state and trajectory parameter such that any plans generated from \textit{outside} the avoid set are guaranteed to not collide with the obstacles at all \textit{continuous} time before $\tf$.

Before we can overapproximate the avoid set between two-time steps, we must adapt our notion of extended translation invariance to the PWA planning model:
\begin{lem}[PWA-ETI in Workspace]\label{lem:pwa_eti}
    Suppose $\dyn\plan$ is an ETI system in the workspace as in Assum. \ref{ass:eti}.
    Also suppose we convert it to a PWA system as in Section \ref{subsec:pwa_planning_model_conversion} and \eqref{eq:pwa_conversion}.
    Then in mode $\is$ at time $\ts$, each state $\xv$ of the PWA system obeys
    \begin{align}\label{eq:eti_pwa_def}
        \norm{\hat{\Ccon}_{i,t}^{\Workspace \times \K}{(\hat{\Ccon}_{i,t}\xv +\dcon_{i,t})_{1:(\nwork+\nparam)}}}\lbl{2} = 0 \quad \forall \xv\in \Xpoly
    \end{align}
    where $\hat{\Ccon}_{i,t}=\Ccon_{i,t}-\eye_{\nlow}$ and $\hat{\Ccon}^{\Workspace \times \K}_{i,t}$ denotes the $(\nwork + \nparam) \times (\nwork+\nparam)$ upper-left submatrix of $\hat{\Ccon}_{i,t}$.
\end{lem}
\begin{proof}
    This follows directly from Assum. \ref{ass:eti} and \eqref{eq:pwa_conversion}.
\end{proof}
\noindent To alleviate the conservativeness of our approach, it is beneficial to extend the ETI property to subspaces beyond the workspace $\Workspace$.
If the planning model is ETI in subspace $\Xpoly\eti \subset \R^{{\nlow}\eti}$, it must fulfill the following conditions:
\begin{subequations} \label{eq:eti_pwa_reorder}
\begin{align}
\norm{\hat{\Ccon}_{i,t}^{\Xpoly\eti}(\hat{\Ccon}_{i,t}\xv+\dcon_{i,t})_{1:{\nlow}\eti}}\lbl{2} &= 0 \quad \forall \xv \in \Xpoly \\
\exists \Xpoly\noneti \quad \regtext{s.t.} \quad \Xpoly &= \Xpoly\eti \times \Xpoly\noneti.
\end{align}
\end{subequations}
\noindent where $\hat{\Ccon}^{\Xpoly\eti}_{i,t}$ denotes the ${\nlow}\eti \times {\nlow}\eti$ upper-left submatrix of $\hat{\Ccon}_{i,t}$.
We refer to the first ${\nlow}\eti$ states of $\xv \in \Xpoly$ as \textit{ETI states}.
Note that $\wv$ and $\paramv$ are ETI states because of Lemma \ref{lem:pwa_eti} and $\Xpoly = \Workspace \times \K \times \Xplan\other$ from \eqref{eq:x_order}.

To lower PARC's conservativeness, we reorder $\xv(\ts)$ such that the most ETI states are placed at the beginning of $\planv\other$, as identified by \eqref{eq:eti_pwa_reorder}.
To simplify exposition for avoid-set computation, we now denote $\xv(\ts)$ as:
\begin{subequations}
\begin{align}
    \xv(\ts) &= \begin{bmatrix}
        \xv\eti(\ts)\\
        \xv\noneti(\ts)
    \end{bmatrix},\\
    \xv\eti(\ts) &= \begin{bmatrix}
        \wv(\ts) \\
        \paramv\\
        {\planv\other}\eti(\ts)
    \end{bmatrix}
\end{align}
\end{subequations}
where ${\planv\other}\eti(\ts)$ are the first ${\nlow}\eti - \nwork - \nparam$ states in $\planv\other(\ts)$ that are ETI states, and $\xv\noneti(\ts)\in \Xpoly\noneti \subset \R^{{\nlow}\noneti}$ are the remaining states in $\planv\other(\ts)$.
Thus, ${\nlow}\noneti = \nlow - {\nlow}\eti$.
Note that all sets and matrices including $\Obs$, and $\Xgoal$ are relabeled accordingly.

We are now ready to state our first key contribution: overapproximating the avoid set without tracking error.
The following theorem overapproximates the continuous-time BRS of an obstacle and intersects it with the discrete-time reach-set of the PWA planning model to identify unsafe plans:

\begin{thm}[Avoid Set w/o Tracking Error] \label{thm:avoid_set}
Consider the obstacle $\Obs_\is$, some time $\ts \in \{0, \Delta\ts, \cdots, \tf - \Delta\ts\}$, and some mode sequence $\Sap = (\sap_0, \cdots, \sap_\ts, \sap_{\ts+\Delta\ts}, \cdots, \sap_{\tf-\Delta\ts})$.
Suppose we compute the $(\tf-\ts)$-time BRS $\ol{\Rbrs}_\ts$ as in \eqref{eq:brs_no_error}.
Let $A = \brs{\proj[1:{\nlow}\eti]{\Obs_i} \times \R^{{\nlow}\noneti}, \Ccon_{\sap_\ts, \ts}, \dcon_{\sap_\ts, \ts}}$, and
define the intermediate avoid set $\ol{\Ravd}_{\is, \ts, \ts} \subset \ol{\Rbrs}_\ts$ as:
\begin{align}\label{eq:int_avoid_set}
    \ol{\Ravd}_{\is, \ts, \ts} = \conv{\proj[1:{\nlow}\eti]{A}\times\Xpoly\noneti, \Obs_\is}\cap\ol{\Rbrs}_\ts
\end{align}
Then, for any $\xv(\ts) \in \ol{\Rbrs}_\ts \setminus \ol{\Ravd}_{\is, \ts, \ts}$, $\xv(\ts+\Delta\ts) = \Ccon_{\sap_\ts, \ts}\xv(\ts) + \dcon_{\sap_\ts, \ts}$,
\begin{align}\label{eq:int_avoid_con}
    \left\{\xv(\ts) + \gams(\xv(\ts+\Delta\ts) - \xv(\ts))\ |\ 0\leq\gams\leq1\right\}\cap\Obs_\is = \emptyset.
\end{align}
\end{thm}
\begin{proof}
See Appendix \ref{app:extended_proof}.    
\end{proof}

In essence, the intermediate avoid set from \eqref{eq:int_avoid_set} overapproximates the avoid set contained in each intermediate BRS.
Under specific circumstances, the intermediate avoid set can be computed without convex hull or projection.
See Appendix \ref{subsec:avoid_set_nohull} for details.

\begin{rem}
We note that PARC can be readily extended to moving obstacles by replacing $\Obs_\is$ in \eqref{eq:int_avoid_set} with the over-approximated swept volume of the obstacle between $\ts$ and $\Delta\ts$, assuming its motion is known or can be over-approximated.
This strategy has shown to be effective on other set-based planner-tracker frameworks \cite{vaskov2019towards,vaskov2019not}.
For ease of exposition, we consider only static obstacles in this paper, as many of the methods we are comparing PARC against do not handle moving obstacles \cite{chen2021fastrack, dawson2022safe, verginis2022kdf}, and reach-avoid for static obstacles in near-danger setups is already a difficult problem for existing literature.
\end{rem}

We now propose a straightforward scheme to construct the avoid set using Theorem \ref{thm:avoid_set}, where reference trajectories generated from outside the avoid set are guaranteed to avoid the obstacle:
\begin{cor}[Avoid Set Computation]
\label{cor:avoid_set_computation}
    We compute the avoid set $\ol{\Ravd}_{\is, \ts, 0}$ by computing the $\ts$-time BRS of $\ol{\Ravd}_{\is, \ts, \ts}$.
    Mathematically,
    \begin{align}\label{eq:avoid_brs}
        \ol{\Ravd}_{\is, \ts, \js} &= \brs{\ol{\Ravd}_{\is, \ts, \js+\Delta\ts}, \Ccon_{\sap_{\js}, \js}, \dcon_{\sap_{\js}, \js}}\cap\hpoly(\Acon_{\sap_{\js}, \js}, \bcon_{\sap_{\js}, \js}) ,
    \end{align}
    for all $\js = 0, \Delta\ts, \cdots, \ts-\Delta\ts$.
    Then, the final, overapproximated avoid set $\ol{\Ravd}$ can be stored as a union of H-polytopes as:
    \begin{align}
        \ol{\Ravd} = \union_{\is=1}^{\nobs} \union_{\ts=0}^{\tf-\Delta\ts}{\ol{\Ravd}_{\is, \ts, 0}}.
    \end{align}
\end{cor}
\begin{proof}
    This follows as a consequence of Theorem \ref{thm:avoid_set} and the BRS operation in \eqref{eq:avoid_brs}, which is exact for an affine system defined for a mode sequence (Lemma \ref{lem:one_step_BRS}).
\end{proof}
To sample from within the reach set but not from the avoid set, one can either rejection-sample, sample via random-walk algorithms \cite{kannan2009random, chen2018fast}, or compute the set difference between the reach set and the avoid set \cite{baotic2009polytopic}, and then sample via solving a linear program (LP).
An example of the reach set and the avoid set is shown in Fig. \ref{fig:avoid_set_turtlebot}; notice that the avoid set is a union of polytopes, hence can be non-convex.

\begin{figure}
    \centering
    \includegraphics[width=1\columnwidth]{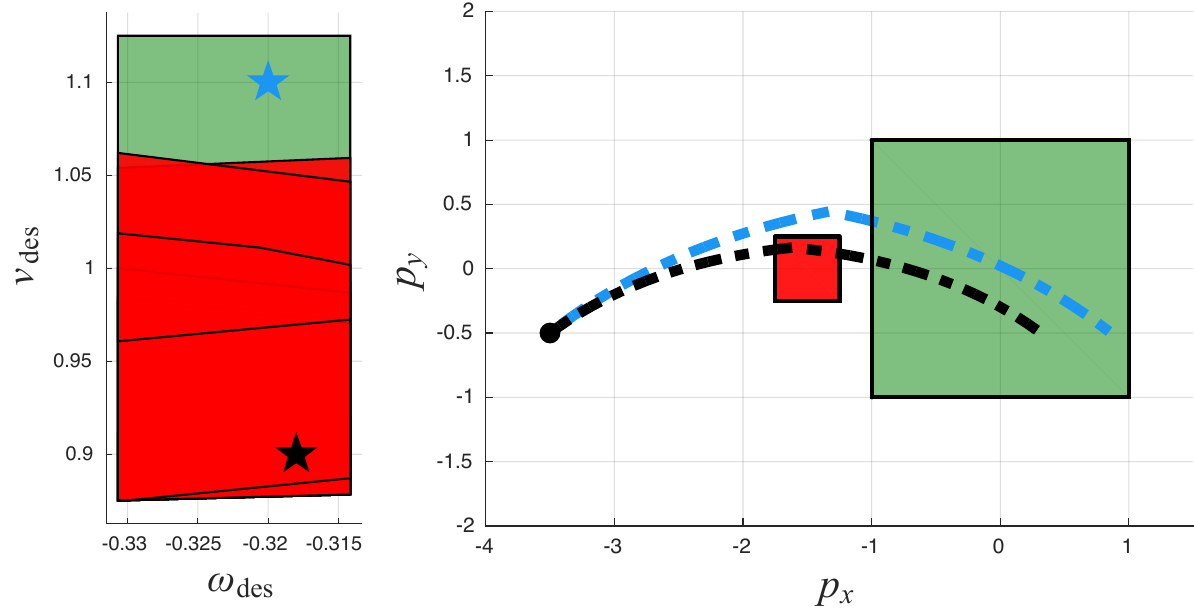}
    \caption{
    (Left) Visualization of the reach set (green) and the avoid set (red) computation from Section \ref{subsec:parc_avoid} for the TurtleBot example (Section \ref{subsec:turtlebot_example}), sliced in the first three dimensions with respect to the initial condition $\planv_0=[-3.5,-0.5,\pi/5]\tp$.
    The black star denotes the expert plan with $\paramv = [-0.318, 0.9]\tp$ and the blue star denotes a safe plan with $\paramv = [-0.320, 1.1]\tp$.
    (Right) Illustration of the resulting trajectory in $\px$ and $\py$, with the expert plan shown as a dashed black line and the safe plan shown as a dashed blue line.
    While the expert plan collides with the obstacle (red), PARC is still able to compute a safe plan with the same mode sequence that safely reaches the goal set (green).
    }
    \label{fig:avoid_set_turtlebot}
\end{figure}

\begin{figure}
    \centering
    \includegraphics[width=1\columnwidth]{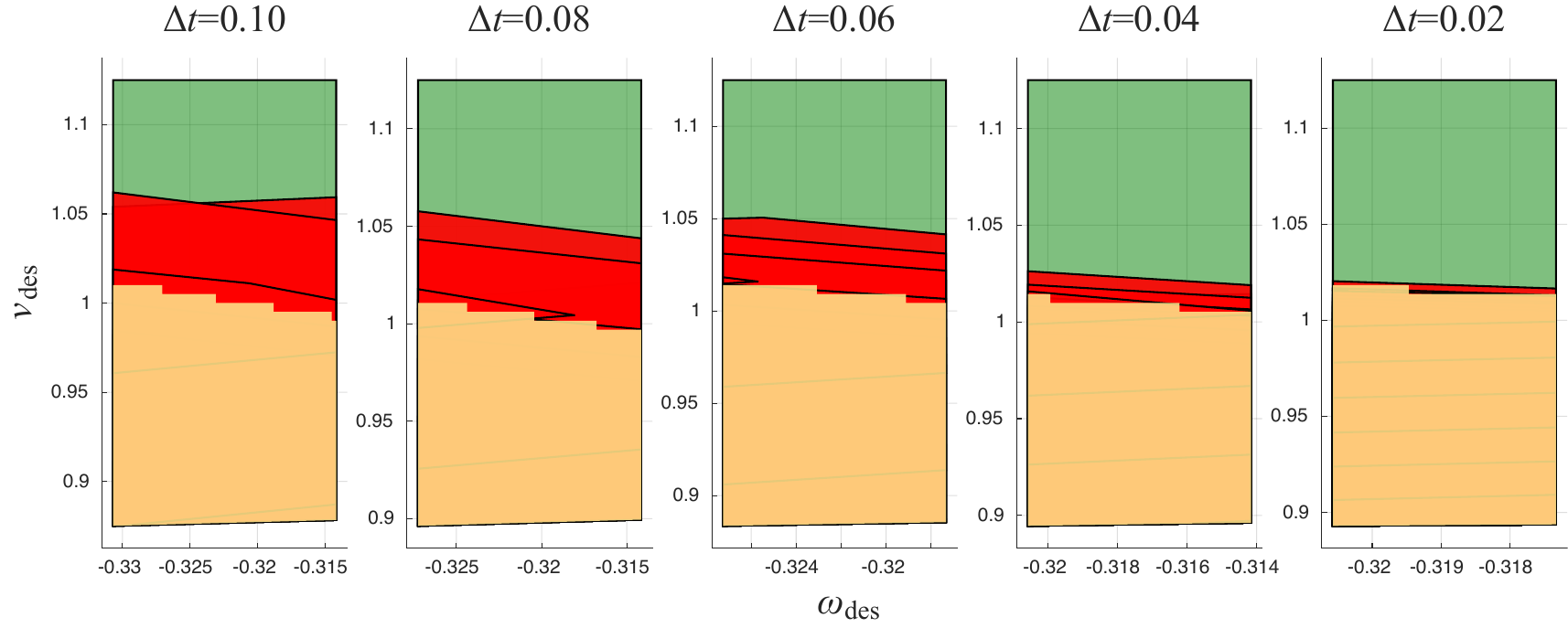}
    \caption{
    Visualization of the reach set (green), the avoid set (red), and a grid-based over-approximation of the true avoid set (yellow) for the TurtleBot example (Section \ref{subsec:turtlebot_example}) for different values of the timestep $\Delta\ts$, with the same problem setup and slicing as Fig. \ref{fig:avoid_set_turtlebot}.
    The grid-based over-approximation was obtained by a slower method detailed in Appendix \ref{subsec:grid_avoid_set} with forward reachable set (FRS) verification on the reach set, divided into a $50\times50$ grid.
    Both the reach set and the numerical overapproximation error gets smaller as $\Delta\ts$ decreases.
    }
    \label{fig:avoid_approx}
\end{figure}

Disregarding tracking error and perception uncertainties, which all planning and control methods suffer from, PARC's only source of conservativeness (from computing the BRAS with respect to the expert plan's mode sequence) comes from the convex hull approximation for linearly interpolated trajectories and the projection operation to account for non-ETI states.
We illustrate the extent of this overapproximation in Fig. \ref{fig:avoid_approx}.
Note that the amount of overapproximation depends heavily on the system, the problem setup, and the expert plan, but can generally be alleviated by choosing smaller timesteps at the expense of computation time.

\subsection{Estimating Tracking Error}\label{subsec:tracking_error}

Now that we have a basis for computing the BRAS, we will model tracking error such that, when incorporated into the PARC framework, can produce plans that, \textit{when followed by the tracking model}, guaranteed to reach the goal region while avoiding all obstacles.

Our method is based on \cite{kousik2019safe,shao2021reachability}. 
We sample $(\trackv_{0, \js}, \paramv_\js) \in \itv{\Xtrack}\times\itv{\K}\subset \Xtrack\times\K$ for $\js = 1, \cdots, \ndim_\sample$ uniformly, where $\itv{\Xtrack}\times\itv{\K}$ is a hyperrectangle created by partitioning $\Xtrack\times\K$; we do not index this partition to ease notation.
We then have the robot start at $\trackv_{0, \js}$ to track the plans parameterized by initial planning state $\planv_{0, \js} \triangleq (\trackv_{0, \js})_{1:\nplan}$ and trajectory parameter $\paramv_\js$, and compute the error between the planned and realized trajectories in the workspace states (recall that we assume the goal and obstacles are subsets of workspace).

We consider two types of errors: the \textit{maximum final error} $\es_{\tf, \is} \in \R$ for goal-reaching, and \textit{maximum interval error} $\itv{\es}_{\ts, \is} \in \R$ for obstacle avoidance, defined for the discrete-time at time $\ts$ and the $\is^\regtext{th}$ dimension of $\trackv$ and $\planv$.
PARC require two types of errors because it provides continuous-time guarantees for obstacle avoidance, but gives guarantees at the final timestep only for goal-reaching.

For each workspace coordinate $\is = 1, \cdots, \nwork$, we define the maximum final error $\es_{\tf, \is}$ as:
\begin{align}\label{eq:max_final_error}
\es_{\tf, \is} = \max_\js
        \abs{\planv_\is(\tf; \planv_{0, \js}, \paramv_\js) - \trackv_\is(\tf; \trackv_{0, \js}, \paramv_\js)},
\end{align}
where $\planv_\is(\cdot)$ and $\trackv_\is(\cdot)$ are the $\is^\regtext{th}$ element of $\planv(\cdot)$ and $\trackv(\cdot)$.
In plain words, $\es_{\ts, \is}$ represents the maximum error observed at the final time $\tf$ for the $\is^\regtext{th}$ workspace dimension.

For $\ts = 0, \Delta\ts, \cdots, \tf-\Delta\ts$, $\is = 1, \cdots, \nwork$,
we define the maximum interval error $\itv{\es}_{\ts, \is}$ as:
\begin{align}\label{eq:max_int_error}
    \itv{\es}_{\ts, \is} = \max_{\js,\ \ts' \in [\ts,\ts+\Delta\ts]}\abs{\planv_\is(\ts'; \planv_{0, \js}, \paramv_\js) - \trackv_\is(\ts'; \trackv_{0, \js}, \paramv_\js)}
\end{align}
which is the maximum error observed \textit{between} each discrete timestep for the $\is^\regtext{th}$ workspace dimension.

We assume this approach provides an upper bound for feasible tracking error, which can be proven for simple planning models \cite{kousik2019safe, shao2021reachability}.
\begin{assum}[Maximum Sampled Error Contains Maximum Error]\label{as:RTD_assum}
    We assume that, for a large enough $\ndim_\sample$,
    \begin{subequations}
    \begin{align}
        \es_{\tf, \is} &\geq \max_{\trackv_{0}\in\itv{\Xtrack},\,\paramv\in\itv{\K}}
            \abs{\planv_\is(\tf; \planv_{0}, \paramv) - \trackv_\is(\tf; \trackv_{0}, \paramv)},\\
        \itv{\es}_{\ts, \is} &\geq \max_{\trackv_{0}\in\itv{\Xtrack},\,\paramv\in\itv{\K},\,\ts' \in [\ts,\ts+\Delta\ts]}
            \abs{\planv_\is(\ts'; \planv_0, \paramv) - \trackv_\is(\ts'; \trackv_0, \paramv)},
    \end{align}
    \end{subequations}
    for all $\ts = 0, \Delta\ts, \cdots, \tf-\Delta\ts$, where $\planv_0 = (\trackv_0)_{1:\nplan}$.
\end{assum}
\noindent
This assumption is clearly a potential weakness of the method, but we find that it holds in practice because the sampling process produces adversarial $(\trackv_0,\paramv)$ pairs.
Furthermore, we can refine the tracking error estimate by increasing the fineness of the partition of $\Xtrack\times\K$ into subsets $\itv{\Xtrack}\times\itv{\K}$ \cite{kousik2019safe}.
Note there exist alternative methods to collect or construct tracking error functions in the literature \cite{verginis2022kdf,chen2021fastrack}.
If desired, they could be swapped out with \eqref{eq:max_final_error} and \eqref{eq:max_int_error} and PARC would provide guarantees with respect to the new tracking error functions instead.
We demonstrate PARC's modularity to the tracking error model in Section \ref{exp:time-varying}.

To incorporate these errors into the geometric computation of PARC, we express them as the \textit{maximum final error set} $\Eset_{\tf} \subset \R^{\nlow}$ and the \textit{maximum interval error set} $\itv{\Eset}_{\ts} \subset \R^{\nlow}$:
\begin{align}
    \Eset_{\tf} &=  [-\es_{\tf, 1}, \es_{\tf, 1}]\times\cdots\times[-\es_{\tf, \nwork}, \es_{\tf, \nwork}] \times\underbrace{\{0\}\times\cdots\times\{0\}}_{({\nplan}\other + \nparam)\regtext{ times}},\\
    \itv{\Eset}_{\ts} &= [-\itv{\es}_{\ts, 1}, \itv{\es}_{\ts, 1}]\times\cdots\times[-\itv{\es}_{\ts, \nwork}, \itv{\es}_{\ts, \nwork}]\times\underbrace{\{0\}\times\cdots\times\{0\}}_{({\nplan}\other + \nparam)\regtext{ times}}.
\end{align}
It is straightforward to express these intervals as H-polytopes \cite{herceg2013multi}.
The Cartesian product with singletons of zeros is necessary to augment the states for Minkowski sum and Pontryagin difference operations with the obstacles and the goal.

We note that collection of the tracking error only requires forward simulation or rollouts of the system and knowledge of the input and output states, so our approach could still apply if the tracking model is a ``\textit{black box}'' by adapting methods such as \cite{selim2022safe, alanwar2023data}.
In this paper, we consider ``\textit{white-box}'' dynamics in our experiments to treat them as a control condition to compare against the literature.
While the forward simulation and collision checking of trajectories may sometimes be computationally expensive, we consider the tradeoff acceptable as the tracking error of PARC is expected to be collected offline.
An example of computed tracking error for a near-hover quadrotor dynamical model is shown in Fig.~\ref{fig:nh_tracking_error}.

\subsection{Incorporating Tracking Error into PARC}

To conclude our proposed method, we incorporate tracking error into PARC's BRAS computation, thereby transferring guarantees for the planning model to the tracking model.
In short, we get reach-avoid guarantees that incorporate tracking error by simply Minkowski summing the error with obstacles and Pontryagin differencing from the goal set.

\subsubsection{PARC: Reach Set with Tracking Error}\label{subsec:parc_reach_real}

From Assumption \ref{as:RTD_assum}, we assume the sampled tracking error is valid only over the compact hyperrectangle $\itv{\Xtrack}\times\itv{\K}$.
Thus, we must define a valid input region $\itv{\Xpoly}\subset \Xpoly$ for which the computed reach set will have guarantees over:
\begin{align}
    \itv{\Xpoly} = \proj[1:\nwork]{\itv{\Xtrack}}\times\itv{\K}\times\proj[(\nwork+1):\nplan]{\itv{\Xtrack}}.
\end{align}
The projection and Cartesian product are necessary to reorder $\itv{\Xtrack}$ and $\itv{\K}$ into that of \eqref{eq:x_order}.

To incorporate tracking error, we first tighten our assumption on the expert plan to find a mode sequence:

\begin{assum}[Expert Plan with Tracking Error]\label{as:expert_traj_error}

We assume that a goal-reaching initial condition $\xv_0$ that defines an expert plan is known and given for the PWA planning model in Section \ref{subsec:PWA_convert}.
Particularly, $\xv_0$ is any point that satisfies:
\begin{align}
    \xv(\tf; \xv_0) \in \Xgoal \ominus \Eset_{\tf}.
\end{align}
\end{assum}
\noindent Intuitively, $\Eset_{\tf, \is}$ captures the maximum tracking error at time $\tf$: if the plans can arrive at the smaller goal region $\Xgoal \ominus \Eset_{\tf}$, then the robot should be guaranteed to arrive at the original, bigger goal region $\Xgoal$ at time $\tf$.

With a given expert plan, we can obtain $\xv_0$'s mode sequence using \eqref{eq:ms_identify}, then compute the underapproximated reach set $\Rbrs_0$ of the goal similar to Proposition \ref{prop:reach_set}.
\begin{prop}[Reach Set with Tracking Error]
\label{prop:reach_set_real}
    Consider a mode sequence $\Sap = (\sap_0, \cdots, \sap_{\tf - \Delta\ts})$.
    Then we can underapproximate the reach set $\Rbrs_{0}$ corresponding to this mode sequence by:
    \begin{subequations}\label{eq:reach_yeserror}
    \begin{align}
    \Rbrs_{\tf} = &\Xgoal\ominus \Eset_{\tf},\\
    \Rbrs_{\ts} = &\brs{\ol{\Rbrs}_{\ts+\Delta\ts}, \Ccon_{\sap_{\ts},\ts}, \dcon_{\sap_{\ts},\ts}}\cap\hpoly(\Acon_{\sap_{\ts},\ts}, \bcon_{\sap_{\ts},\ts})\notag\\
    &\forall \ts = \Delta\ts,\cdots,\tf - \Delta\ts,\\
    \Rbrs_{0} = &\brs{\ol{\Rbrs}_{\Delta\ts}, \Ccon_{\sap_{\Delta\ts},\Delta\ts}, \dcon_{\sap_{\Delta\ts},\Delta\ts}}\cap\itv{\Xpoly},
    \end{align}
    \end{subequations}
\end{prop}
\begin{proof}
    This follows from Lemma \ref{lem:one_step_BRS}, Assumption \ref{as:RTD_assum}, and Proposition \ref{prop:reach_set}.
\end{proof}
\noindent This extends Proposition \ref{prop:reach_set} by shrinking the goal set and intersection with the tracking error data coverage domain to include tracking error.
With $\Rbrs_{0}$ computed, if the robot tracks any plan generated from an initial condition within $\Rbrs_{0}$, it is guaranteed to reach the goal at time $\tf$, but may not be safe; next we guarantee safety by computing the avoid set.

\subsubsection{PARC: Avoid Set with Tracking Error}\label{subsec:parc_avoid_real}

Consider the obstacle $\Obs_\is$ and some time $\ts \in \{0, \Delta\ts, \cdots, \tf - \Delta\ts\}$.
We define the $\is^\regtext{th}$ buffered obstacle $\tilde{\Obs}_{\ts,\is}$ between time $\ts$ and $\ts+\Delta\ts$ as:
\begin{align}
    \tilde{\Obs}_{\ts,\is} &= \Obs_\is \oplus \itv{\Eset}_{\ts}.
\end{align}

Assuming $\itv{\Eset}_{\ts}$ captures the maximum tracking error from $\ts$ to $\ts+\Delta\ts$, by buffering $\Obs_\is$ with $\itv{\Eset}_{\ts}$, if the trajectory does not collide with $\tilde{\Obs}_{\ts,\is}$, then the robot following the trajectory should never collide with $\Obs_{\ts, \is}$.
Thus, incorporating tracking errors into the PARC formulation for avoid set follows closely to that of Section \ref{subsec:parc_avoid}, but with $\tilde{\Obs}_{\ts,\is}$ instead of $\Obs_{\ts, \is}$ as an obstacle.


We can finally state our second key contribution: overapproximating the avoid set with tracking error (i.e., computing the BRAS):

\begin{thm}[Avoid set w/ Tracking Error]
\label{thm:avoid_set_real}
Consider the buffered obstacle $\tilde{\Obs}_{\ts,\is}$ for some time $\ts \in \{0, \Delta\ts, \cdots, \tf - \Delta\ts\}$, and from \eqref{eq:reach_yeserror}, the $(\tf-\ts)$-time BRS ${\Rbrs}_\ts$ of the mode sequence $\Sap = (\sap_0, \cdots, \sap_\ts, \sap_{\ts+\Delta\ts}, \cdots, \sap_{\tf-\Delta\ts})$.
Let $A = \brs{\proj[1:{\nlow}\eti]{\tilde\Obs_{\ts, \is}} \times \R^{{\nlow}\noneti}, \Ccon_{\sap_\ts, \ts}, \dcon_{\sap_\ts, \ts}}$.
Define the intermediate avoid set ${\Ravd}_{\is, \ts, \ts} \subset {\Rbrs}_\ts$ as:
\begin{align}\label{eq:int_avoid_set_real}
    {\Ravd}_{\is, \ts, \ts} = \conv{\proj[1:{\nlow}\eti]{A}\times\Xpoly\noneti, \tilde{\Obs}_{\ts,\is}}\cap{\Rbrs}_\ts.
\end{align}
Then, for any $\xv(\ts) \in {\Rbrs}_\ts \setminus {\Ravd}_{\is, \ts, \ts}$, $\xv(\ts+\Delta\ts) = \Ccon_{\sap_\ts, \ts}\xv(\ts) + \dcon_{\sap_\ts, \ts}$,
\begin{align}
    \left\{\xv(\ts) + \gams(\xv(\ts+\Delta\ts) - \xv(\ts))\ |\ 0\leq\gams\leq1\right\}\cap\tilde{\Obs}_{\ts,\is} = \emptyset.
\end{align}
\end{thm}
\begin{proof}
    Follows from Theorem \ref{thm:avoid_set}.
\end{proof}

We can adapt Corollary \ref{cor:avoid_set_computation} to compute the avoid set ${\Ravd}_{\is, \ts, 0}$ by computing the $\ts$-time BRS of ${\Ravd}_{\is, \ts, \ts}$ as
\begin{align}\label{eq:avoid_set_chain}
    {\Ravd}_{\is, \ts, \js} &= \brs{{\Ravd}_{\is, \ts, \js+1}, \Ccon_{\sap_{\js}, \js}, \dcon_{\sap_{\js}, \js}}\ \forall \js = 0, \Delta\ts, \cdots, \ts-\Delta\ts.
\end{align}
Then, the final overapproximated avoid set ${\Ravd}$ can be stored as a union of H-polytopes as:
\begin{align}
    {\Ravd} = \union_{\is=1}^{\nobs} \union_{\ts=0}^{\tf-\Delta\ts}{{\Ravd}_{\is, \ts, 0}}.
\end{align}

With the avoid set and the reach set from Section \ref{subsec:parc_reach_real} computed, if the robot tracks any plans starting in $\Rbrs_{0} \setminus \Ravd$, it is guaranteed to reach the goal region while avoiding all obstacles.
Thus, the set $\Rbrs_{0}\setminus\Ravd$ is the desired object $\regtext{BRAS}(\tf, \ol{\Xgoal}, \ol{\Obs})$.
Next, we demonstrate PARC's BRAS computation on a variety of numerical examples.
\section{Experiments}\label{sec:experiments}

We now assess the utility of our proposed PARC method for generating safe trajectory plans near danger.
We seek to understand the impact of planning model design (Section \ref{exp:planning_model_design}), cooperative vs. adversarial planner-trackers (Section \ref{exp:coop_impact}), tight approximation of planning model reachable sets (Section \ref{exp:exact_plan_reach_impact}), and respecting time-varying tracking error in planning level (Section \ref{exp:time-varying}).
We also assess the utility of a learning-based reach-avoid method in Section \ref{exp:planning_model_design} and Section \ref{exp:coop_impact}.

To ensure fair assessment, we apply PARC to robots from the safe motion planning and control literature: a near-hover quadrotor in 2-D \cite{dawson2022safe} and 3-D \cite{chen2021fastrack}, and a more general quadrotor model in 3-D \cite{kousik2019safe}.
Implementation details are reported in Appendix \ref{app:experiment}.
All experiments, and training of learning-based methods, were run on a desktop computer with a 24-core i9 CPU, 32 GB RAM, and an Nvidia RTX 4090 GPU.

\textbf{Key result:}
PARC outperforms other methods in goal reaching and safety in near-danger scenarios but requires careful choice of planning model.
While it can still guarantee reach-avoid in non-near-danger scenarios, its performance can be worse than existing methods.
Most experiments are summarized in Table~\ref{table:experiments}.

\subsection{Impact of Planning Model Design}\label{exp:planning_model_design}

This experiment seeks to understand the impact of choosing a planning model for a given tracking model.
We deployed PARC on a tracking model and environment that have been successfully solved by a Neural CLBF \cite{dawson2022safe} approach.

\textbf{Key result:}
Na\"ive design of PARC's planning model can impact performance negatively.
We confirm that a better choice of planning model solves the same task with higher performance in Appendix \ref{app:additional_exp}.


\subsubsection{Experiment Setup}
To examine the importance of the planning model, we employed the \textit{same dynamics} for both the planning and tracking model, namely 2-D near-hover quadrotor dynamics:
\begin{align}\label{eq:2dhover_dyn}
    \dot{\planv} = \dot{\trackv}
        = \frac{d}{dt}\begin{bmatrix}
            \px \\ \pz \\ \theta \\ \vx \\ \vz \\ \omega
        \end{bmatrix}
        = \left[\begin{array}{l}
            \vx \\
            \vz \\
            \omega \\
            \tfrac{1}{\mq} \sin(\pth) (\flf+\frt) \\
            \tfrac{1}{\mq} \cos(\pth) (\flf+\frt) - g \\
            \tfrac{\radq}{\Iq} (\flf-\frt)
        \end{array}\right]
\end{align}
The workspace states are $\wv = [\px,\py]\tp$, and the inputs are left- and right-side propeller forces 
\begin{align*}
\uv = \paramv &= [\flf, \frt]\tp,
\end{align*}
which were held constant for $[0,\tf]$ to generate plans.
The feedback controller $\uv\feedback$ produces time-varying left- and right-side propeller forces using an iLQR controller \cite{kong2021ilqr}.

The gravitational constant is $g = 9.81$\unit{m/s^2}. 
The quadrotor's physical parameters [$\Iq, \mq, \radq$] are obtained from \cite{dawson2022safe}.

We used the obstacle and goal environment from \cite[Sec. 6.3]{dawson2022safe}, wherein the quadrotor must navigate over an obstacle and through a small gap to a goal region (see the left side of Fig. \ref{fig:quad2d_noerror} and Fig. \ref{fig:quad2d_yeserror}).

We used PARC to compute a BRAS for this scenario with an expert plan, accounting and without accounting for tracking error.
The planning model piecewise-affinized \eqref{eq:2dhover_dyn} with 80 PWA regions, timestep $\Delta\ts = 0.1$ \unit{s}, and final time $\tf = 2$ \unit{s}.
We assessed the computation time, as well as the size of the BRAS to investigate the impact of planning model design.

\subsubsection{Hypothesis}
We expected PARC to be able to compute plans to the goal reliably, with no safety violations during $[0,\tf]$.
However, due to the na\"ive choice of planning model, we expected the avoid set to be overly conservative, especially since \eqref{eq:2dhover_dyn} has three non-ETI states ($\vx, \vz, \omega$) in its PWA formulation.

\subsubsection{Results}
Fig. \ref{fig:quad2d_noerror} and Fig. \ref{fig:quad2d_yeserror} show PARC's attempt to construct a BRAS with and without accounting for tracking error.
Surprisingly, despite the high-dimensional planning model, PARC only took 5.08 \unit{s} to compute the BRAS with tracking error and 4.62 \unit{s} without.
By contrast, the Neural CLBF took 3 \unit{h} to train.
After accounting for tracking error, PARC was no longer able to compute trajectories traveling from the left of the leftmost obstacle.
No crashes or failure to goal-reach were reported when the tracking agent followed the trajectories generated from within the BRAS.

\begin{figure}
\centering
\includegraphics[width=0.8\columnwidth]{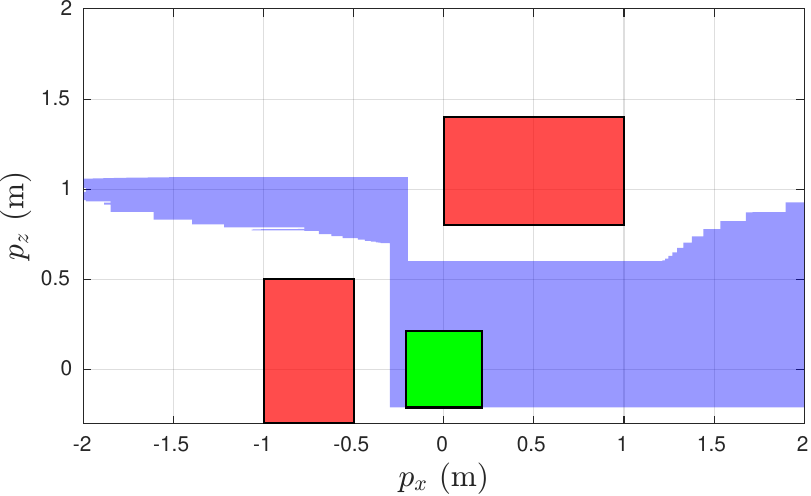}
\caption{Results of PARC on a na\"ive choice of planning model, without accounting for tracking error. Green indicates the goal region, red indicates the obstacles with the volume of the drone accounted for, and blue shows the reach set set-differenced with the avoid set without tracking error. Safe trajectories were successfully generated from the left of the leftmost obstacle.
}
\label{fig:quad2d_noerror}
\end{figure}
\begin{figure}
\centering
\includegraphics[width=0.8\columnwidth]{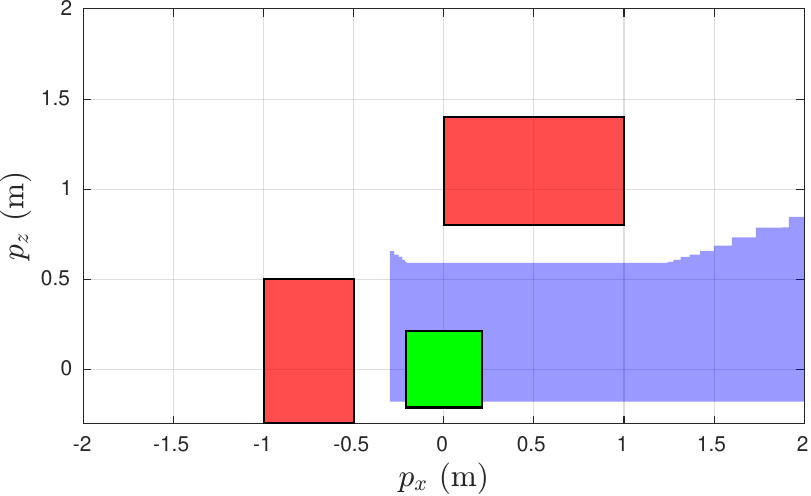}
\caption{Results of PARC on a na\"ive choice of planning model after accounting for tracking error. Green indicates the goal region, red indicates the obstacles with the volume of the drone accounted for, and blue shows the computed BRAS. There were no longer any safe trajectories generating from the left of the leftmost obstacle, which the Neural CLBF was able to accomplish.
}
\label{fig:quad2d_yeserror}
\end{figure}

\subsubsection{Discussion}
The decrease in BRAS volume was expected due to the many non-ETI states, leading to overconservativeness in avoid set computation.
Moreover, as the problem framework requires the trajectory parameters $\flf, \frt$ to be constant, the quadrotor shall continue accelerating after reaching the goal region, leading to instability should $\flf, \frt$ be tracked.

In Appendix \ref{app:additional_exp} and the next few examples, we show how a careful design of the planning model leads to much better performance on more complicated scenarios, as well as enabling stabilizing behavior such as braking without explicitly modeling velocity as a goal state.

\subsection{Impact of Planner-Tracker Cooperation}\label{exp:coop_impact}

We now study cooperative vs. adversarial planning and tracking models.
We also further study the Neural CLBF \cite{dawson2022safe} learning-based reach-avoid approach.
For this experiment, we use a 3-D near-hover quadrotor introduced in FaSTrack \cite{chen2021fastrack}, a reachability-based planner tracker that uses adversarial planning and tracking models.

\textbf{Key result:}
PARC takes advantage of cooperative planner-tracker behavior to reduce conservativeness.

\subsubsection{Experiment Setup}
The tracking model and PARC's planning model are listed in Appendix \ref{app:experiment}.
The quadrotor must navigate a narrow gap (width 1.8 m, compared to drone width 0.54 m) to reach its goal, as shown in Fig. \ref{fig:nearhover-single-comparison}.
We simulated 8,100 trials from different initial conditions.
We measured the number of trials that reached the goal, and the number that crashed.
We also considered the total computation time of each method.
Further details are in Appendix \ref{app:exp_add_details}.

\begin{figure}[!htb]
\centering
\includegraphics[width=1\columnwidth]{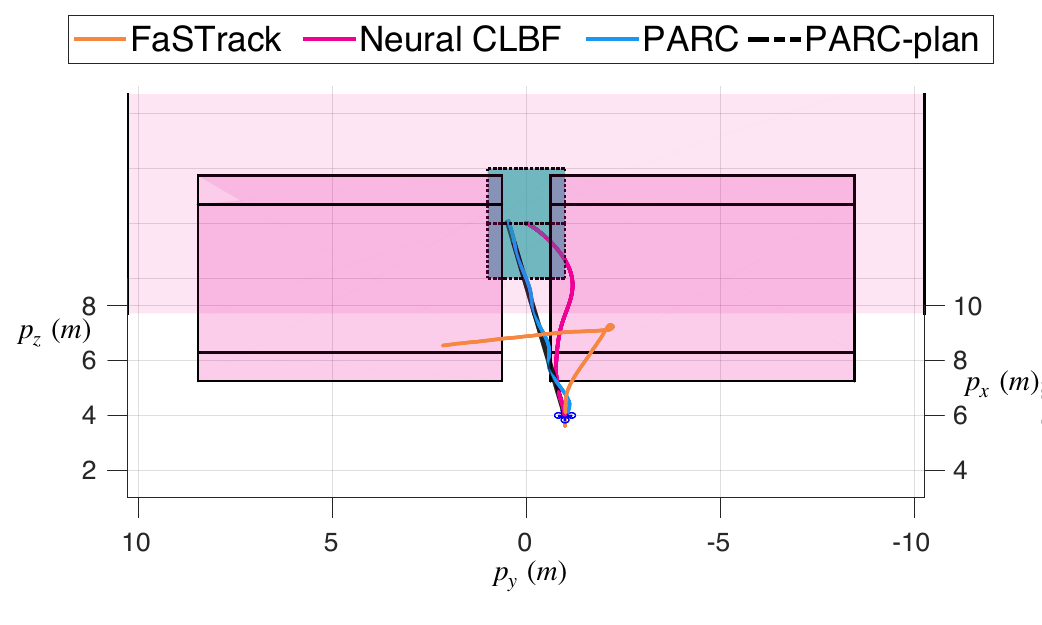}
\vspace{-2em}
\caption{Realized trajectory of the 10D quadrotor in the environment for 10 \unit{s} using three methods: PARC (blue), FaSTrack \cite{chen2021fastrack} (orange), and Neural CLBF \cite{dawson2022safe} (pink). The black dashed line denotes the plan computed using PARC.
The goal is represented as a green cube and the obstacles are the pink boxes.
The environment has a boundary that is not depicted in the figure.
}
\label{fig:nearhover-single-comparison}
\vspace{-1em}
\end{figure}

\subsubsection{Hypotheses}
We anticipated that our PARC planner would always reach the goal safely.
We also anticipated faster computation time in our reachability analysis compared to the other two approaches, although FaSTrack does not need to recompute the reachable set for every obstacle.
Since FaSTrack provides a formal proof of safety, we anticipate it would have no crashes.
However, since it treats the planner and tracker as adversarial, we expected it to struggle to reach the goal.
On the other hand, since CLBFs specifically balance safety and liveness, we expected it to reach the goal more often than FaSTrack but at the expense of crashing due to the high chance of learning an incorrect representation of safety when forced to operate near obstacles.

\subsubsection{Results}
The overall results are reported in Table \ref{table:experiments}.
Fig. \ref{fig:nearhover-single-comparison} illustrates results for all methods with an initial position of $[4, -1, 3]\tp$. 
FaSTrack (orange) fails to find a feasible plan that reaches the goal, the trajectory rolled out by Neural CLBF (pink) violates safety, while the trajectory rolled out by PARC (blue) finds a plan that reaches the goal safely despite tracking error.

Fig. \ref{fig:nearhover-multi-comparison} illustrates the trajectories simulated on the grid of 8,100 initial positions using each method.
Neural CLBF found 576 successful trajectories (7.1\%), while safety was violated for the rest.
FaSTrack found 0 successful trajectories but never violated safety.
However, PARC found the plan to reach the goal without collision for 2,592 initial states (32\%) and never violated safety.

FaSTrack took 1 \unit{h} to compute its Tracking Error Bound (TEB) (which only needs to be done once offline), while Neural CLBF took 5 \unit{h} to train for this obstacle configuration.
On the other hand, PARC computed the BRAS in 15.34 \unit{s}.

\begin{figure}[!htb]
\centering
\includegraphics[width=1\columnwidth]{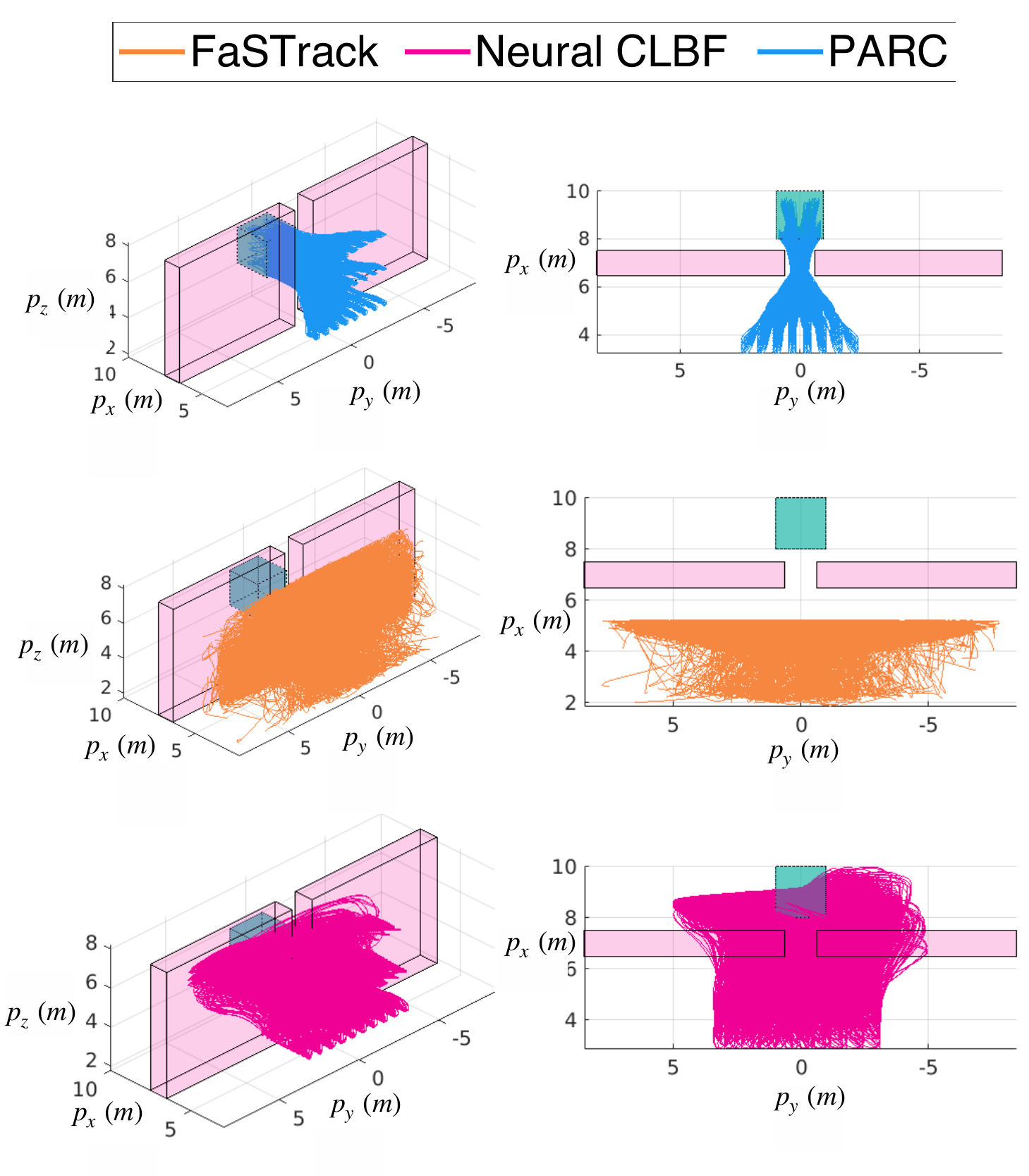}
\vspace{-0.2em}
\caption{Simulated trajectories of 10-D quadrotor for 8,100 initial states using three methods: PARC, FaSTrack \cite{chen2021fastrack}, and Neural CLBF \cite{dawson2022safe}. The green box is the goal set and the pink boxes are the obstacles.}
\label{fig:nearhover-multi-comparison}
\vspace{-0.5em}
\end{figure}

\begin{figure}[!htb]
\centering
\includegraphics[width=1\columnwidth]{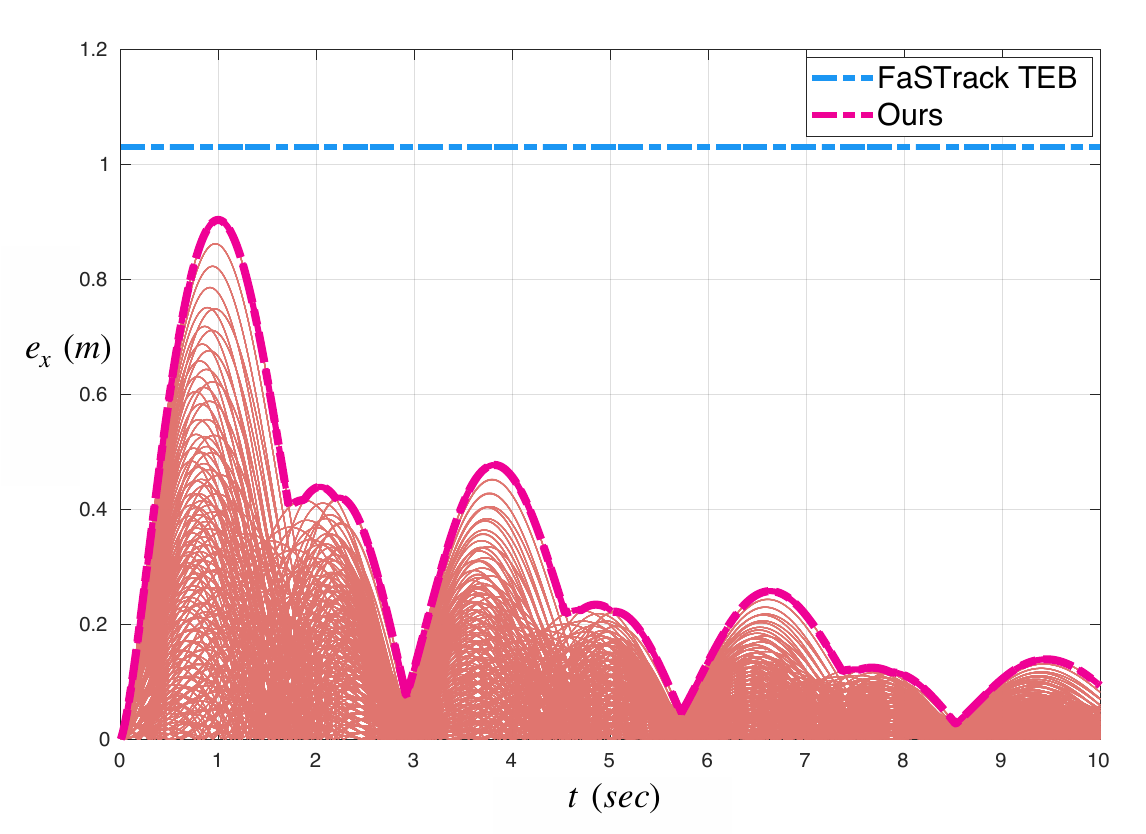}
\caption{Comparison of tracking error function between FaSTrack and PARC for planning horizon $\tf=10$ \unit{s}.
FaSTrack computes a fixed, worst-case tracking error bound for all time because it assumes an adversarial planner, whereas PARC takes advantage of the smaller, time-varying tracking error that results from a planner and controller working together.
The pale pink trajectories are samples used to compute the error as in Section \ref{subsec:tracking_error}.
}
\label{fig:nh_tracking_error}
\vspace{-1em}
\end{figure}

\subsubsection{Discussion}
The performance of FaSTrack is notably compromised by the large, conservatively computed TEB, specified as a $2.47 \times 2.47 \times 0.1$ box. 
This resulted in its inability to compute feasible plans in spaces tighter than the TEB.
However, by sampling tracking error under a parameterized plan, PARC's tracking error can be much smaller than TEB, as shown in Fig. \ref{fig:nh_tracking_error}. 
This enables PARC to find plans even in challenging scenarios.

Although Neural CLBF succeeded in capturing most of the safe, unsafe, and goal regions, the synthesized control policy would violate safety when starting at some of the initial states.
Thus, the learning-based method does not provide safety guarantees.

Since PARC has a strong safety guarantee, every plan it found reached the goal and did not collide with the obstacles.
We note that a better choice of planning model could increase PARC's success rate; the simplicity of the single-integrator planning model, employed mainly for fair comparison with FaSTrack, reduced the flexibility of possible plans.
In the next section, we will leverage a better choice of trajectory model.

\begin{table*}[!ht]
    \centering
    \setlength{\tabcolsep}{10pt} 
    \renewcommand{\arraystretch}{1.5} 
    \caption{
    \textbf{Performance Evaluation of PARC and Other Reach-Avoid Frameworks:} This table compares success and safety rates across various systems and scenarios. 
    The success rate is the percentage of simulations achieving their goals without collisions, while the safety rate measures that of obstacle avoidance. 
    Computation times are categorized into offline and online phases.
    The goals, obstacles, and initial states were known at the beginning of the online phase but not during the offline phase.
    We report the mean and the standard deviation for computation time.}
    \label{table:experiments}
    \begin{tabular}{l|l|r|r|r|r|r}
    \textbf{Scenario} & \textbf{Methodology} & \textbf{Success Rate} & \textbf{Safety Rate} & \textbf{Time (Online) (\unit{s})} & \textbf{Time (Offline) (\unit{s})} & \textbf{Section} \\
    \hline
    \multirow{3}{*}{\textbf{10-D Quadrotor}}
    & PARC & 0.32 & 1.00 & $15.34 \pm 1.44$ & 115 
    & \multirow{3}{*}{\ref{exp:coop_impact}}\\
    & FaSTrack & 0.00 & 1.00 & Timeout & 3,600 \\
    & Neural CLBF & 0.07 & 0.07 & 18,000 & 0 \\
    \hline
    \multirow{1}{*}{}
    
    & PARC-TEF & 0.07 & 1.00 & $65.42 \pm 9.73$ & 2,880
    & \multirow{2}{*}{\ref{exp:exact_plan_reach_impact}} \\
    \multirow{1}{*}{\textbf{13-D Quadrotor}}
    
    & RTD & 0.00 & 1.00 & Timeout & 2,880\\
    \cline{2-7}
    \multirow{1}{*}{\textbf{(Narrow Gap)}}
    & PARC-Funnel & 0.08 & 1.00 & $12.20 \pm 1.70$ & 0
    & \multirow{2}{*}{\ref{exp:time-varying}}\\
    \multirow{1}{*}{}
    & KDF & 0.00 & 1.00 & Timeout & 0\\

    \hline
    \multirow{1}{*}{\textbf{13-D Quadrotor}}
    & PARC-Funnel & 1.00 & 1.00 & $3.07 \pm 0.06$ & $0$
    & \multirow{2}{*}{\ref{exp:time-varying}} \\
    \multirow{1}{*}{\textbf{(Wide Gap)}}
    & KDF & 1.00 & 1.00 & $0.44 \pm 0.27$ & 0\\
    \end{tabular}
\vspace{-1em}
\end{table*}

\subsection{Impact of Tight Approximation for Planning Reachability}\label{exp:exact_plan_reach_impact}

We seek to understand the importance of computing a tighter approximation of the reachable set for the planning model.
To do this, we compare against RTD implemented on a general 3-D quadrotor \cite{kousik2019safe} using zonotope reachability \cite{althoff2015introduction}.
This allows us to use the same tracking error model but employs a different approach to planning model reachability analysis.

\subsubsection{Experiment Setup}
Our experimental setup is similar to that of Section \ref{exp:coop_impact} but with a narrower gap of 0.85 \unit{m} (the drone width is 0.54 \unit{m}).
See Appendix \ref{subsec:3dquad_env} for details.

We use the same planning model as quadrotor RTD \cite{kousik2019safe}, which generates time-varying polynomial positions in 3-D \cite{mueller2015computationally} with $\Delta\ts = 0.02$ \unit{s} with $\tf = 3$ \unit{s}; this formulation, detailed in Appendix \ref{subsec:3dquad_model}, is affine time-varying in the trajectory parameters with no non-ETI states, so reachability can be tightly approximated by PARC.
Both PARC and RTD use the same tracking error that is precomputed in \cite{kousik2019safe}.

We initialized the quadrotor in 675 uniformly sampled positions in front of the wall with zero initial velocity.
It only receives information about the obstacles when each of the experiments starts (i.e., PARC cannot precompute the BRAS).
We gave both methods 120 \unit{s} to attempt to find a trajectory to the goal through the gap.
Note that RTD is focused on fast, safe replanning, so we let it attempt to replan as often as possible.

We declare a trial safe if the quadrotor does not crash, and success if the quadrotor reaches the goal region without colliding with any obstacles.

\subsubsection{Hypothesis}
We anticipated that PARC would safely find ways to reach the goal more often, as its primary source of conservativeness is its tracking error estimation (which is identical to that of RTD).
We anticipated that RTD would reach the goal much less often because it suffers an additional source of conservativeness from using a polynomial planning model with zonotopic reachable sets.

\subsubsection{Results}
A visual comparison of RTD and PARC for one of the initial conditions is shown in Fig. \ref{fig:polyquad}.
The overall results are reported in Table \ref{table:experiments}.
On average, PARC took $65.42 \pm 9.73$ \unit{s} to compute a BRAS (i.e, a continuum of safe motion plans).
After BRAS computation, PARC was able to sample a safe motion plan (i.e., the BRAS is non-empty) in 45 out of 675 scenarios, all of which were successfully executed to reach the goal.
On the other hand, RTD was not able to reach the goal at all, despite frequently replanning an average of every 51 \unit{ms}.
RTD would eventually be stuck after flying too close to an obstacle (but not colliding).
Neither method had any collisions.

\begin{figure}
    \centering
    \begin{subfigure}[t]{0.9\columnwidth} 
        \centering
        \includegraphics[width=0.9\columnwidth]{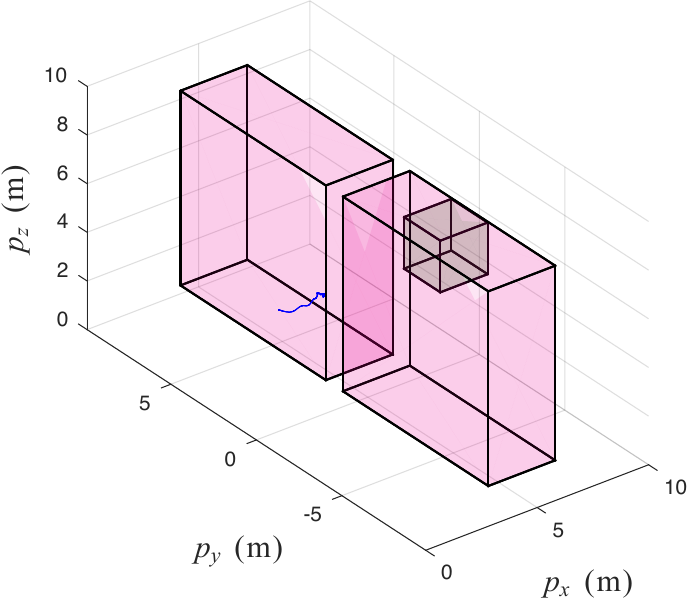}
        \caption{}
    \end{subfigure}
    \begin{subfigure}[t]{0.9\columnwidth}
        \centering
        \includegraphics[width=0.9\columnwidth]{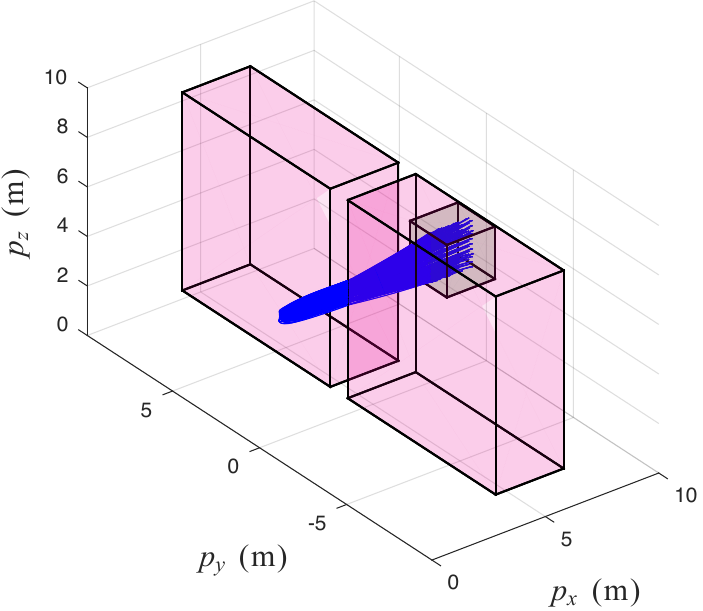}   \caption{}
    \end{subfigure}
    \caption{One run of the experiment with general quadrotor model in 3-D using RTD in (a) and PARC in (b).
    The pink boxes are the obstacles and the green cube is the goal set.
    RTD's overall trajectory after 120 \unit{s} of replanning is shown in blue, unable to reach the goal.
    PARC generates the set of initial conditions that all produce \textit{guaranteed} safe plans that reach the goal and avoids the obstacles.
    We sample PARC's BRAS to find 308 safe trajectories (including tracking error), shown in blue.
    }
    \label{fig:polyquad}
    \vspace*{-0.5cm}
\end{figure}

\subsubsection{Discussion}
As expected, both RTD and PARC guaranteed safety.
However, only PARC was able to traverse through the narrow gap.
Thus, this example clearly illustrates the utility of PARC's tightly approximating reachable set computation in comparison to RTD's approximation.
Although PARC is significantly slower in planning when compared to RTD, once a BRAS is computed, the robot obtains a continuum of safe plans.

The main slowdown of PARC's computation comes from the avoid set computation in \eqref{eq:int_avoid_set} and \eqref{eq:avoid_set_chain}, whose computation time increases linearly with the number of time steps.
In this example, we chose the small timesteps of 0.02 \unit{s} over the planning horizon of 3 \unit{s} for a fair comparison with RTD.
If a faster computation time is preferred over a finer trajectory model with potentially less tracking error, one can decrease the number of timesteps when defining the PWA system.
Furthermore, our PARC implementation is not yet parallelized, which may lead to further speedups.


\subsection{Impact of Time-Varying Tracking Error} \label{exp:time-varying}

\begin{figure}
\centering
\includegraphics[width=0.9\columnwidth]{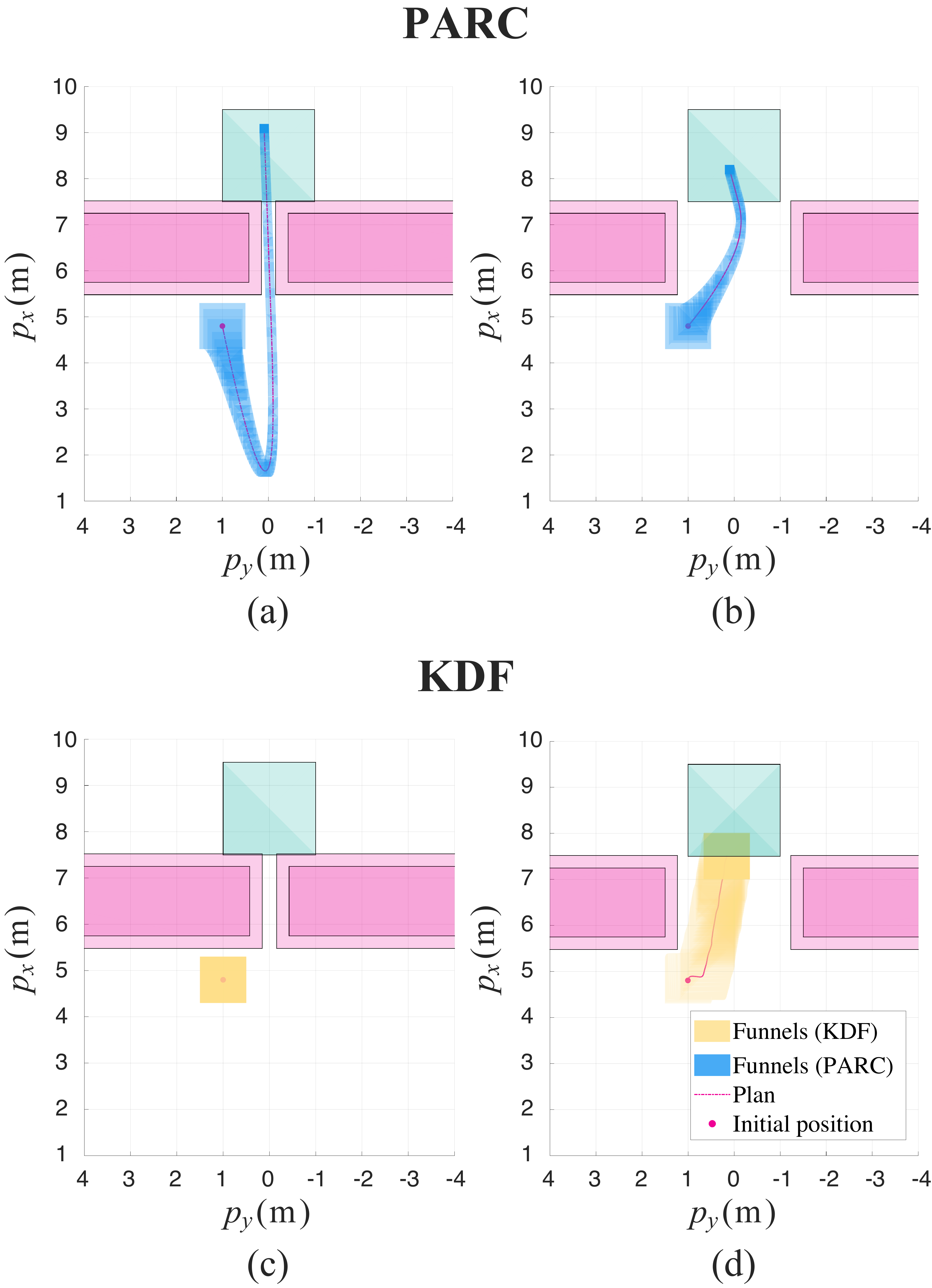}
\caption{
Comparison of PARC (a, b) and KDF (c, d) motion planning for a 3-D quadrotor across two scenarios: a narrow gap of $0.85 \unit{m}$ (left) and a wide gap of $3 \unit{m}$ (right), projected onto the $p_x$-$p_y$ plane. 
The initial planning state is $\planv_0 = [4.8, 1, 4]^{\trans}$ with a planning horizon of $\tf=15 \unit{s}$. 
The goal set is the green box, the actual obstacles are the dense pink boxes, and the lighter pink boxes show the obstacles buffered by the drone's volume. 
Pink trajectories within blue or yellow volumes represent the computed motion plans and bounds of realized trajectories when tracking the plan (i.e. funnels), illustrating how each method handles tracking errors at the motion-planning level.
}\label{fig:parc_kdf_compare}
\vspace{-2em}
\end{figure}

Finally, we explore the importance of accounting for the temporal correlation of tracking errors.
In this experiment, we compare our approach with the KDF framework that combines sampling-based planning with funnel-based feedback control \cite{verginis2022kdf}.
This control mechanism prescribes the tracking error to remain in funnels described by a performance function, $\rho(t)$, which decays exponentially over time \cite{bechlioulis2014low}.
Employing the same funnel-based control across both the KDF and PARC methodologies allows us to maintain a consistent tracking error model.
This setup confirms that any observed differences in performance are solely attributed to the planner's strategy of \textit{using} tracking errors.

\textbf{Key result:}
By considering time-varying as opposed to worst-case tracking error, PARC significantly reduces the conservativeness of motion planning.
Furthermore, PARC's use of parameterized trajectories enabled the design of a funnel controller that maintained within the prescribed tracking error bounds, whereas the same controller could not reliably track the smoothed output of a sampling-based motion planner.

\begin{rem}
As shown in the other experiments, PARC can leverage a data-driven tracking error function that is potentially less conservative than the prescribed $\rho(t)$.
In this experiment, we use $\rho(t)$ to ensure a fair comparison between PARC and KDF.
\end{rem}

\subsubsection{Experiment Setup}
We extend the setup from Section \ref{exp:coop_impact} to compare PARC and KDF using gap widths of 0.85 \unit{m} and 3 \unit{m}, relative to the quadrotor's 0.54 \unit{m} of body width.
We evaluate success and safety ratios, along with computation times, by simulating the quadrotor from zero velocity at 3,700 and 2,400 uniformly distributed initial positions across each scenario.
See Appendix \ref{subsec:kdf_env} for details.

For the planner, PARC employs the polynomial planning model from Section \ref{exp:exact_plan_reach_impact}, but with $\tf=15 \sec$. 
$\Delta \ts$ is $0.2 \sec$ for the narrow gap and $0.5 \sec$ for the wide gap.
KDF employs RRT \cite{lavalle1998rapidly} as its motion planner.

For the tracker, both PARC and KDF use a specialized funnel controller adapted for an underactuated 3-D quadrotor proposed in \cite{lapandic2022robust}, as the conventional KDF funnel-based feedback controller is unsuitable for the system \cite{lapandic2023kinodynamic}. 
This controller conforms to the performance function $\rho(t) = \rho_0 + (\rho_0 - \rho_\infty)e^{-\lambda t}$ where we choose $\rho_0=0.5, \rho_\infty=0.1, \lambda=0.5$.
We note that this controller does \textit{not} respect input limits.
The 48 hyperparameters for the funnel controllers are reported in Appendix \ref{subsec:kdf_param}.

Despite extensive tuning efforts, a universal hyperparameter set for KDF could not be determined that respected the prescribed tracking error bounds; thus, we assume the ideal performance of the funnel controller, and base our statistics on reference trajectories ignoring tracking error.
\textit{In other words, the comparison is biased towards benefiting KDF and disadvantaging PARC.}
We note that we designed controllers that function for a selected few initial states with KDF, the results of which are shown in Appendix \ref{subsec:kdf_param} and Fig. \ref{fig:parc_kdf_compare}.
In contrast, since PARC uses parameterized reference trajectories, we were able to succesfully identify a universal set of hyperparameters that enabled us to base the statistics on realized trajectories that included tracking error.

\subsubsection{Hypothesis}
In the wide-gap scenario, we expected both PARC and KDF to successfully find and execute reach-avoid plans. 
We anticipated KDF to demonstrate faster computation times, benefiting from its sampling-based planning approach. 
In the narrow-gap scenario, KDF is expected to struggle due to its reliance on maximum funnel values for free space construction in RRT, whereas PARC is expected to be more successful by leveraging the time-varying nature of tracking errors.
We expected no safety violations for all experiments.

\subsubsection{Results}
In the wide-gap scenario, both PARC and KDF successfully computed plans from all 2,400 initial positions to reach the goal safely. 
On average, PARC computed the BRAS in 2.86 \unit{s} and sampled safe plans from the BRAS in 0.21 \unit{s}, slower than KDF’s online planning time of 0.46 \unit{s}. 
However, in the narrow-gap scenario, where the clearance is only 0.31 \unit{m}, KDF could not generate any plans, while PARC managed to create valid plans for 341 out of 3,700 positions, taking 9.5 \unit{s} to compute BRAS and 2.7 \unit{s} to sample plans on average.
Results are detailed in Table \ref{table:experiments}.
\subsubsection{Discussion}
As anticipated, PARC demonstrates superior performance over KDF in narrow-gap scenarios due to its ability to utilize the time-varying nature of tracking errors effectively; thus taking advantage of a low tracking-error regime.
This advantage is visually depicted in Fig. \ref{fig:parc_kdf_compare}a, where PARC waits for the funnel controller to reduce the tracking error before navigating through narrow gaps.
In contrast, for wider gaps shown in Fig. \ref{fig:parc_kdf_compare}b, the quadrotor proceeds through the gap immediately.
Conversely, KDF struggles in narrow spaces, as depicted in Fig. \ref{fig:parc_kdf_compare}c since its planner computes paths based on a static maximum tracking error that exceeds the actual gap width (i.e., $\rho_0 = 0.5 > 0.31$), but performs well in wider spaces where this limitation does not impede planning, as depicted in Fig. \ref{fig:parc_kdf_compare}d.

Notably, the difficulty in tuning the funnel controller to follow the smoothed reference trajectory proposed by RRT has led us to assume perfect funnel control in reporting KDF’s statistics. 
This aligns with observations in Section \ref{exp:coop_impact} that parameterized trajectories foster better cooperation with controllers compared to trajectories derived from smoothing splines over RRT waypoints, which can create adversarial tracking conditions.

PARC achieves reach-avoid plans in narrow-gap scenarios but faces lengthy computation times due to detailed sampling from BRAS and smaller $\Delta t$ steps. 
In contrast, KDF uses sampling-based motion planning effectively in scenarios with adequate gap widths, swiftly creating viable plans. 
This indicates a potential research direction: integrating sampling-based motion planning within PARC could shorten computation times and facilitate the generation of expert trajectories that adeptly handle near-danger scenarios, surpassing conventional kinodynamic sampling-based motion planning methods such as KDF.

\subsection{Overall Discussion}
Through the above experiments, we have shown how our proposed PARC approach confers advantages in computing and using BRAS.
In near-danger scenarios, PARC is less conservative than FaSTrack and KDF since the former uses a time-variant tracking error function, whereas the latter two assume constant, worst-case tracking errors.
While RTD can similarly handle time-variant tracking error functions, PARC still outperforms it due to the lower numerical approximation error from representing the planning model as a PWA system instead of a polynomial and using H-polytopes instead of zonotopes.
Finally, Neural CLBF cannot provide hard reach-avoid guarantees, resulting in many crashes with obstacles.

However, PARC still has room for improvement.
Most obviously, the safe plan computation time is on the order of seconds for just a few obstacles; this is too slow for real-time operation.
In contrast, RTD \cite{kousik2019safe}, although more conservative, can compute a new plan about every 51 \unit{ms} in Section \ref{exp:exact_plan_reach_impact}.
Typically, we expect motion planners with real-time capabilities to compute new plans at least every 0.5 \unit{s} \cite{kousik2020bridging}.
To achieve this for PARC, parallelization needs to be implemented in future work.
Furthermore, the planning model requires careful hand-crafting as shown in our experiments.
\section{Drift Vehicle Demonstration}\label{sec:drift_demo}

We now demonstrate the utility of PARC in safely planning extreme near-danger vehicle drifting maneuvers, where the robot loses controllability.

\textbf{Key result:} PARC enables computation of provably safe drift parking trajectories for the first time, extending the state of the art.

\subsection{Background}

\subsubsection{Drifting}
Drifting is when a vehicle turns with a high sideslip angle, causing tire saturation and losing direct control of its heading \cite{goh2016simultaneous}. 
Stable drifting can be achieved by finding ``drift equilibria'' with constant yaw rate and sideslip angle of the highly nonlinear dynamics \cite{goh2016simultaneous, goh2020toward, goel2020opening, weber2023modeling}.
Drift parking, however, remains challenging due to the highly coupled inputs \cite{leng2023deep} and reduced controllability as the vehicle specifically must not maintain a drift equilibrium \cite{kolter2010probabilistic}. 

\subsubsection{Reach-Avoid Method Comparison}
To the best of our knowledge, no other method can guarantee safety and liveness for drift parking.

Our drifting tracking model dynamics (see Appendix \ref{app:drift_details}) are not control affine, so directly applying a control barrier function (CBF) or CLBF control approach is not applicable.\cite{ames2016control,dawson2023safe}.
Note that \cite{dawson2022safe} poses a control affine vehicle dynamics model but does not include tire saturation, so it cannot drift.
Furthermore, our planning and tracking models are too high-dimensional for standard Hamilton-Jacobi-Bellman (HJB) value function computation \cite{chen2021fastrack,mitchell2008flexible}.
We also leave a comparison against learning-based HJB \cite{bansal2021deepreach} approaches as future work.
It is not obvious to design a valid funnel controller for such an under-actuated, hybrid system as a drifting vehicle to compare with KDF \cite{verginis2022kdf}.
The most relevant work is RTD \cite{kousik2020bridging,liu2022refine}, which has not been applied to these extreme dynamics.
In future work, RTD's reachable set can use PARC's formulation to enable a direct drifting comparison, such that RTD's only difference with PARC is that it implicitly represents the avoid set.
Due to this similarity, we focus on demonstrating PARC's capabilities as opposed to a head-to-head comparison.

\subsection{Demo Details}

Extensive details are presented in Appendix \ref{app:drift_details}.
Here, we briefly present setup, simulation results, and a discussion.

\subsubsection{Planning and Tracking Models}
As we saw in Section \ref{exp:planning_model_design}, a careful planning model design is critical to PARC's reach-avoid guarantees.
In this case, we fit an affine time-varying model to drifting data collected by rolling out open-loop drifting maneuvers.
This not only guarantees a reasonable planning model (i.e., it successfully parameterizes drifting maneuvers) but also demonstrates the flexibility of using PWA models for PARC since they can readily be learned from expert demonstrations.

We use a tracking model based on \cite{jelavic2017autonomous,goh2020toward,weber2023modeling,goh2016simultaneous}, detailed in Appendix \ref{app:drift_details}.
Our tracking controller is based on \cite{jelavic2017autonomous}:
we first use nonlinear MPC to bring the car into a drifting state, then switch to open-loop control to complete the parking action (since the vehicle loses controllability when sliding sideways to a stop).

\subsubsection{Environment and Demo Setup}

The objective of this demo is to drift safely into a parallel parked state between two closely parked cars, as shown in Fig. \ref{fig:front_figure} and Fig. \ref{fig:drift_examples}.
Our environment is inspired by a real-world demo video \cite{guinness2014parallel}.

We assess PARC's ability to compute the BRAS for this maneuver.
We further evaluate safety trajectory rollouts using initial conditions and trajectory parameters sampled from the BRAS to check for collision.

\subsubsection{Results}
\begin{figure}[t]
\centering
\begin{subfigure}[t]{0.99\columnwidth}
    \includegraphics[width=0.99\columnwidth]{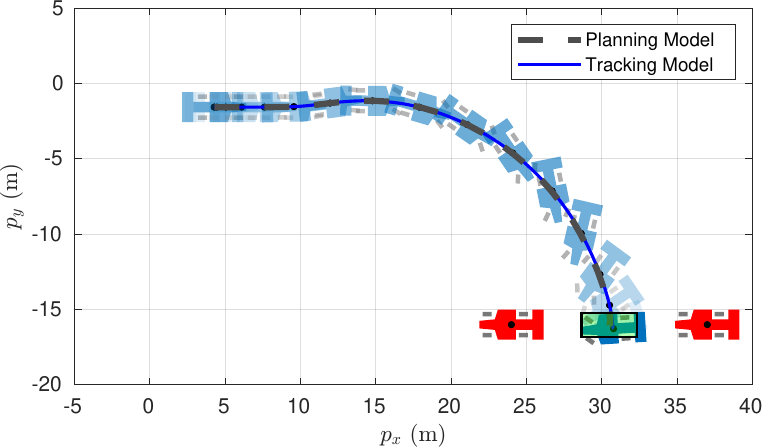}
    \caption{}
    \label{fig:best_drift}
\end{subfigure}
\begin{subfigure}[t]{0.99\columnwidth}
    \includegraphics[width=0.99\columnwidth]{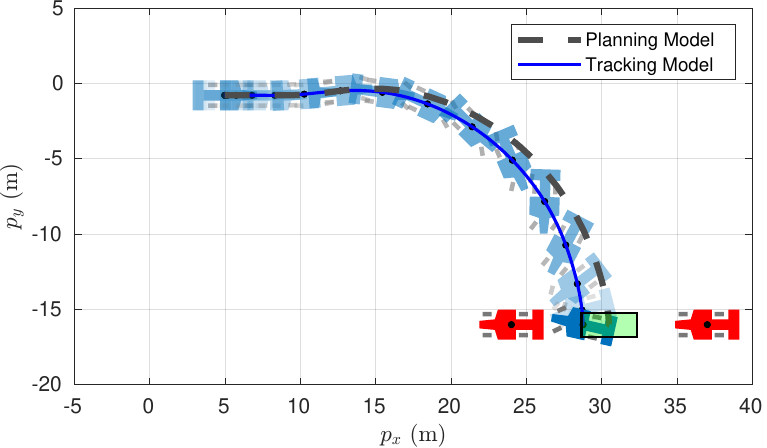}
    \caption{}
    \label{fig:worst_drift}
\end{subfigure}
\caption{Examples of planned drifting trajectories sampled from the Backward Reach-Avoid Set using PARC.
The two red cars are the obstacles. 
The goal set for the center of mass is shown in green.
The trajectory with a very low tracking error is shown in (a), whereas that with a high tracking error is shown in (b).
Critically, \textit{none} of the sampled trajectories from the BRAS collide with the obstacles when tracked.}
\label{fig:drift_examples}
\end{figure}

The BRAS computed by PARC for drift parking is shown in Fig. \ref{fig:drift_BRAS}. 
The BRAS took 5.52 \unit{s} to compute.
We see that with the chosen controller and learned planning model, every pair of initial planning state and trajectory parameter in the BRAS results in safe drift-parking of the car.
Fig. \ref{fig:best_drift} and Fig. \ref{fig:worst_drift} show the trajectories with the best and worst tracking error; in both cases, since the trajectories were sampled from the BRAS, the vehicle satisfied both safety and goal-reaching.

\subsubsection{Discussion}
This demo showcased (i) PARC's capability to generate safe and goal-reaching trajectories with challenging near-danger tasks such as reach-avoid drift-parking, and
(ii) the flexibility of PARC in being compatible with a planning model learned from data.
We notice that the BRAS is relatively small (see Fig. \ref{fig:drift_BRAS}) in the state space, indicating that drift parking is indeed a difficult maneuver.
This emphasizes the importance for PARC to minimize conservativeness in the BRAS computation.
However, it also suggests that this maneuver is brittle to large changes in initial conditions, and PARC will struggle to find drifting maneuvers for different goal sets or obstacle configurations.
In future work, we plan to test more robust drifting controllers to produce a larger BRAS.
With that in mind, our ultimate goal is physical hardware experiments on a platform such as the F1:10 car in Fig. \ref{fig:F1tenth}.

\section{Hardware Demonstration}\label{sec:hardware_demo}
Thus far, PARC has only been shown to work in simulation.
In reality, the tracking model will have model mismatch with the actual hardware due to disturbances and perception errors.
As such, we now present preliminary results for PARC's reach-avoid guarantees applied on mobile robot hardware navigating a narrow gap, as shown in Fig.~\ref{fig:tb_hw_demo}.
Specifically, we use a TurtleBot3 Burger Model \cite{grossimplementation} differential drive robot running Robot Operating System (ROS) 2 \cite{macenski2022robot}.

\subsection{Environment and Demo Setup}
In this demo, we do not consider object detection error, and instead assume full knowledge of the geometry and location of the obstacles.
The location of the goal set and the obstacles are shown in Fig. \ref{fig:tb_hw_demo}.
The robot, with a diameter of 0.21 \unit{m}, must navigate through two obstacles with width of 0.13 \unit{m} and length of 0.35 \unit{m}, placed 0.32 \unit{m} apart and 0.12 \unit{m} in front of the goal set, which is 0.6 \unit{m} $\times$ 0.6 \unit{m} in size.

\begin{figure}
\centering
\includegraphics[width=1\columnwidth]{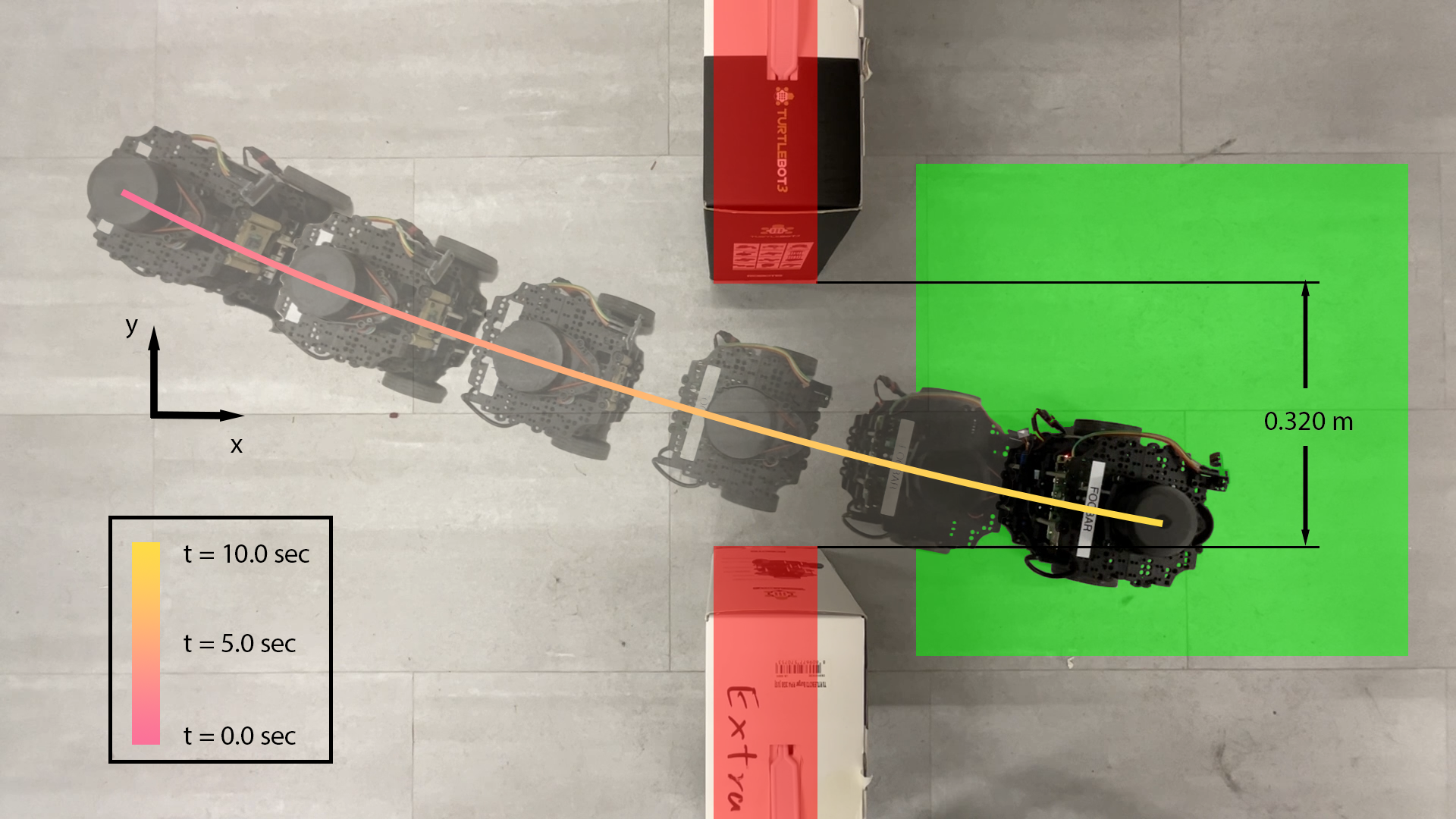}
\caption{
Problem setup of the hardware demo.
A snapshot of the robot's realized motion and the tracked trajectory (gradient line) from one run of the experiment was shown.
The robot successfully avoided the obstacles (red) and reached the goal set (green).
}
\label{fig:tb_hw_demo}
\vspace{-1em}
\end{figure}

The robot is equipped with IMU sensors and wheel encoders, which gives estimation of its planning states $\px$, $\py$, and $\theta$.
As future work, these states can be more accurately estimates using an onboard lidar sensor.

\subsection{Planning and Tracking Models}
We use the same planning and tracking model as Section \ref{subsec:turtlebot_example}, with $\K = [-0.4\ \unit{rad/s}, 0.4\ \unit{rad/s}]\times[0\ \unit{m/s}, 0.1\ \unit{m/s}]$, $\tf = 10$ \unit{s}, and $\Delta\ts = 0.5$ \unit{s}.
The robot uses a proportional-derivative (PD) controller to track the desired trajectories.

The tracking error was collected in Gazebo simulations \cite{koenig2004design} in ROS 2 \cite{macenski2022robot}.
The positional error from sim-to-real was experimentally determined to be less than 0.005 \unit{m}, mainly due to localization drift from IMU and encoders.
To account for this, we added 0.005 \unit{m} to the collected tracking error in the horizontal and vertical directions.
This is illustrated in Fig. \ref{fig:tb_tracking_error}.

\begin{figure}
\centering
\includegraphics[width=1\columnwidth]{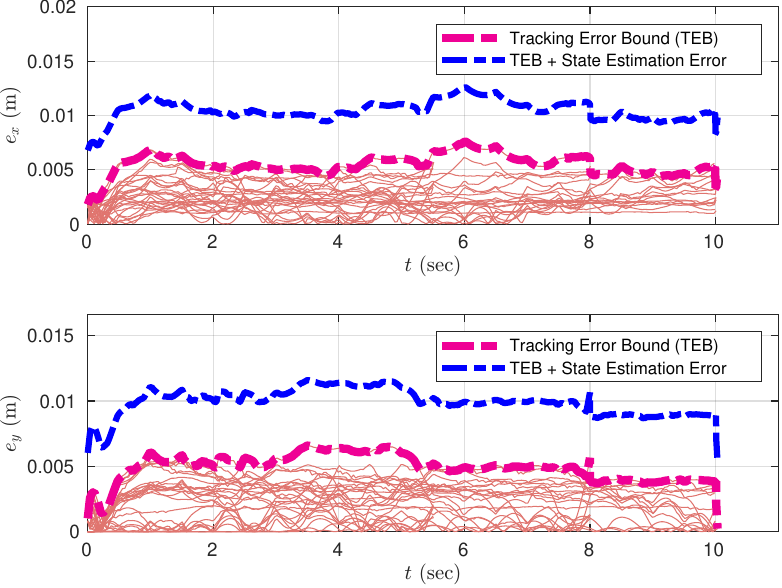}
\caption{
Tracking error function of the hardware demo.
The pale pink lines were tracking error in the workspace $\px$ and $\py$ collected across multiple Gazebo simulations \cite{koenig2004design} in ROS 2 \cite{macenski2022robot}.
The magenta line indicates TEB, the maximum tracking error recorded across the simulations.
To account for positional error from sim-to-real, 0.005 \unit{m} was added to the TEB, as shown in blue.
The blue line is tracking error ultimately used in BRAS computation.
}
\label{fig:tb_tracking_error}
\vspace{-1em}
\end{figure}
\subsection{Results}
We placed the robot at 17 different initial configurations.
On average, computing the BRAS took 1.65 \unit{s}.
In all experiments, the robot successfully avoided the obstacles and reached the goal, despite the gap being less than two times the diameter of the robot.
Snapshot of the realized motion and the tracked trajectory for one of the initial configuration is shown in Fig. \ref{fig:tb_hw_demo}.

\subsection{Discussion}
This demo presented preliminary results on how PARC can be extended to hardware and real robots.
We acknowledge that (i) the chosen robot is slow, with maximum velocity of 0.1 \unit{m/s}, so the tracking error is rather small, (ii) the model mismatch from sim-to-real is not challenging to account for, with little disturbance present and a bound on the state estimation error can be easily established, and (iii) the obstacles' exact position and shape are assumed to be known, with no object detection and estimation error.
We leave experiments with more complicated hardware and challenging environments as future work, since the focus of this paper is computing the BRAS in the first place.
That said, we are confident that PARC can be extended to more difficult hardware problem setups, as shown on other similar set-based planner-tracker reach-avoid methods \cite{kousik2020bridging, michaux2023can,vaskov2019towards,chen2021fastrack}.
\section{Conclusion}\label{sec:conclusion}

This paper has presented an approach to use Piecewise Affine Reach-avoid Computation (PARC) for safe goal-reaching robot trajectory planning near danger.
PARC enables goal-reaching through narrow gaps and in tight obstacle configurations, even with extreme dynamics.
The method outperforms a variety of state-of-the-art reach-avoid methods and establishes a new benchmark for safe, extreme vehicle motion planning.
Future work will improve PARC's computation time, consider more types of uncertainty as opposed to lumped tracking error, and extend beyond mobile robots.

\textit{Limitations:}
PARC has several key limitations.

First and foremost, it requires an assumption that the data coverage of sampled tracking error is large enough.
While this assumption holds in practice (over hundreds of collision-free trials planned in this paper), it means that further mathematical work is necessary for true safety guarantees.

Second, if the planning of the robot takes place at the joint level, which is common for manipulator and legged robots, augmentation of the obstacles with tracking error is not straight-forward.
Though this has been successfully done in other set-based planner-tracker reach-avoid methods such as Autonomous Reachability-based Manipulator Trajectory Design (ARMTD) \cite{holmes2020reachable} and Autonomous Robust Manipulation via Optimization with Uncertainty-aware Reachability (ARMOUR) \cite{michaux2023can} by expressing H-polytopes in configuration space as polynomial zonotopes \cite{kochdumper2020reachability} in workspace, we opted to leave its extension to PARC as future work, as the focus of this paper is just computing the BRAS in the first place.

Third, although PARC is faster than compared methods, it is still too slow for real-time; the present implementation has not been parallelized and may benefit from alternative representations to H-polytopes.

Fourth, depending on the choice of planning model, PARC could be difficult to implement on long-horizon problem settings.
This is because the assumption of having obtained an expert plan \textit{with constant trajectory parameters} that connects the initial position to the goal set (Assumption \ref{as:expert_traj}), if such plan exists at all, gets increasingly difficult to fulfill the longer the planning horizon is \cite{bonalli2019gusto, lofberg2004yalmip}.
Though a good choice of planning model (e.g. time-variant affine systems) would alleviate this problem, we restrict the problem setups in this paper to short-horizon (i.e. near danger) settings to avoid overclaiming the capabilities of PARC.

Lastly, PARC has only been compared on hand-crafted examples of goal-reaching near danger, and hardware validation has only been done on a simple robot. 
The experiments have only considered tracking error, with the hardware validation considering error from localization drift of the IMU and encoders in addition.
It is critical to validate the method across randomly-generated, realistic environments on a variety of hardware with multiple significant sources of uncertainties, such as perception error.
To this end, the authors are preparing a small robotic car to perform drift parking experiments.

\section*{Acknolwedgments}
We are grateful to the anonymous reviewers and editor's suggestions for additional comparisons and presentation.
We would also like to thank Trey Weber for assistance with the drifting dynamics, and Christos Verginis for assistance with implementing KDF.

\renewcommand{\bibfont}{\normalfont\footnotesize}
{\renewcommand{\markboth}[2]{}
\printbibliography}

\appendix
\subsection{Related Work}\label{app:related_work}

We now discuss methods that guarantee safety (avoiding dangerous states) and liveness (reaching a goal) for control systems and robot motion planning.
In particular, we discuss optimal control methods, safety filters, and planner-tracker frameworks.
We also note connections between our approach and sampling-based motion planning.
Finally, we discuss PWA system models, which we leverage to overcome limitations of conservativeness and computational tractability for providing safety and liveness guarantees with state-of-the-art methods.

\subsubsection{Optimal Control Methods}
Many safety methods focus on synthesizing or analyzing controllers.
There has been extensive development of techniques such as CBFs \cite{ames2016control, agrawal2017discrete, rosolia2022unified, castaneda2022probabilistic}, Hamilton-Jacobi (HJ) reachability \cite{herbert2021scalable, seo2022real}, sums-of-squares (SOS) reachability \cite{majumdar2017funnel,tedrake2010lqr}, robust MPC \cite{rakovic2012parameterized, villanueva2017robust, wabersich2021predictive}, and graphs of convex sets (GCS) \cite{marcucci2023motion, cohn2023non}.
Yet, only a minority of these methods address the dual objectives of liveness and safety \cite{hsu2021safety, so2023solving, borquez2023safety,fridovich2019safely}. 
Some methods seek a balance between performance and safety by implicitly integrating each constraint into the cost function of optimal control problem, solved through methods like MPC \cite{williams2016aggressive} or reinforcement learning (RL) \cite{srinivasan2020learning}, but can fail to offer hard satisfaction of such constraints.

The key challenge with most optimal control methods is to synthesize a controller that accounts for arbitrary bounded disturbances to a model of a robot.
By contrast, in this work, we seek to synthesize a planner that cooperates with, and leverages knowledge about, an imperfect controller.

\subsubsection{Safety Filters}
A related strategy prioritizes safety through the use of safety filters \cite{wabersich2021probabilistic, wabersich2023data}; indeed, many optimal control methods can also be cast as safety filters \cite{hsu2023safety,ames2016control, agrawal2017discrete}.
In this case, a performance-oriented controller is minimally altered by a safety-oriented controller when the system is \textit{about to} violate the safety constraints, mostly building on the safe invariant sets.
Since this approach prioritizes safety over performance, thereby enforcing the hard satisfaction of safety in any case, the safety filter might produce conservative behavior.
One way to mitigate conservativeness is to leverage reinforcement learning to exploit the safety filter \cite{thumm2022provably,shao2021reachability}.

The key challenge with safety filters is that they intervene at the control level, and typically treat upstream planning as a disturbance \cite{chen2021fastrack, singh2020robust}.
This can severely impact liveness, because the planner and controller are treated as adversarial, as opposed to cooperative.
By contrast, we adopt a cooperative framework.

Another notable challenge, which we do not address, is extending to systems with very high-dimensional states or input/observation spaces.
While some traditional approaches have tried a divide-and-conquer approach for high dimensions \cite{chen2016fast}, the most promising path forward appears to be learning-based methods \cite{bansal2021deepreach,michaux2023reachability,dawson2022safe,xiao2023barriernet,selim2022safe}.
However, by relying on deep or machine learning, these methods essentially lose all safety guarantees above 12-D.
We look forward to potentially attacking high-dimensional systems by fusing our proposed method with learning in the future.

\subsubsection{Planner-Tracker Methods}
The final category of safe planing and control methods that we discuss is the \textit{planner-tracker} framework \cite{chen2021fastrack,kousik2019safe,kousik2020bridging,michaux2023reachability}.
These methods use a simplified model to enable fast plan generation; plans are then tracked by a controller that uses a high-fidelity tracking model of a robot.

Two representative approaches are FaSTrack \cite{chen2021fastrack} and RTD \cite{kousik2019safe,kousik2020bridging}.
FaSTrack employs HJ reachability to precompute a worst-case error bound between the planning and tracking models, then dilates obstacles by this bound to ensure safe planning online.
RTD precomputes a forward reachable set (FRS) of a parameterized set of planning model trajectories and associated (time-varying) tracking error, then uses the FRS to find safe plans at runtime.
We compare our proposed method against both FaSTrack and RTD.

The primary challenge for planner-tracker methods is that the computational advantage of the planning model comes at the cost of model discrepancy \cite{chen2021fastrack, singh2020robust, kousik2017safe, kousik2019safe}.
An additional key challenge is that the numerical representation of safety can be overly conservative even when just representing the planning model and ignoring model error.
In this work, we design a planner-tracker framework that \textit{exactly} represents trajectories of the planning model using piecewise affine systems, meaning that we only need to account for model error.
We find that this significantly reduces conservativeness over related methods.

\subsubsection{Sampling-Based Motion Planning}
The challenge of planning near danger, or through a narrow gap, has been well-studied with classical sampling-based motion planning techniques \cite{elbanhawi2014sampling} and is still an active area of research (e.g., \cite{orthey2024multilevel,liu2023simultaneous,yu2023gaussian}).
When paths or trajectories \textit{must} pass near obstacles, it becomes critical to accurately represent dynamics and account for error in modeling and control.
To this end, methods such as kinodynamic RRT* \cite{webb2013kinodynamic} require repeatedly rolling out and collision-checking the trajectories, as well as solving two-point boundary value problems of nonlinear dynamical systems (i.e., steer function) online, which can be computationally expensive.
If in addition the dynamical system is high-dimensional and/or the set of true feasible reach-avoid set is small (which is the case for ``near danger'' problem setups), it could take a very long time for a solution to be found \cite{berenson2009manipulation}.

In this paper, instead of sampling to find paths or trajectories, we directly approximate a \textit{continuum} of all feasible trajectories through a narrow gap.
That said, we do not view our approach as \textit{competing} with sampling-based motion planning.
By contrast, our PARC planner is complementary to sampling in three ways.
First, PARC requires \textit{expert plans} that seed our reachability analysis; in this paper, we generate the plans by approximating the solution to a two-point boundary value problem, but we could leverage a sampling approach instead.
Second, as we show in our experiments, PARC is only necessary when a goal is specifically near obstacles; a sampling-based approach that ignores dynamics could be used to find goals and identify these scenarios (e.g., if the sampling-based approach fails, then one should switch to using PARC).
Third, our approach can be used to \textit{robustify} a sampling-based approach by computing safety funnels around a sampled plan, similar to linear quadratic regulator (LQR) trees \cite{tedrake2010lqr} or funnel libraries \cite{majumdar2017funnel}.

\subsubsection{Piecewise-Affine Systems}

In this paper, we propose a parameterized, time-variant PWA system as our planning model. 
PWA systems, which evolve through different affine models based on state regions (modes), adeptly capture nonlinear dynamics.
Their capacity for modeling non-smooth dynamics through hybridization is well-documented \cite{asarin2007hybridization, dang2010accurate} and they are prevalently used to approximate complex nonlinear hybrid systems, including legged locomotion \cite{marcucci2017approximate}, traffic systems \cite{mehr2017stochastic}, and neural network \cite{vincent2021reachable}.

However, planning with PWA systems, which involves solving their control problem, is recognized as challenging due to the necessity of determining both input and mode sequences \cite{sadraddini2019sampling}. 
Standard solutions, including mixed-integer convex programming (MICP) \cite{marcucci2019mixed} and multiparametric programming \cite{sakizlis2007linear,baotic2005optimal}, face computational challenges, particularly with increasing time steps exacerbating the number of integer decision variables. 

To handle this challenge, we take advantage of \textit{mode sequences}; it is well-established that a known mode sequence greatly simplifies computation \cite{thomas2006robust,kerrigan2002optimal,rakovic2005robust}.
By leveraging sampling to synthesize likely mode sequences, our proposed approach significantly improves computational efficiency to enable successful computation of reach-avoid motion plans.

\subsection{Operations on H-Polytopes}\label{app:h_poly_ops}
The intersection of H-polytopes is an H-polytope:
\begin{align}
\begin{split}\label{eq:intersection}
    \hpoly(\Acon_1, \bcon_1) \cap \hpoly(\Acon_2, \bcon_2) &= \left\{\xv\ |\ \Acon_1\xv\leq\bcon_1, \Acon_2\xv\leq\bcon_2\right\},\\
    &= \hpoly\left(\begin{bmatrix}
        \Acon_1 \\
        \Acon_2
    \end{bmatrix}, \begin{bmatrix}
        \bcon_1 \\
        \bcon_2
    \end{bmatrix}\right).
\end{split}
\end{align}

The Minkowski sum $\oplus$ is
\begin{align}\label{eq:minkowski_sum}
    \Poly_1 \oplus \Poly_2 &= \left\{\xv + \yv\ |\ \xv \in \Poly_1, \yv \in \Poly_2\right\},
\end{align}
which can be computed as one H-polytope by solving an LP \cite{herceg2013multi}.
The Minkowski sum essentially adds a ``buffer'' to the object being summed, which we use to pad the obstacles to account for tracking error in our reachability analysis.

Similarly, the Pontryagin Difference $\ominus$ is
\begin{align}\label{eq:pontryagin_diff}
    \Poly_1 \ominus \Poly_2 &= \left\{\xv \in \Poly_1\ |\ \xv + \yv \in \Poly_1\ \forall\ \yv \in \Poly_2\right\},
\end{align}
which can be expressed as one H-polytope by subtracting the support of $\Poly_2$ for each inequality of $\Poly_1$ \cite{herceg2013multi}.
This essentially inverts the Minkowski sum, reducing the volume of an H-polytope, which we use to shrink a goal set to account for tracking error.

The Cartesian product $\times$ is exactly an H-polytope:
\begin{subequations}
\begin{align}
    \Poly_1 \times \Poly_2 &= \left\{\begin{bmatrix}
        \xv\\\yv
    \end{bmatrix}\ \Bigr|\ \xv \in \hpoly(\Acon_1, \bcon_1), \yv \in \hpoly(\Acon_2, \bcon_2)\right\},\\
    &= \hpoly\left(\begin{bmatrix}
        \Acon_1 & \zeros\\
        \zeros & \Acon_2
    \end{bmatrix}, \begin{bmatrix}
        \bcon_1\\\bcon_2
    \end{bmatrix}\right).
\end{align}
\end{subequations}
We use this to combine low-dimensional reachable sets into higher dimensions, and to ensure that H-polytopes have equal dimensions for operations such as Minkowski sums and Pontryagin differences.

The convex hull is
\begin{align}\label{eq:conv_hull}
    \conv{\Poly_1, \Poly_2} &= \left\{\xv + \gams (\yv - \xv)\ |\ 0 \leq \gams \leq 1, \xv, \yv \in \Poly_1 \cup \Poly_2\right\},
\end{align}
which can be computed as one H-polytope by vertex enumeration, computing the convex hull of the vertices, and expressing the result in H-polytope form \cite{fukuda2003cddlib}.
The convex hull contains all straight lines between two sets, which we use to overapproximate trajectories between discrete time points to secure continuous time safety guarantees.

For an $\ndim$-dimensional H-polytope $\hpoly(\Acon, \bcon)$, the projection from the $\pdim^\regtext{th}$ dimension to the $\qdim^\regtext{th}$ dimension for some $\pdim \leq \qdim \leq \ndim$ is defined as:
\begin{align}
\begin{split}
\label{eq:projection}
    \proj[\pdim:\qdim]{\hpoly(\Acon, \bcon)} &= \left\{
        \begin{bmatrix}
        \zeros_{a},\
        \eye_{b}, \ 
        \zeros_{c}
    \end{bmatrix}\xv\ |\
    \Acon\xv\leq\bcon\right\},\quad\regtext{where} \\
    a &= (\qdim-\pdim+1)\times(\pdim-1), \\
    b &= (\qdim-\pdim+1)\ \regtext{and} \\
    c &= (\qdim-\pdim+1) \times (\ndim - \qdim)
\end{split}
\end{align}
The right-hand side of \eqref{eq:projection} is exactly an \textit{AH-polytope}, which can be converted to an H-polytope using block elimination, Fourier-Motzkin elimination, parametric linear programming, or vertex enumeration \cite{fukuda2003cddlib, marechal2017scalable}.
Somewhat the opposite of Cartesian product, the projection operation lowers the dimensions of the input polytope, which will be useful for retrieving low-dimensional information if only a subset of the original states is of interest.

Note that convex hull and projection are NP-hard for H-polytopes, exponentially scaling in computational time and complexity with the dimension of the polytope \cite{althoff2021set}.
However, since PARC performs polytope computation only for low-dimensional ($\leq$ 6-D) planning models, the computational time and complexity remain reasonable, as shown in Section \ref{sec:experiments}.

Finally, slicing \cite{herceg2013multi} an $\ndim$-dimensional H-polytope $\hpoly(\Acon, \bcon)$ from the $\pdim^\regtext{th}$ dimension to the $\qdim^\regtext{th}$ dimension with respect to a constant $\xv_0 \in \R^{\qdim-\pdim+1}$, where $\pdim \leq \qdim \leq \ndim$, is defined as:
\begin{subequations}
\begin{align}
    \slice[\pdim:\qdim]{\hpoly(\Acon, \bcon), \xv_0} &= \left\{ \begin{bmatrix}
        \yv_1 \\ \yv_2
    \end{bmatrix}\ |\ \Acon\begin{bmatrix}
        \yv_1 \\ \xv_0 \\ \yv_2
    \end{bmatrix} \leq \bcon
    \right\},\\
    &= \left\{\yv\ |\ \begin{bmatrix}
        \Acon_1,\
        \Acon_3
    \end{bmatrix}\yv \leq \bcon - \Acon_2 \xv_0
    \right\},\\
    &= \hpoly(\begin{bmatrix}
        \Acon_1,\
        \Acon_3
    \end{bmatrix}, \bcon - \Acon_2 \xv_0), 
\end{align}
\end{subequations}
where $\Acon = [\Acon_1, \Acon_2, \Acon_3]$, $\Acon_1 \in \R^{\nhp\times\ndim_{\pdim-1}}$, $\Acon_2 \in \R^{\nhp\times\ndim_{\qdim-\pdim+1}}$, $\Acon_3 \in \R^{\nhp\times\ndim_{\ndim-\qdim}}$.
Thus, slicing results in exactly an H-polytope.
We use slicing to aid in some visualization and mathematical proofs.
\subsection{Method for Affinizing a Nonlinear System}\label{app:voronoi_affinization}

We convert a nonlinear system into a PWA system via sampling in its state space, then computing Voronoi regions and linearizing the dynamics in each region, similar to in \cite{casselman2009new}.
In particular, we apply the following:

\begin{prop}\label{prop:voronoi affinization}
Express the domain $\Xpoly$ as a H-polytope $\hpoly(\Acon_X, \bcon_X)$.
Consider linearization points $\xv\lin_{\ts, \is} \in \Xpoly$, $\is = 1, \cdots, \ndim\lin_\ts$.
Then for $\is = 1, \cdots, \ndim\lin_\ts$, the Voronoi cells $\Voro_{\is,\ts}$ given by
\begin{subequations}
\begin{align}\label{eq:cell_def}
    \Voro_{\is,\ts} &= \Voror_{\is,\ts} \cap \Xpoly,\\
    \Voror_{\is,\ts} &= \left\{\xv\ |\ \norm{\xv - \xv\lin_{\ts, \is}}^2_2 \leq \norm{\xv - \xv\lin_{\ts, \js}}^2_2, \js = 1, \cdots, \ndim\lin_\ts\right\},
\end{align}
\end{subequations}
are compact H-polytopes that fulfill \eqref{eq:bound_con}, and can therefore be the PWA regions of a PWA system.
\end{prop}
\begin{proof}
It was proven in \cite{boyd2004convex} that $\Voror_{\is,\ts}$ is a closed, but not necessarily bounded H-polytope.
In fact,
\begin{subequations}
\begin{align}
    \Voror_{\is,\ts} =& \Biggl\{\xv\ \Biggr|\ \begin{bmatrix}
        (\xv - \xv\lin_{\ts, \is})\cdot(\xv - \xv\lin_{\ts, \is})\\
        \vdots\\
        (\xv - \xv\lin_{\ts, \is})\cdot(\xv - \xv\lin_{\ts, \is})
    \end{bmatrix} \leq \notag\\
    &\begin{bmatrix}
        (\xv - \xv\lin_{\ts, 1})\cdot(\xv - \xv\lin_{\ts, 1})\\
        \vdots\\
        (\xv - \xv\lin_{\ts, \ndim\lin(\ts)})\cdot(\xv - \xv\lin_{\ts, \ndim\lin(\ts)})
    \end{bmatrix}\Biggr\},\\
    =& \Biggl\{\xv\ \Biggr|\ \begin{bmatrix}
        (2\xv\lin_{\ts, 1} - 2\xv\lin_{\ts, \is})\tp\\
        \vdots\\
        (2\xv\lin_{\ts, \ndim\lin(\ts)} - 2\xv\lin_{\ts, \is})\tp
    \end{bmatrix}\xv \leq \notag\\
    &\begin{bmatrix}
        \norm{\xv\lin_{\ts, 1}}_2^2 - \norm{\xv\lin_{\ts, \is}}_2^2\\
        \vdots\\
        \norm{\xv\lin_{\ts, \ndim\lin(\ts)}}_2^2 - \norm{\xv\lin_{\ts, \is}}_2^2
    \end{bmatrix}\Biggr\},\\
    =& \hpoly\left(\begin{bmatrix}
        (2\xv\lin_{\ts, 1} - 2\xv\lin_{\ts, \is})\tp\\
        \vdots\\
        (2\xv\lin_{\ts, \ndim\lin(\ts)} - 2\xv\lin_{\ts, \is})\tp
    \end{bmatrix}, \begin{bmatrix}
        \norm{\xv\lin_{\ts, 1}}_2^2 - \norm{\xv\lin_{\ts, \is}}_2^2\\
        \vdots\\
        \norm{\xv\lin_{\ts, \ndim\lin(\ts)}}_2^2 - \norm{\xv\lin_{\ts, \is}}_2^2
    \end{bmatrix}\right).
\end{align}
\end{subequations}
Since H-polytopes are closed under intersection (as shown in \eqref{eq:intersection}), and the intersection of a closed set and a compact set is compact, $\Voro_{\is,\ts}$ is a compact H-polytope from \eqref{eq:cell_def}.

Further, for any $\is, \js \in \{1, \cdots, \ndim\lin_\ts\}$, $\is \neq \js$, the intersection of the interior of $\Voro_{\is,\ts}$ with $\Voro_{\js,\ts}$ is given by:
\begin{align}
    &\interior{\Voro_{\is,\ts}}\cap\Voro_{\js,\ts} = \Bigl\{ \xv\ |\ \norm{\xv - \xv\lin_{\ts, \is}}^2_2 < \norm{\xv - \xv\lin_{\ts, \ks}}^2_2, \notag\\
    &\norm{\xv - \xv\lin_{\ts, \js}}^2_2 \leq \norm{\xv - \xv\lin_{\ts, \ks}}^2_2, \ks = 1, \cdots, \ndim\lin_\ts\Bigr\} \cap \interior{\Xpoly},
\end{align}
where $\interior{\cdot}$ is the interior of a set.
However, notice that $\norm{\xv - \xv\lin_{\ts, \is}}^2_2 < \norm{\xv - \xv\lin_{\ts, \ks}}^2_2$ and $\norm{\xv - \xv\lin_{\ts, \js}}^2_2 \leq \norm{\xv - \xv\lin_{\ts, \ks}}^2_2$ for $\ks = 1, \cdots, \ndim\lin(\ts)$ implies $\norm{\xv - \xv\lin_{\ts, \is}}^2_2 < \norm{\xv - \xv\lin_{\ts, \js}}^2_2$ and $\norm{\xv - \xv\lin_{\ts, \js}}^2_2 \leq \norm{\xv - \xv\lin_{\ts, \is}}^2_2$, which is impossible.
Thus,
\begin{align}
    \interior{\Voro_{\is,\ts}}\cap\Voro_{\js,\ts} = \emptyset,
\end{align}
which is exactly the condition \eqref{eq:bound_con}.

Therefore, $\Voro_{\is,\ts}$ defines the PWA regions of a PWA system, with each of the $\npwat = \ndim\lin_\ts$ regions given by:
\begin{align}
    \Acon_{\is,\ts} &= \begin{bmatrix}
        (2\xv\lin_{\ts, 1} - 2\xv\lin_{\ts, \is})\tp\\
        \vdots\\
        (2\xv\lin_{\ts, \ndim\lin(\ts)} - 2\xv\lin_{\ts, \is})\tp\\
        \Acon_X
    \end{bmatrix},\\ 
    \bcon_{\is,\ts} &= \begin{bmatrix}
        \norm{\xv\lin_{\ts, 1}}_2^2 - \norm{\xv\lin_{\ts, \is}}_2^2\\
        \vdots\\
        \norm{\xv\lin_{\ts, \ndim\lin(\ts)}}_2^2 - \norm{\xv\lin_{\ts, \is}}_2^2\\
        \bcon_X
    \end{bmatrix},
\end{align}
for $\is = 1, \cdots, \ndim\lin_\ts$.
\end{proof}

\subsection{Proof of Theorem \ref{thm:avoid_set}: Avoid Set without Tracking Error} \label{app:extended_proof}
Here, we prove Theorem \ref{thm:avoid_set}, which is our paper's main result.
We first introduce the notations used throughout the proof, then provide a concise proof sketch, before finally elaborating the detailed proof.
\begin{proof} \textbf{\textit{Notation.}} For clarity, we introduce the block matrix representation of $\Ccon_{s_t, t} - \eye_{\nlow}$ and $\dcon_{s_t, t}$:

\begin{equation}
\Ccon_{\sap_\ts, \ts} - \eye_{\nlow} \triangleq \left[
\begin{array}{c|c} \Ccon_1 & \Ccon_2 \\ \hline \Ccon_3 & \Ccon_4 \end{array}
\right], \quad \dcon_{\sap_\ts, \ts} \triangleq \left[
    \begin{array}{c} \dcon_1 \\ \hline \dcon_2 \end{array}
\right]
\end{equation}
\noindent where 
\begin{align*}
\Ccon_1 &\in \R^{{\nlow}\eti \times {\nlow}\eti}, & \Ccon_2 &\in \R^{{\nlow}\eti \times {\nlow}\noneti},\\
\Ccon_3 &\in \R^{{\nlow}\noneti \times {\nlow}\eti}, & \Ccon_4 &\in \R^{{\nlow}\noneti \times {\nlow}\noneti}, \\ 
\dcon_1 &\in \R^{{\nlow}\eti}, & \dcon_2 &\in \R^{{\nlow}\noneti}.
\end{align*}

\noindent We denote the projection of obstacle $\Obs_i$ and $(\tf-\ts)$-time BRS $\ol\Rbrs_{\ts}$ onto the subspace $\Xpoly\eti$ by $\Obs\eti \triangleq \proj[1:{\nlow}\eti]{\Obs_i}$ and $\ol\Rbrs\eti \triangleq \proj[1:{\nlow}\eti]{\ol\Rbrs_{\ts}}$, respectively. Similarly, the projection of $\ol\Rbrs_{\ts}$ onto the subspace $\Xpoly\noneti$ is denoted as $\ol\Rbrs\noneti$.
We also denote $\Obs_{+} \triangleq \Obs\eti \times \R^{{\nlow}\noneti}$ for brevity.

\textbf{\textit{Proof sketch.}} We reinterpret the theorem as a set-inclusion problem between a \textit{true avoid set} $P$ and \textit{over-approximated avoid set} $\ol\Ravd_{\is, t, t}$.  
Then, a conservative avoid set $Q$ is defined, from which we derive the sufficient condition for set-inclusion. 
Using the Extended-Translation-Invariance (ETI) assumption (Assumption \ref{ass:eti}) and Lemma \ref{lem:pwa_eti}, we prove the satisfaction of the sufficient condition.

\textbf{\textit{Detailed Proof.}} We first define the true avoid set $P$, a set of states that satisfy the equation \eqref{eq:int_avoid_set} as follows:
\begin{equation}
\begin{split} P = \{
\xv \in \ol\Rbrs_{\ts} & \mid ~\exists (\hat \xv, \gams) \in \Obs_i \times [0, 1] \quad \\
&\regtext{s.t.} \quad \xv + \gams(\Ccon_{\sap_{\ts}, \ts}\xv + \dcon_{\sap_{\ts}, \ts} - \xv)=\hat \xv 
\} \end{split}
\end{equation}
\noindent Note that the theorem holds if and only if $P \subset \ol\Ravd_{\is, \ts, \ts}$. 

Next, to derive the sufficient condition for $P \subset \ol\Ravd_{\is, \ts, \ts}$, we define a \textit{conservative} avoid set $Q$:
\begin{equation}
\begin{split} Q = &\{\xv\eti 
\in \ol\Rbrs\eti \mid ~\exists (\xv\noneti, \hat\xv\eti, \gamma) \in \ol\Rbrs\noneti \times \Obs\eti \\ 
&\times [0,1] \ \regtext{s.t.}\  \xv\eti + \gams(\Ccon_{1}\xv\eti + \Ccon_{2}\xv\noneti + \dcon_{1})=\hat \xv\eti 
\} \end{split}
\end{equation}
\noindent Note that $\xv=[\xv\eti\tp, \xv\noneti\tp]\tp \in P$ implies $\xv\eti \in Q$, $\xv\noneti \in \Xpoly\noneti$, $\xv \in \ol\Rbrs_{\ts}$, which leads to:

\begin{equation} \label{eq:z-agnostic}
P \subset (Q \times \Xpoly\noneti) \cap \ol\Rbrs_t
\end{equation}
\noindent Building on this relation, we derive the sufficient condition for the theorem as follows:
\begin{equation}\label{eq:avoid_suff_condition}
Q \subset \conv{\proj[1:{\nlow}\eti]{A}, \Obs\eti}
\end{equation}
\noindent since satisfaction of \eqref{eq:avoid_suff_condition} and \eqref{eq:z-agnostic} implies $P \subset \ol\Ravd_{\is, \ts, \ts}$.

Now we prove the sufficient condition \eqref{eq:avoid_suff_condition} using the ETI assumption.
Fix $\tilde{\xv}\eti \in Q$.
Then by definition of $Q$, there exists $(\tilde{\xv}\noneti, \hat{\xv}\eti, \tilde\gams) \in \ol\Rbrs\noneti \times \Obs\eti \times [0, 1]$ that satisfies:
\begin{equation}
    \tilde{\xv}\eti + \tilde\gams(\Ccon_{1} \tilde{\xv}\eti + \Ccon_{2}\tilde\xv\noneti + \dcon_{1})=\hat\xv\eti
\end{equation}
By rearranging the equation, we get:
\begin{equation} \label{eq:avoid_convhull}
\begin{split}
    \tilde\xv\eti &= \tilde\gams \left( \hat\xv\eti - (\Ccon_1\tilde\xv\eti+\Ccon_2\tilde\xv\noneti+\dcon_1)
    \right)+ (1-\tilde\gams)\hat\xv\eti \\
    &= \tilde\gams\vc{y} + (1-\tilde\gams)\hat\xv\eti
\end{split}
\end{equation}
\noindent where noting $\hat\xv\eti \in \Obs\eti$, it suffices to prove $\vc{y} \in \proj[1:{\nlow}\eti]{A}=\proj[1:{\nlow}\eti]{\brs{\Obs_+, \Ccon_{s_t, t}, \dcon_{s_t, t}}}$ to prove \eqref{eq:avoid_suff_condition}. This is shown by proving $[\vc{y}\tp, \tilde\xv\noneti\tp]\tp \in \brs{\Obs_+, \Ccon_{s_t, t}, \dcon_{s_t, t}}$. 
Noting $\Obs_+ = \Obs\eti \times \R^{{\nlow}\noneti}$, we have to show
\begin{align}
\vc{y} + \Ccon_{1} \vc{y} + \Ccon_{2} \tilde\xv\noneti + \dcon_{1} &\in \Obs\eti \label{eq:eti_avoid}\\ 
\tilde\xv\noneti + \Ccon_{3} \vc{y} + \Ccon_{4} \tilde\xv\noneti + \dcon_{2} &\in \R^{{\nlow}\noneti}.\label{eq:other_avoid}
\end{align}
Equation \eqref{eq:other_avoid} is immediate, while \eqref{eq:eti_avoid} is true by 
\begin{subequations}
\begin{align}
&\vc{y} + \Ccon_{1} \vc{y} + \Ccon_{2} \tilde\xv\noneti + \dcon_{1}\notag\\
=&\hat\xv\eti - (\Ccon_{1} \hat\xv\eti + \Ccon_{2} \tilde\xv\noneti + \dcon_{1}) \notag\\
    &+ \Ccon_{1}(\hat\xv\eti - (\Ccon_{1} \hat\xv\eti + \Ccon_{2} \tilde\xv\noneti + \dcon_{1}))+ \Ccon_2\tilde\xv\noneti+\dcon_1 \\
=& \hat\xv\eti - \Ccon_{1}(\Ccon_{1}\hat\xv\eti+\Ccon_{2}\tilde\xv\noneti+\dcon_{1})\label{eq:before_eti} \\ 
=& \hat\xv\eti \in \Obs\eti \label{eq:after_eti}
\end{align}
\end{subequations}
\noindent The equation \eqref{eq:after_eti} follows from \eqref{eq:before_eti} using ETI Assumption \ref{ass:eti} and the result of Lemma \ref{lem:pwa_eti} (i.e. \eqref{eq:eti_pwa_reorder}) namely:
\begin{equation}
\begin{split}
    &\norm{\Ccon_1(\Ccon_1 \hat\xv\eti + \Ccon_2 \tilde\xv\noneti + \dcon_1)}_2 \\
    =& 
    \norm{\hat{\Ccon}_{s_t, t}^{\Xpoly\eti}{(\hat{\Ccon}_{s_t, t}\xv + \dcon_{s_t, t})}_{1:{\nlow}\eti}}_2 \\  
    =& 0 
\end{split}
\end{equation}
Since $\vc{y} \in \proj[1:{\nlow}\eti]{A}$ and $\hat\xv\eti \in \Obs\eti$, by \eqref{eq:avoid_convhull}, $\tilde\xv\eti \in \conv{\proj[1:{\nlow}\eti]{A}, \Obs\eti}$.
Hence the sufficient condition \eqref{eq:avoid_suff_condition} is proved, and the theorem is proved accordingly.
\end{proof}

\subsection{Alternate Methods for Computing the Avoid Set}

Here, we present two methods to compute the avoid set without tracking error as alternatives to Corollary \ref{cor:avoid_set_computation} and Theorem \ref{thm:avoid_set}.
To account for tracking error, the goal set and the obstacles can be augmented or subtracted with the corresponding error sets as in Section \ref{subsec:parc_avoid_real}.

\subsubsection{Grid-Based Avoid Set}\label{subsec:grid_avoid_set}
Recall that, per \eqref{eq:cont_time_approx}, we model the continuous-time behavior as straight line segments between the discrete states of the PWA system.
Intuitively, the avoid set should include all trajectories where the straight lines connecting the PWA planning model's discrete states collide with any obstacles.
Thus, we can establish a necessary condition for collision to occur between two discrete timesteps with the following Proposition:

\begin{prop}[Unsafe Plans w/o Tracking Error]
\label{prop:obs_filter}
Consider the obstacle $\Obs_\is$, some time $\ts \in \{0, \Delta\ts, \cdots, \tf - \Delta\ts\}$, and the mode sequence $\Sap = (\sap_0, \cdots, \sap_\ts, \sap_{\ts+\Delta\ts}, \cdots, \sap_{\tf-\Delta\ts})$.
Suppose we compute  the $(\tf-\ts)$-time BRS $\ol{\Rbrs}_\ts$ and the $(\tf-\ts-\Delta\ts)$-time BRS $\ol{\Rbrs}_{\ts+\Delta\ts}$ as in \eqref{eq:brs_no_error}.
If any trajectories between $\ol{\Rbrs}_\ts$ and $\ol{\Rbrs}_{\ts+\Delta\ts}$ collide with $\Obs_\is$, that is, if $\exists\ \xv(\ts) \in \ol{\Rbrs}_\ts, \xv(\ts+\Delta\ts) = \Ccon_{\sap_\ts, \ts}\xv(\ts) + \dcon_{\sap_\ts, \ts} \in \ol{\Rbrs}_{\ts+\Delta\ts}$ such that
\begin{align}\label{eq:obs_filter_if}
    \Obs_\is \cap \left\{\xv(\ts) + \gams(\xv(\ts+\Delta\ts) - \xv(\ts))\ |\ 0\leq\gams\leq1\right\}\neq\emptyset,
\end{align}
then
\begin{align}\label{eq:obs_filter}
    \Obs_\is \cap \conv{\ol{\Rbrs}_\ts, \ol{\Rbrs}_{\ts+\Delta\ts}}\neq\emptyset.
\end{align}
\end{prop}
\begin{proof}
By the definition of convex-hull, we have
\begin{subequations}
\begin{align}
    \conv{\ol{\Rbrs}_\ts, \ol{\Rbrs}_{\ts+\Delta\ts}} = &\{\xv(\ts) + \gams(\xv(\ts + \Delta\ts) - \xv(\ts))\ | \ 0\leq\gams\leq1, \notag\\
    &\xv(\ts), \xv(\ts+\Delta\ts)\in\ol{\Rbrs}_\ts \cup \ol{\Rbrs}_{\ts+\Delta\ts}\},\\
    \supset &\{\xv(\ts) + \gams(\xv(\ts + \Delta\ts) - \xv(\ts))\ |\ 0\leq\gams\leq1, \notag\\
    &\xv(\ts) \in \ol{\Rbrs}_\ts, \xv(\ts+\Delta\ts)\in \ol{\Rbrs}_{\ts+\Delta\ts}\}.
\end{align}
\end{subequations}
Thus \eqref{eq:obs_filter_if} necessarily implies \eqref{eq:obs_filter}.
\end{proof}

One can exploit Proposition \ref{prop:obs_filter} to compute the avoid set by gridding up the goal set and checking for each grid if any of their intermediate BRSs satisfies \eqref{eq:obs_filter} with any of the obstacles.
The union of all grids that fulfill such condition can then be taken as an over-approximation of the true avoid set without tracking error.
Alternatively, one can also grid up the reach set and check if their FRS \cite{herceg2013multi} intersects with any of the obstacles instead.

Since this method does not require ETI assumptions on the PWA planning model, one can show that its approximation of the true avoid set gets tighter as the grid gets smaller.
In practice, we found that computing avoid set in this manner is significantly slower than Theorem \ref{thm:avoid_set} for a desirable grid size.
Thus, we have only used it in Fig. \ref{fig:avoid_approx} to show the extend of overapproximation of Theorem \ref{thm:avoid_set}.

\subsubsection{Avoid Set w/o Convex Hull or Projection}\label{subsec:avoid_set_nohull}

In PWA planning models with specific structures, the intermediate avoid sets without tracking error $\ol{\Ravd}_{\is, \ts, \ts}$, and hence the avoid set $\Ravd$ can sometimes be computed without using convex hull or projection, significantly speeding up and reducing the conservativeness of the approximation:

\begin{thm}[Intermediate Avoid Set w/o Convex Hull or Projection]\label{thm:avoid_set_nohull}
    Consider a planning model where the augmented planning state $\xv(\ts)$ is:
    \begin{align}
        \xv(\ts) = \begin{bmatrix}
            \wv(\ts)\\
            \paramv
        \end{bmatrix},
    \end{align}
    where $\wv(\ts) \in \Workspace \subset \R$ and $\paramv = [\ks_1, \cdots, \ks_{\nparam}] \in \K \subset \R^{\nparam}$.
    If 
    \begin{align}\label{eq:slicedTraj}
        \slicedInt{\Obs_\is, A, \paramv\lin} =& \slice[2:(\nparam+1)]{\Obs_\is, \paramv\lin}\cup\notag\\
        &\slice[2:(\nparam+1)]{A, \paramv\lin}
    \end{align}
    is connected for all $\paramv\lin \in \{\ks_{1, \mini}, \ks_{1, \maxi}\}\times\cdots\times\{\ks_{\nparam, \mini}, \ks_{\nparam, \maxi}\}$, where
    \begin{align}
        \ks_{\is, \mini} &= \min_\paramv \{ \paramv_{\is}\ |\ \paramv\in\K\},\\
        \ks_{\is, \maxi} &= \max_\paramv \{ \paramv_{\is}\ |\ \paramv\in\K\},
    \end{align}
    for $\is = 1, \cdots, \nparam$, where $\paramv_{\is}$ is the $\is^\regtext{th}$ element of $\paramv$, then
    \begin{align}
        \ol{\Ravd}_{\is, \ts, \ts} = (\Obs_\is \cup A)\cap\ol{\Rbrs}_\ts,
    \end{align}
    can replace \eqref{eq:int_avoid_set} to fulfill the condition in \eqref{eq:int_avoid_con}.
\end{thm}
\begin{proof}
    We follow the proof strategy outlined in Appendix \ref{app:extended_proof}.
    Since all states in $\xv(\ts)$ are ETI, $\wv(\ts)$ is 1-D, and all $\paramv\in\K$ are constant, Theorem \ref{thm:avoid_set_nohull} is proven if $\yv \in A$ and $\hat\xv = \hat\xv\eti \in \Obs_\is$ implies $\tilde\xv = \tilde\xv\eti \in \Obs_\is \cup A$ by \eqref{eq:avoid_convhull}.
    This is true if $\slicedInt{\Obs_\is, A, \paramv}$ are connected for all $\paramv\in\K$.
    But if \eqref{eq:slicedTraj} is connected for all $\paramv\lin$, by convexity of $\Obs_\is$ and $A$, $\slicedInt{\Obs_\is, A, \paramv}$ must be connected for all $\paramv\in\K$, since $[\ks_{1, \mini}, \cdots, \ks_{\nparam, \mini}]\tp \leq \paramv \leq [\ks_{1, \maxi}, \cdots, \ks_{\nparam, \maxi}]\tp\ \forall \paramv\in\K$.
\end{proof}

In this paper, the PWA planning models in Section \ref{exp:exact_plan_reach_impact}, Section \ref{exp:time-varying}, and Appendix \ref{app:additional_exp} fulfill this condition at all timesteps.

\subsection{Experiment Details}\label{app:experiment}

\subsubsection{Near-Hover Quadrotor in 3-D}
For the planning model, we employ the 3D single integrator, described as:
\begin{align}
    \begin{split}
        \dot\px &= \hat{\vx} \\
        \dot\py &= \hat{\vy} \\
        \dot\pz &= \hat{\vz} \\
    \end{split}
\end{align}
where the planning states $\planv=[\px, \py, \pz]\tp$ represents the position of the robot, and trajectory parameter $\paramv=[\hat{\vx}, \hat{\vy}, \hat{\vz}]\tp$ denotes the desired velocity of the robot. The trajectory parameter domain is chosen to be $\K = [-0.5,0.5]^3$.

For the tracking model, we adopt a simplified 10-D quadrotor assuming near-hover conditions (i.e. small pitch and roll) \cite{chen2021fastrack}, expressed by:
\begin{align}
    \begin{split}
        \dot\px &= \vx \\
        \dot\py &= \vy \\
        \dot\pz &= \vz \\
        \dot\vx &= g \tan(\pthx) \\
        \dot\vy &= g \tan(\pthy) \\
        \dot\vz &= \nhkt \nhaz - g \\
        \dot\pthx &= - \nhdd \pthx + \vthx\\
        \dot\pthy &= - \nhdd \pthy + \vthy\\
        \dot\vthx &= -\nhd \pthx + \nhn \nhax \\
        \dot\vthy &= -\nhd \pthy + \nhn \nhay
    \end{split}
\end{align}
where the states $(\px, \py, \pz)$ denote the position, $(\vx, \vy, \vz)$ are velocities, $(\pthx, \pthy)$ are pitch and roll, and $(\vthx, \vthy)$ are pitch and roll rates. 
The control input $\uv = (\nhax, \nhay, \nhaz)$ is the desired pitch and roll $(\nhax, \nhay)$ and the vertical thrust $\nhaz$.
The model parameters $(\nhd, \nhdd, \nhn, \nhkt)$ and control input bound $\U$ are referenced from FaSTrack \cite{chen2021fastrack}.
A simple LQR linearized around the tracking error $\es=0$ is employed for the feedback controller.

\subsubsection{General Quadrotor in 3-D}\label{subsec:3dquad_model}

We consider the time-switched polynomial planning model defined in \cite{kousik2019safe, mueller2015computationally}, where, for $0\leq\ts<\tpk$,
\begin{subequations}\label{eq:polyquad_accel}
\begin{align}
    &\dot\px = \frac{\cf_1(\kvx, \kax, \kpkx)}{6}\ts^3 + \frac{\cf_2(\kvx, \kax, \kpkx)}{2}\ts^2 + \kax\ts + \kvx,\label{eq:accel_x}\\
    &\dot\py = \frac{\cf_1(\kvy, \kay, \kpky)}{6}\ts^3 + \frac{\cf_2(\kvy, \kay, \kpky)}{2}\ts^2 + \kay\ts + \kvy,\label{eq:accel_y}\\
    &\dot\pz = \frac{\cf_1(\kvz, \kaz, \kpkz)}{6}\ts^3 + \frac{\cf_2(\kvz, \kaz, \kpkz)}{2}\ts^2 + \kaz\ts + \kvz,\label{eq:accel_z}\\
    &\cf_1(\kv, \ka, \kpk) = \frac{12}{\tpk^3}\kv + \frac{6}{\tpk^2}\ka - \frac{12}{\tpk^3}\kpk,\\
    &\cf_2(\kv, \ka, \kpk) = - \frac{6}{\tpk^2}\kv - \frac{4}{\tpk}\ka + \frac{6}{\tpk^2}\kpk,
\end{align}
\end{subequations}
and for $\tpk\leq\ts\leq\tf$,
\begin{subequations}\label{eq:polyquad_decel}
\begin{align}
    \dot\px &= \frac{\cf_3(\kpkx)}{6}(\ts-\tpk)^3 + \frac{\cf_4(\kpkx)}{2}(\ts-\tpk)^2 + \kpkx,\label{eq:decel_x}\\
    \dot\py &= \frac{\cf_3(\kpky)}{6}(\ts-\tpk)^3 + \frac{\cf_4(\kpky)}{2}(\ts-\tpk)^2 + \kpky,\label{eq:decel_y}\\
    \dot\pz &= \frac{\cf_3(\kpkz)}{6}(\ts-\tpk)^3 + \frac{\cf_4(\kpkz)}{2}(\ts-\tpk)^2 + \kpkz,\label{eq:decel_z}\\
    \cf_3(\kpk) &= \frac{12}{(\tf - \tpk)^3}\kpk,\\
    \cf_4(\kpk) &= \frac{-6}{(\tf - \tpk)^2}\kpk,
\end{align}
\end{subequations}
where $\tf = 3$, $\tpk = 1$, the planning states and workspace $\planv = \wv = [\px, \py, \pz]\tp$ that represents the position of the quadrotor, and trajectory parameters $\paramv = [\kvx, \kax, \kpkx, \kvy, \kay, \kpky, \kvz, \kaz, \kpkz]\tp$, where $\kvx, \kvy, \kvz$ represents the desired initial speed, $\kax, \kay, \kaz$ represents the desired initial acceleration, and $\kpkx, \kpky, \kpkz$ represents the desired speed at $\ts = \tpk$.

This planning model can be separated into three low-dimensional models for each of the three workspace dimensions.
That is, we can consider the planning states and workspace $\planv_x = \wv_x = \px$ and trajectory parameters $\paramv_x = [\kvx, \kax, \kpkx]$, defined by the dynamics \eqref{eq:accel_x} and \eqref{eq:decel_x}, compute the reach set $\Rbrs_{0, x}\subset\R^4$ and avoid set ${\Ravd}_{\is, \ts, 0, x} \subset \R^4$ for each $\ts = 0, \Delta\ts, \cdots, \tf-\Delta\ts$ and each $\is = 1,\cdots,\nobs$ using PARC, and repeat for $\py$ and $\pz$.

We recover the 12-D reach set $\Rbrs_0$ and avoid set $\Ravd$ by:
\begin{align}
    \Rbrs_0 &= \Rbrs_{0, x} \times \Rbrs_{0, y} \times \Rbrs_{0, z},\label{eq:decouple_reach}\\
    \Ravd &= \union_{\is=1}^{\nobs} \union_{\ts=0}^{\tf-\Delta\ts} {\Ravd}_{\is, \ts, 0, x}\times{\Ravd}_{\is, \ts, 0, y}\times{\Ravd}_{\is, \ts, 0, z}.\label{eq:decouple_avoid}
\end{align}

Further, we note that for each of the low-dimensional model, the intermediate avoid sets without tracking error $\ol{\Ravd}_{\is, \ts, \ts, x}$, $\ol{\Ravd}_{\is, \ts, \ts, y}$, and $\ol{\Ravd}_{\is, \ts, \ts, z}$ and hence the avoid set $\Ravd$ can be computed without using convex hull or projection according to Theorem \ref{thm:avoid_set_nohull}.

The tracking model for the system is from \cite{kousik2019safe, lee2010control}, defined as:
\begin{subequations}
\begin{align}
    \begin{bmatrix}
        \dot\px\\
        \dot\py\\
        \dot\pz
    \end{bmatrix} &= \vv,\\
    \dot\vv &= \thrust \Rm \eunit - \mass\grav\eunit,\\
    \dot\angv &= \moi^{-1}\left(\begin{bmatrix}
        \mu_x\\
        \mu_y\\
        \mu_z
    \end{bmatrix}-\angv\times\moi\angv\right),\\
    \dot\Rm &= \Rm\ \hatmap(\angv),
\end{align}
\end{subequations}
where $\vv\in\R^3$ is velocity, $\Rm\in\regtext{SO}(3)$ is the attitude, $\eunit\in\R^3$ is the unit vector in the inertial frame that points ``up'' relative to the ground, $\mass = 0.547$, $\grav = 9.81$, $\angv\in\R^3$ is the angular velocity, $\moi = \regtext{diag}(0.0033, 0.0033, 0.0058)$, $\hatmap:\R^3 \to \regtext{SO}(3)$ is the hat map operator \cite{lee2010control}, and the control inputs $\uv = [\thrust, \mu_x, \mu_y, \mu_z]\tp$ are net thrust and body moments.

\subsubsection{Comparison Methods}
\paragraph{FaSTrack}
FaSTrack is a planner-tracker framework that precomputes a uniform TEB (maximum possible tracking error) using HJ reachability analysis. 
Obstacles are padded by the TEB; then, as long as the planner only plans outside the padded obstacles the robot is safe.
During the TEB computation, FaSTrack also synthesizes a controller to guarantee the bound.
We use RRT as the planning model as per \cite{chen2021fastrack}.

\paragraph{Neural CLBF}
Neural CLBF approximates a safety filter value function as a neural network, thereby segmenting the environment into goal, safe, and unsafe regions.
The value function, trained with data from the entire state space, enables the generation of safe, goal-directed control inputs through Quadratic Programming (QP).

\paragraph{Reachability-based Trajectory Design}
RTD is a planner-tracker framework that is very similar to PARC in that it uses a planner-tracker framework and compensates for tracking errors in a time-varying way.
However, RTD leverages forward reachable sets (FRS).
RTD numerically approximates the FRS either with sums-of-squares \cite{kousik2020bridging} or zonotope \cite{kousik2019safe,althoff2015introduction} reachability.
At runtime, RTD performs receding-horizon trajectory planning with nonlinear, nonconvex constraints derived from the FRS intersecting with obstacles to eliminate unsafe motion plans.
Each parameterized plan contains a failsafe maneuver such that, if replanning fails, the previous (safe) maneuver can be executed.

\paragraph{KDF}
KDF combines sampling-based motion planning with funnel-based feedback control \cite{verginis2022kdf}.
This method keeps the tracking error within funnels defined by a performance function which decreases exponentially over time \cite{bechlioulis2014low}.
These funnels help higher-level motion planners by incorporating maximum funnel values to expand the free space for planning.
The implementation details are described in Appendix \ref{subsec:kdf_param}.

\subsection{Additional Experiment Details} \label{app:exp_add_details}


\subsubsection{Impact of Planner-Tracker Cooperation}

\paragraph{Environment}
The environment is a $[0, 10] \times [-10, 10] \times [0, 10]$ cuboid (all lengths are in meters).
The goal $\ol\Xgoal$ is a cube centered at $(\px, \py, \pz)=(9, 0, 5)$ with a side of $1$ \unit{m}. 
Obstacles $\ol\Obs$ include the floor, ceiling, and walls, and two additional cuboid obstacles of $[6.5, 7.5] \times [-8.2, -\Delta w/2] \times [0, 10]$ and $[6.5, 7.5] \times [\Delta w/2, 8.2] \times [0, 10]$ with varying $\Delta w$ that denotes the gap between two obstacles.
These are shown in Fig. \ref{fig:nearhover-single-comparison} with a gap width $\Delta w = 1.8$.
The obstacles are padded to accommodate the quadrotor's body, approximated by a cuboid of $0.54 \times 0.54 \times 0.05$ using the specifications of Hummingbird \cite{dong2015development}.
The planning horizon is $\tf=10$ \unit{s} with time discretization $\Delta\ts=0.1$ \unit{s} each. 
The frequency of the tracking controller is $1,000$ \unit{Hz} for all methods.
\subsubsection{Impact of Tighter Planning Reachability}

\paragraph{Environment} \label{subsec:3dquad_env}
The general 3-D quadrotor environment is similar to that of the 10-D near-hover quadrotor environment but with a narrower gap to the goal.
The obstacles are defined as $\ol{\Obs}_1 = [3.23, 6.77]\times[0.23, 9.27]\times[0.73, 9.27]$, $\ol{\Obs}_2 = [3.23, 6.77]\times[-9.27, -0.23]\times[0.73, 9.27]$ (after accounting for an overapproximated volume of the quadrotor), that is, two walls with a clearance of 0.46, the goal region is defined as $\ol{\Xgoal} = [7.44, 9.56]\times[-1.06, 1.06]\times[3.94, 6.06]$, and the domains are defined as $\Xplan = [0, 10]\times[-10, 10]\times[0, 10]$ and $\K = [-5.25, 5.25]\times[-10, 10]\times[-5.25, 5.25]\times[-5.25, 5.25]\times[-10, 10]\times[-5.25, 5.25]\times[-5.25, 5.25]\times[-10, 10]\times[-5.25, 5.25]$.
Initial positions for simulations are sampled from a $(15, 15, 3)$ grid spanning $[0.1, 4.8] \times [-9.9, 9.9] \times [3, 7]$, all with zero initial velocity. 

\subsubsection{Impact of Proper Use of Time-varying Tracking Error}

\paragraph{Environment}\label{subsec:kdf_env}
In the 3-D quadrotor environment, similar to the one described in Section \ref{exp:exact_plan_reach_impact}, we evaluate two scenarios differentiated by obstacle gap widths. 
The obstacles are specified for both scenarios as $\overline{\Obs}_1 = [5.75, 7.25]\times[\triangle w/2, 10]\times[1, 9]$ and $\overline{\Obs}_2 = [5.75, 7.25]\times[-10, -\triangle w/2]\times[1, 9]$, using $\triangle w = 0.85$ for the narrow gap and $\triangle w = 3.0$ for the wide gap. 
The goal region is set as $\overline{\Xgoal} = [7, 9]\times[-1, 1]\times[4, 6]$, with planning domains defined by $\Xplan = [0, 10]\times[-10, 10]\times[0, 10]$. 
The trajectory parameter space, $\K = [-1.25, 1.25]\times[-1.5, 1.5]\times[-2, 2]\times[-1.25, 1.25]\times[-1.5, 1.5]\times[-2, 2]\times[-1.25, 1.25]\times[-1.5, 1.5]\times[-2, 2]$, was selected to simplify the tuning of a funnel-based feedback controller. 
Initial positions for simulations are sampled from a $(25, 25, 3)$ grid spanning $[0.1, 4.8] \times [-9.9, 9.9] \times [3, 7]$. 
Due to uniform obstacle distribution in the $\pz$ direction, sampling is denser in the $\px, \py$ plane, since we expect varying initial states in $\pz$ direction would not be a significant influence.
Each simulation starts with the quadrotor stationary, following reference trajectories determined by each methodology.

\paragraph{Implementation of Comparison Methods}\label{subsec:kdf_param}
In this section, we detail the implementation of KDF \cite{verginis2022kdf}, and the tracking controller we used to compare PARC and KDF.
The originally proposed controller from KDF \cite{verginis2022kdf} was unsuitable for under-actuated quadrotors; thus we adopted the Prescribed Performance Controller (PPC) from \cite{lapandic2022robust}.
This PPC ensures that position and yaw angle tracking errors stay within prescribed performance funnels, allowing it to be compatible with the KDF framework where it uses maximal tracking error in motion planning level to define \textit{extended free space}.

For PARC, we successfully tuned 48 hyper-parameters and defined 12 funnel functions for each quadrotor state, with corresponding control gains as follows:
\begin{subequations}
\begin{align}
\rho_{p_i}(t) &= 0.1+(0.5-0.1)e^{-0.5t},\quad i=\{x, y, z\},\\
    \rho_{\vx}(t) &= 0.5+(5-0.5)e^{-0.2t},\\
    \rho_{\vy}(t) &= 0.5+(5-0.5)e^{-0.2t},\\
    \rho_{\vz}(t) &= 0.2+(5-0.2)e^{-1.5t},\\
    \rho_{\phi\theta_1}(t) &= 0.1+(0.5-0.1)e^{-0.5t},\\
    \rho_{\phi\theta_2}(t) &= 0.1+(0.5-0.1)e^{-0.5t},\\
    \rho_{\psi}(t) &= 0.1+(0.4-0.1)e^{-0.05t},\\
\rho_{w_i}(t)  &= 0.3+(0.5-0.3)e^{-0.5t},\quad i=\{x, y, z\},\\
K_p &= \diag{1.25, 1.25, 12.5},\\
K_v&= \diag{20, 20, 5},\\
    K_{\phi\theta} &= \diag{3, 1.5},\\
    K_{\psi} &= 1,\\
K_w &= 10\cdot\eye_3.
\end{align}
\end{subequations}
We refer to \cite{lapandic2022robust} for details of the role of each hyper-parameter.
This hyper-parameter effectively tracked 2,400 reference trajectories from PARC in the wide gap scenario and all computed reach-avoid plans in the narrow gap scenario, as detailed in Table \ref{table:experiments} and Fig. \ref{fig:parc_narrow_realized}.

However, despite extensive efforts, we could not find a universal hyper-parameter that effectively tracked KDF's trajectories, which may be due to RRT-generated adversarial trajectories being difficult for the controller to follow. 
To address this, we used smoothing splines to fit the RRT plans to adhere to the twice-differentiability requirement of funnel controllers but still encountered challenges with universal hyper-parameter settings. 
Therefore, assuming the existence of an effective funnel controller for KDF, we evaluated its performance based on planned trajectories, and padding obstacles by $\rho_{p}(0)$ as per the MATLAB Navigation Toolbox's basic RRT implementation.
We also present a selection of successfully tracked trajectories by KDF in Fig. \ref{fig:parc_kdf_compare} to validate our implementation of KDF.

\begin{figure}
\centering
\includegraphics[width=\columnwidth]{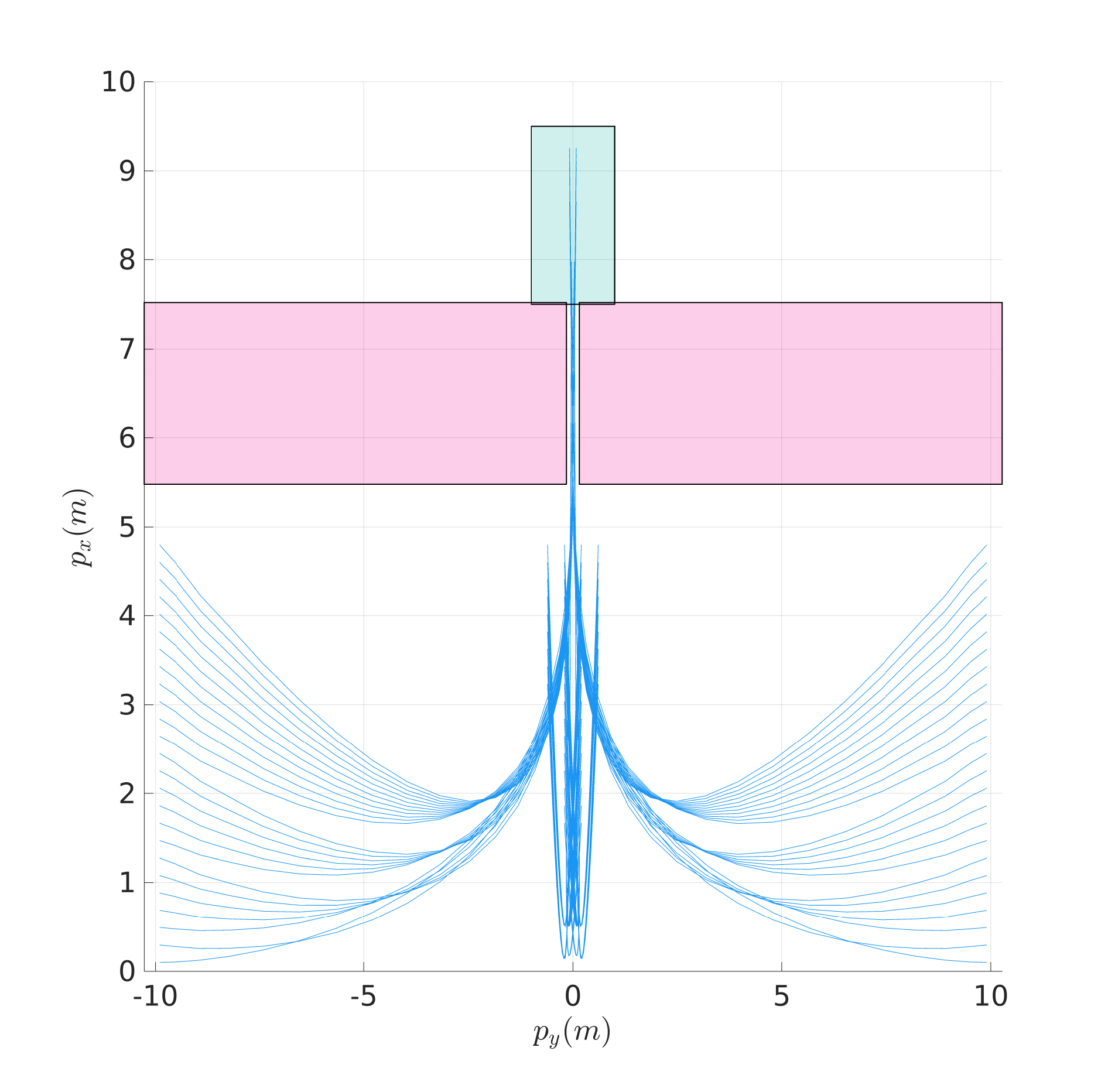}
\vspace{-2em}
\caption{The plot of realized trajectory in narrow-gap scenario, clearance of $0.31 \unit{m}$ for PARC.
The 306 reach-avoid plans were computed out of 3,750 initial states.
}
\vspace*{-0.5cm}
\label{fig:parc_narrow_realized}
\end{figure}
\subsection{Additional Experiments}\label{app:additional_exp}

We provide a further comparison to the experiment in \ref{exp:planning_model_design} by using a different planning model in the same environmental setting.
The experiment seeks to show how PARC could improve using a better planning model under the same experimental conditions.

\subsubsection{Planning and Tracking Model}\label{subsec:2dquad_poly_model}

We consider the time-switched polynomial model defined by \eqref{eq:accel_x}, \eqref{eq:accel_z}, \eqref{eq:decel_x}, and \eqref{eq:decel_z}, with $\tf = 2$, $\tpk = 1.5$, the planning states and workspace $\planv = \wv = [\px, \pz]\tp$, and trajectory parameters $\paramv = [\kvx, \kax, \kpkx, \kvz, \kaz, \kpkz]\tp$.
Just like in Section \ref{subsec:3dquad_model}, we separated the planning model into low-dimensional models for each of the two workspace dimensions, and recovered the high-dimensional reach set, avoid set, and BRAS by \eqref{eq:decouple_reach} and \eqref{eq:decouple_avoid}.

We used the same tracking model defined in \eqref{eq:2dhover_dyn} for a fair comparison.

\subsubsection{Experimental Setup}
Our experimental setup for the obstacles, goal region, and workspace domain are identical to that in Section \ref{exp:planning_model_design} and \cite{dawson2022safe}.
We also use the same iLQR controlling scheme as in Section \ref{exp:planning_model_design} but with the gains tuned to the new planning model.
We define the domain for trajectory parameters with $\K = [-1.5, 1.5]\times\{0\}\times[-0.1, 0.1]\times\{0\}\times[-1, 1]\times[-1.5, 1.5]$.

\subsubsection{Hypothesis}
We anticipated PARC to be able to compute a less conservative, but still guaranteed safe BRAS compared to Section \ref{exp:planning_model_design} due to the reduced number (three vs. zero) of non-ETI states in the planning model.
Moreover, we also expected PARC's computation time to be faster than the na\"ive model choice, as the dimension for polytopic computation has been reduced from 8 to 4.

Since the planning model in Appendix \ref{subsec:2dquad_poly_model} also incorporates stabilizing behavior for the quadrotor near the goal region \cite{kousik2019safe, mueller2015computationally}, we expected to produce similar reach-avoid results to \cite{dawson2022safe}, that is, the quadrotor would be able to safely navigate between the two obstacles while stopping at the goal region.

\subsubsection{Results}
Fig. \ref{fig:quad2dpoly_safe_set} shows the BRAS constructed by PARC under the planning model in Appendix \ref{subsec:2dquad_poly_model}.
The BRAS computation took $1.47  \unit{s}$, more than 3 times faster than that of the na\"ive model.

Trajectories were generated by uniformly sampling the BRAS.
No crashes were reported, with the agent successfully \textit{stopping} at the goal region in all simulations.
Comparing Fig. \ref{fig:quad2dpoly_safe_set} with Fig. \ref{fig:quad2d_yeserror}, safe trajectories that travel in-between two obstacles can be generated, reproducing similar results to \cite{dawson2022safe}.

\subsubsection{Discussion}
The time-switched polynomial planning model has comparable tracking error with the na\"ive planing model in Section \ref{exp:planning_model_design}.
Thus, the reduced conservativness in the computed BRAS for the polynomial model is largely due to the reduced number of non-ETI states in the PWA system, leading to a tighter approximation of the avoid sets.
Note that the BRAS in Fig. \ref{fig:quad2d_yeserror} extends more to the right compared with Fig. \ref{fig:quad2dpoly_safe_set} because the na\"ive planning model does not encode any stabilizing behavior, with the trajectory leading the agent to continuously accelerating towards the goal.
Thus, within the same $\tf$, it was able to reach the goal from a further distance compared with the stabilizing trajectories.

We further note that, while the chosen domain for the time-switched polynomial planning model was not able to generate very parabolic trajectories that enable the agent to reach the goal from behind the left obstacle, this can be easily fixed by receding-horizon planning into the BRAS, or by performing PARC across multiple domains of trajectory parameters, both of which remains as future work.

In conclusion, these experiments show the importance of choosing a good planning model for PARC with minimal dimensions and non-ETI states without losing representation power in the trajectories.
Should good planning models be unavailable, Section \ref{sec:drift_demo} shows how one can construct planning models from data while using PARC to compute their corresponding BRAS.

\begin{figure}
\centering
\includegraphics[width=0.8\columnwidth]{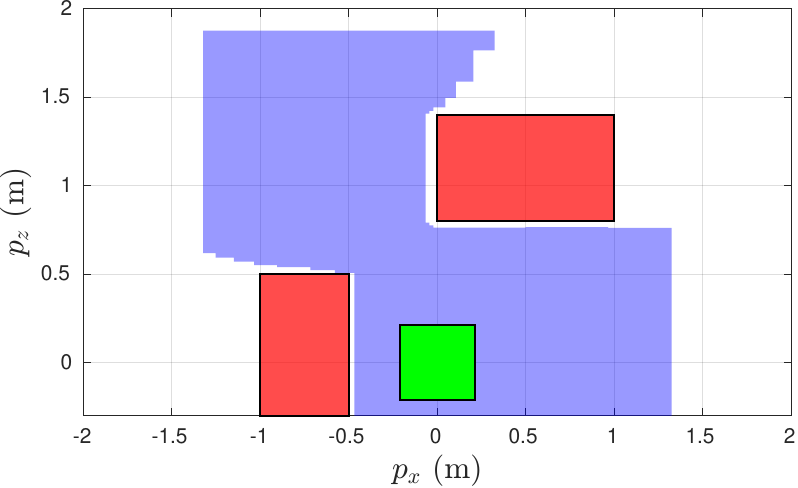}
\caption{Results of PARC on a time-switched polynomial planning model with stabilizing end-behavior, after accounting for tracking error. Green indicates the goal region, red indicates the obstacles with the volume of the drone accounted for, and blue shows the computed BRAS.
Safe trajectories that pass between the two obstacles can now be generated compared to the results in Section \ref{exp:planning_model_design}.
}
\label{fig:quad2dpoly_safe_set}
\vspace*{-0.5cm}
\end{figure}
\subsection{Drift Vehicle Demo Details}\label{app:drift_details}

We now define our tracking model, drift parking controller, and the time-variant affine planning model. 
We then present results and a brief discussion.

\subsubsection{Tracking Model}
We use a tracking model based on \cite{jelavic2017autonomous,goh2020toward,weber2023modeling,goh2016simultaneous}:   
\begin{equation}
\trackv = 
[
p_x, p_y, \theta, \vth, \vel, \sideslip
] \tp
\end{equation}
with its dynamics defined as:
\begin{align}\label{eq:drift tracking dyn}
\begin{split}
    \dot{p}_x =& \vel\cos(\theta+\sideslip)\\
    \dot{p}_y =& \vel\sin(\theta+\sideslip)\\
    \dot\theta =& \vth \\
    \dot\vth =& \frac{1}{\Iz}(\La \fyf \cos(\dl) + \La \fxf \sin(\dl)- \Lb \fyr) \\
    \dot\vel =& \frac{1}{\mc}(\fxf \cos(\dl - \sideslip) - \fyf \sin(\dl-\sideslip) \\
        &+ \fxr \cos(\sideslip) + \fyr \sin(\sideslip))  \\
    \dot\sideslip &= \frac{1}{\mc \vel} (\fxf \sin(\dl - \sideslip) - \fyf \sin(\dl - \sideslip) \\
        &+ \fxr \cos(\sideslip) + \fyr \sin(\sideslip)) - \vth 
\end{split}
\end{align}

In this single-track bicycle dynamics model, the first three states $p_x (\unit{m}), p_y (\unit{m}), \theta(\unit{rad})$ describe the location of the car's center of mass and its orientation or yaw angle with respect to the stationary world frame. 
$\vth (\unit{rad/s})$ is the yaw rate of the car; $\vel (\unit{m/s})$ is the magnitude of the car's velocity measured from its center of mass. 
$\sideslip (\unit{rad})$ is the sideslip angle between the car's traveling direction and its orientation.

The tracking model's control inputs are $\uv = [\fxr,\dl, \fxf_\regtext{max}]\tp$ which correspond to throttle (\unit{N}), steering (\unit{rad}), and brake (\unit{N}).
Note that $\fxr$ only provides forces in the longitudinal direction of the rear tire. 
The brake only provides forces along the direction of the front tire.
While its maximum value $\fxf_\regtext{max}$ is set by input, its actual magnitude $\fxf$ is inversely proportional to the front tire frame longitudinal velocity. 
Such design ensures that the braking force goes to zero when the car stops moving.

The lateral forces on both the front and the rear tires,
\begin{align}\label{eq:fiala}
    \fyf &= f_{\text{front}}(\vth, \vel, \sideslip, \fxf, \dl),
    \quad\regtext{and} \\
    \fyr &= f_{\text{rear}}(\vth, \vel, \sideslip, \fxr),
\end{align}
are formulated using the Fiala brush tire model \cite{fiala1954lateral}.
Lastly, the car's physical parameters ($\Iz, \mc, \La, \Lb$) as well as other constants used in the Fiala tire model are adopted from the DeLorean car specification from \cite{goh2016simultaneous}.

\subsubsection{Feedback Controller}\label{sec:drifting_controller}
Literature in drift parking controllers falls into two categories.
\cite{lau2011learning, leng2023deep} approached this problem with learning-based controllers.
\cite{kolter2010probabilistic, jelavic2017autonomous} designed mixed open and closed-loop control strategies with the former using the probabilistic method and the latter deterministic. 

In this demo, we adopt the control scheme described in \cite{jelavic2017autonomous}.
We first used a nonlinear MPC tracking controller to bring the car into a drift state then switched to an open-loop control scheme to complete the drift parking action. 

Specifically, given some initial state $\trackv_0$ and $\paramv = [\vel, \sideslip]\tp$, the controller provides input sequence $\uv\feedback=[\fxr, \dl, \fxf_\regtext{max}]\tp$ for the tracking model such that the vehicle would be at the appropriate sideslip angle to enter drifting regime at $\vel$.
Next, as an open-loop PD controller continuously adjusts $\dl$ to maintain an increasing sideslip angle $\sideslip$, a one-time control input of $\fxr =0 \unit{N}, \fxf_\regtext{max} = 5,163 \unit{N}$  (i.e., hard braking) is executed at the moment when $\sideslip$ exceeds the planned $\beta$ value (see $\paramv$ in the planning model below). 

\subsubsection{Planning Model}\label{subsubsec: drift_planning}
Different from the experiment examples, the mixed open and closed loop drifting controller cannot track explicit trajectories when the car enters the drifting regime. 
Therefore, the construction of an analytical model that describes the entire drift parking trajectory is difficult or even impossible just like many robotics applications \cite{nguyen2011model}.

In this demo, we showcase the flexibility of PARC by constructing the planning model using data.
We are directly constructing the planning model as a time-variant affine system with 3-D planning space and 2-D trajectory parameters, so PARC can be directly applied without Section \ref{subsec:PWA_convert}.
Specifically, we wish to design the affine dynamics in the form of:
\begin{align}\label{eq:braking_pwa}
    \begin{bmatrix}
        p_{x,\ts+\Delta\ts}\\
        p_{y,\ts+\Delta\ts}\\
        \theta_{\ts+\Delta\ts}\\
        {\vel}_{\ts+\Delta\ts}\\
        {\sideslip}_{\ts+\Delta\ts}
    \end{bmatrix} = \begin{bmatrix}
        1&0&0&\cparam_{\vel x, \ts}&\cparam_{\sideslip x, \ts}\\
        0&1&0&\cparam_{\vel y, \ts}&\cparam_{\sideslip y, \ts}\\
        0&0&1&\cparam_{\vel \theta, \ts}&\cparam_{\sideslip \theta, \ts}\\
        0&0&0&1&0\\
        0&0&0&0&1
    \end{bmatrix}\begin{bmatrix}
        p_{x,\ts}\\
        p_{y,\ts}\\
        \theta_{\ts}\\
        {\vel}_{\ts}\\
        {\sideslip}_{\ts}
    \end{bmatrix} + \begin{bmatrix}
        \dparam_{x, \ts}\\
        \dparam_{y, \ts}\\
        \dparam_{\theta, \ts}\\
        0\\
        0
    \end{bmatrix}
\end{align}
where the planning states $\planv = [p_x, p_y, \theta]\tp$ are the world-frame coordinates and heading angle of the car, $\wv = [p_x, p_y]\tp$ are the workspace states, trajectory parameters $\paramv = [\vel, \sideslip]\tp$ are the drifting velocity and the side-slip angle for controller switching, the domain for all timesteps of the PWA function is $\Xplan\times\K$, and we wish to design $\cparam_{\vel x, \ts}, \cparam_{\sideslip x, \ts}, \dparam_{x, \ts}, \cparam_{\vel y, \ts}, \cparam_{\sideslip y, \ts}, \dparam_{y, \ts}, \cparam_{\vel \theta, \ts}, \cparam_{\sideslip \theta, \ts},$ and $\dparam_{\theta, \ts}$ for all $\ts = 0, \Delta\ts, \cdots, \tf - \Delta\ts$ to fit \eqref{eq:braking_pwa} to the collected data.

For this example, as we mentioned in Remark \ref{rem:goal set expansion}, we are extending the dimensions of the goal region to all of $\planv$ instead of just $\wv$, so $\ol{\Obs} \subset \Xplan$ and $\Obs_\is = \ol{\Obs}_\is\times \K$ for all $\is = 1, \cdots, \nobs.$
This is such that we can enforce the $\theta$ of the car at $\tf$ so the parking is ``parallel''.
The extension of the goal dimension does not change the other parts of PARC's formulation.

Now we design the parameters in \eqref{eq:braking_pwa}.
Unlike our other system models, due to the lack of controllability during drifting, the behavior of the realized drift parking trajectories are ``open-loop'' in that they are uniquely and entirely determined given $\trackv_0$ and $\paramv$ without needing to provide the planning model, as illustrated in Section \ref{sec:drifting_controller}.
As such, we fit a planning model to the data collected in open-loop simulations and use PARC to compute the BRAS of the initial states.

We first collect data by sampling from the pairs $(\trackv_{0, \is}, \paramv_\is) \in \Xtrack_{\data}\times\K_{\data}\subset \Xtrack\times\K$ for $\is = 1, \cdots, \ndim_\data$.
For the $\is^{\regtext{th}}$ data, we denote the realized trajectories $\trackv_\is(\ts; \trackv_{0, \is}, \paramv_\is) = [{p_x}_\is(\ts), {p_y}_\is(\ts), {\theta}_\is(\ts), \cdots]\tp$ and the sampled trajectory parameters $\paramv_\is = [{\vel}_\is, {\sideslip}_\is]\tp$.
Then, we can fit the parameters in \eqref{eq:braking_pwa} by solving the least squares problems:
\begin{subequations} \label{eq:ltsq_drift}
\begin{align}
    \min_{\cparam_{\vel x, \ts}, \cparam_{\sideslip x, \ts}, \dparam_{x, \ts}}&\Bigg\|\begin{bmatrix}
        {\vel}_1 & {\sideslip}_1 & 1\\
        \vdots&\vdots&\vdots\\
        {\vel}_{\ndim_\data} & {\sideslip}_{\ndim_\data} & 1
    \end{bmatrix}\begin{bmatrix}
        \cparam_{\vel x, \ts}\\ \cparam_{\sideslip x, \ts}\\ \dparam_{x, \ts}
    \end{bmatrix}\notag\\
    &-\begin{bmatrix}
        {p_x}_1(\ts+\Delta\ts) - {p_x}_1(\ts)\\
        \vdots\\
        {p_x}_{\ndim_\data}(\ts+\Delta\ts) - {p_x}_{\ndim_\data}(\ts)
    \end{bmatrix}\Bigg\|^2_2,\\
    \min_{\cparam_{\vel y, \ts}, \cparam_{\sideslip y, \ts}, \dparam_{y, \ts}}&\Bigg\|\begin{bmatrix}
        {\vel}_1 & {\sideslip}_1 & 1\\
        \vdots&\vdots&\vdots\\
        {\vel}_{\ndim_\data} & {\sideslip}_{\ndim_\data} & 1
    \end{bmatrix}\begin{bmatrix}
        \cparam_{\vel y, \ts}\\ \cparam_{\sideslip y, \ts}\\ \dparam_{y, \ts}
    \end{bmatrix}\notag\\
    &-\begin{bmatrix}
        {p_y}_1(\ts+\Delta\ts) - {p_y}_1(\ts)\\
        \vdots\\
        {p_y}_{\ndim_\data}(\ts+\Delta\ts) - {p_y}_{\ndim_\data}(\ts)
    \end{bmatrix}\Bigg\|^2_2,\\
    \min_{\cparam_{\vel \theta, \ts}, \cparam_{\sideslip \theta, \ts}, \dparam_{\theta, \ts}}&\Bigg\|\begin{bmatrix}
        {\vel}_1 & {\sideslip}_1 & 1\\
        \vdots&\vdots&\vdots\\
        {\vel}_{\ndim_\data} & {\sideslip}_{\ndim_\data} & 1
    \end{bmatrix}\begin{bmatrix}
        \cparam_{\vel \theta, \ts}\\ \cparam_{\sideslip \theta, \ts}\\ \dparam_{\theta, \ts}
    \end{bmatrix}\notag\\
    &-\begin{bmatrix}
        {\theta}_1(\ts+\Delta\ts) - {\theta}_1(\ts)\\
        \vdots\\
        {\theta}_{\ndim_\data}(\ts+\Delta\ts) - {\theta}_{\ndim_\data}(\ts)
    \end{bmatrix}\Bigg\|^2_2,
\end{align}
\end{subequations}
repeated for each $\ts = 0, \Delta\ts, \cdots, \tf - \Delta\ts$.
\eqref{eq:ltsq_drift} can be solved by any least squares solver.
Once \eqref{eq:ltsq_drift} is computed, \eqref{eq:braking_pwa} is now a fully defined time-variant affine function that can be used in the PARC framework.

Note that, since \eqref{eq:braking_pwa} is generated only from data with initial conditions in $\Xtrack_{\data}\times\K_{\data}$, the BRAS is also only valid from within that region.
Specifically, when sampling for tracking error \`a la Section \ref{subsec:tracking_error}, one should have $\itv{\Xtrack}\times\itv{\K} \subset \Xtrack_{\data}\times\K_{\data}$.

For the planning model, we set the timestep $\Delta\ts = 0.1$ \unit{s} and final time $\tf = 7.8$ \unit{s}.

\subsubsection{Environment Details}
In our environment, we define the tracking model state and input bounds to be:
\begin{equation}
    \trackv\bound =
    \begin{bmatrix}
        -5, & 40 \\
        -20, & 5 \\
        -1.5\pi, & 0.5\pi \\
        -2, & 1 \\
        -0.5, & 20 \\
        -\pi, & \pi
    \end{bmatrix}\quad
    \uv\bound = 
    \begin{bmatrix}
        0, & 7684 \\
        -0.6632, & 0.6632 \\
        -5163, & 5163
    \end{bmatrix}
\end{equation}

We use the geometry specification of the 1981 DeLorean ($l=4.267, w=1.988$ \unit{m}) \cite{goh2020toward,goh2016simultaneous} to represent the car agent as a 2D rectangular H-Polytope $\ol{B}$. 
The obstacles ${\ol{\Obs}_{\regtext{body}}}_1, {\ol{\Obs}_{\regtext{body}}}_2$ are two parked car obstacles centered at $[24, -16,-\pi]\tp$ and $[37, -16, -\pi]\tp$ each with the same geometry as $\ol{B}$.

Next, to ensure that the agent doesn't collide with the obstacles, we create a circular over approximation of $\ol{B}$ as $\Bc$ with a radius of $r=\sqrt{l^2+w^2}/2$. 
We then dilate the obstacle into:
\begin{align*}
    {\ol{\Obs}_1} &= {\ol{\Obs}_{\regtext{body}}}_1 \oplus \Bc \\
    {\ol{\Obs}_2} &= {\ol{\Obs}_{\regtext{body}}}_2 \oplus \Bc 
\end{align*}
so that as long as the geometrical center of the car agent $\wv$ does not enter $\ol{\Obs}_1$ and $\ol{\Obs}_2$ as formulated in Problem \ref{prob:reach-avoid-set}, the trajectory is safe.
Note that in this demo, the geometric center and the center of mass coincide.

In this demo, our desired state for the car includes position and orientation states. 
Therefore we expand the goal set dimension from the previously used workspace dimension $\R^{\ndim_w}$ to the planning state dimension $\R^{\ndim_p}$ by including a range of  $\theta$ as a goal state.
So, we slightly modify Problem \ref{prob:reach-avoid-set}'s goal reaching condition to
\begin{equation*}
    \proj[1:\ndim_p]{\trackv(\ts; \trackv_0, 
    \paramv)} \in \ol{\Xgoal}
\end{equation*}
We place the center of the goal in between $\ol{\Obs}_1$ and $\ol{\Obs}_2$ at \[\planv_\goal =[
    30.5 , -16 , -\pi
]\tp\] 
The goal set $\ol{\Xgoal}$ is then defined as a 3D cuboid H-Polytope centered at $\planv_\goal$ with its length in $[p_x, p_y, \theta]$ as $[3.7, 1.6, \pi/3]$.

Using the drift parking test environment, we collect 10,000 expert trajectories generated by feeding samples of $\paramv \in \K = [9, 11]\times[0.6109, 0.7854]$ into the drift parking controller described above.
The initial condition of all trajectories $\trackv_0$ is all zeros except for the velocity state $\vel_0 = 0.01$, as the tracking model experiences numerical instability at low velocities and sideslip angles.
The entire simulation process lasted for 2 \unit{h}. 
With the expert trajectories ready, we created the time-variant affine planning model using the formulation above.

\begin{figure}[ht]
    \centering
    \includegraphics[width=0.8\columnwidth]{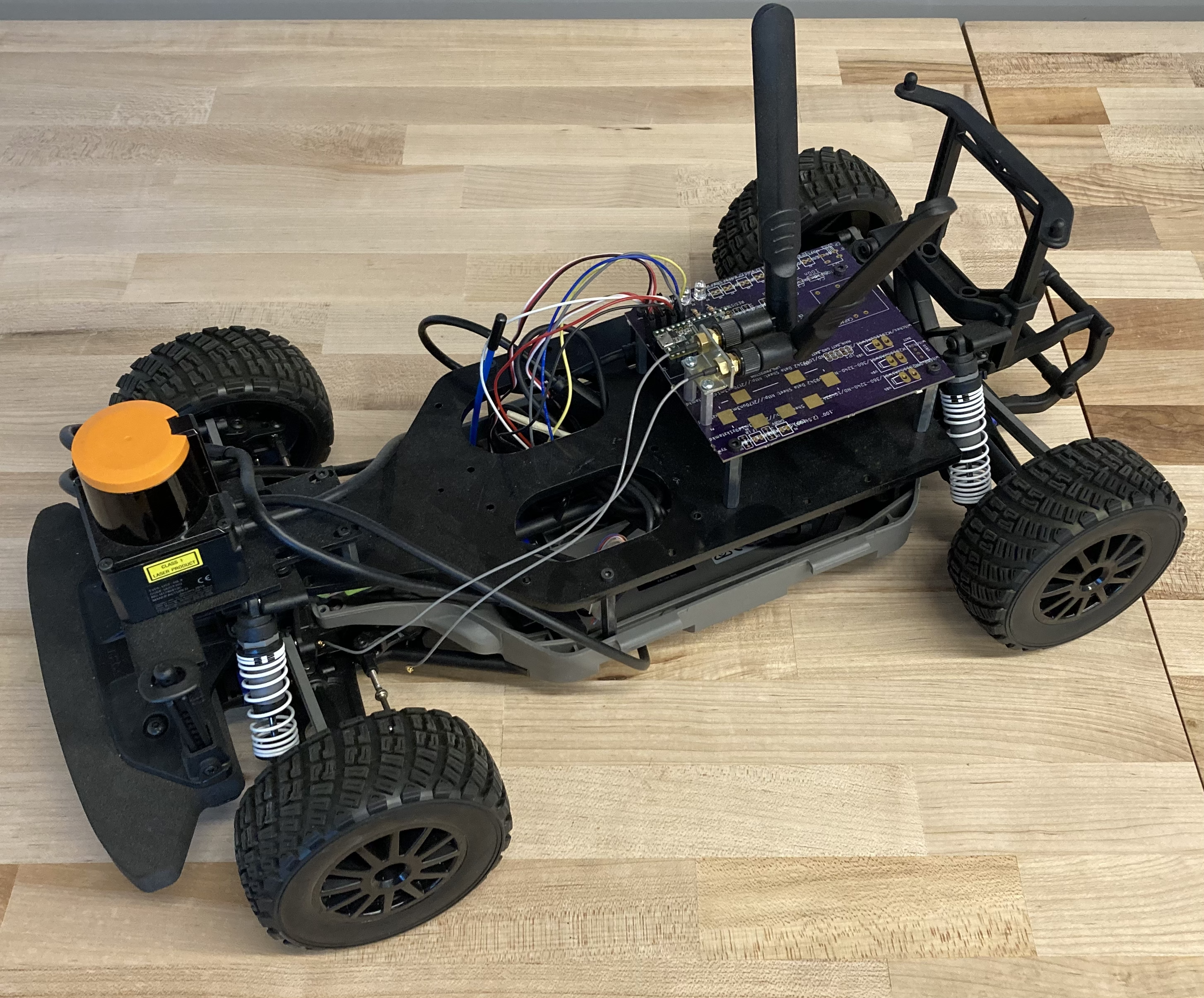}
    \caption{We are preparing an F1:10 class robotic vehicle to implement and deploy our safe drifting maneuvers.}
    \label{fig:F1tenth}
\end{figure}

\end{document}